\documentclass[10pt,twocolumn,letterpaper]{article}

\usepackage{iccv}
\usepackage{times}
\usepackage{epsfig}
\usepackage{graphicx}
\usepackage{amsmath}
\usepackage{amsthm}
\usepackage{amssymb}
\usepackage{booktabs}
\usepackage[ruled,vlined]{algorithm2e}
\usepackage{dsfont}
\usepackage{multirow}
\SetKwInput{KwInput}{Input}                
\SetKwInput{KwOutput}{Output}              
\SetKwInput{KwRequire}{Require}              

\newcommand{\bxi}{{\boldsymbol{\xi}}}
\def\x{{\boldsymbol{x}}}
\def\g{{\boldsymbol{g}}}
\def\m{{\boldsymbol{m}}}

\newcommand{\EE}{{\mathbb{E}}}
\newcommand{\RR}{{\mathbb{R}}}
\newcommand{\vx}{{\mathbf{x}}}
\newcommand{\vy}{{\mathbf{y}}}
\newcommand{\bs}{{\mathbf{s}}}

\newcommand{\vm}{{\mathbf{m}}}
\newcommand{\vg}{{\mathbf{g}}}

\newcommand{\cF}{{\mathcal{F}}}
\newcommand{\cL}{{\mathcal{L}}}
\newcommand{\cO}{{\mathcal{O}}}

\newtheorem{proposition}{Proposition}
\newtheorem{theorem}{Theorem}
\newtheorem{lemma}{Lemma}
\newtheorem{remark}{Remark}
\newtheorem{corollary}{Corollary}
\allowdisplaybreaks

\usepackage[pagebackref=true,breaklinks=true,letterpaper=true,colorlinks,bookmarks=false]{hyperref}

\iccvfinalcopy 

\def\iccvPaperID{2327} 

\begin{document}

\title{DecentLaM: Decentralized Momentum SGD for Large-batch  Deep Training}

\author{Kun Yuan$^1$\thanks{Equal contribution. Correspondence can be addressed to Kun Yuan ({\tt\small kun.yuan@alibaba-inc.com})
} , Yiming Chen$^{1*}$, Xinmeng Huang$^{2*}$, Yingya Zhang$^1$, Pan Pan$^1$, Yinghui Xu$^1$, Wotao Yin$^1$\\
\vspace{-2mm}\\
$^1$Alibaba Group $\quad ^2$University of Pennsylvania \\
}


\maketitle

\begin{abstract}
   The scale of deep learning nowadays calls for efficient distributed training algorithms. Decentralized momentum SGD (DmSGD), in which each node averages only with its neighbors, is more communication efficient than vanilla Parallel momentum SGD that incurs global average across all computing nodes. On the other hand, the large-batch training has been demonstrated critical to achieve runtime speedup. This motivates us to investigate how DmSGD performs in the large-batch scenario.
   
   In this work, we find the momentum term can amplify the  inconsistency bias in DmSGD. Such bias becomes more evident as batch-size grows large and hence results in severe performance degradation. We next propose DecentLaM, a novel  \underline{decent}ralized \underline{la}rge-batch \underline{m}omentum SGD to remove the momentum-incurred bias. The convergence rate for both non-convex and strongly-convex scenarios is established. Our theoretical results justify the superiority of DecentLaM to DmSGD especially in the large-batch scenario. Experimental results on a variety of computer vision tasks and models demonstrate that DecentLaM promises both efficient and high-quality training.
   
   
\end{abstract}

\section{Introduction}

Efficient distributed training across multiple computing nodes is critical for large-scale deep learning tasks nowadays. As a principal training algorithm, Parallel SGD computes a globally averaged gradient either using the {\em Parameter Server (PS)} \cite{li2014scaling} or the {\em All-Reduce} communication primitive \cite{patarasuk2009bandwidth}. 
Such global synchronization across all nodes either incurs significant bandwidth cost or high latency that can severely hamper the training scalability. 

Decentralized SGD \cite{nedic2009distributed, chen2012diffusion, lian2017can, lian2018asynchronous, assran2019stochastic} based on {\em partial averaging} has become one of the major approaches in the past decade to reduce communication overhead in distributed optimization. Partial averaging, as opposed to {\em global averaging} used in Parallel SGD, requires every node to compute the average of the nodes in its neighborhood, see Fig.~\ref{fig:decen_comm}. If a sparse topology such as one-peer exponential graph \cite{assran2019stochastic} is utilized to connect all nodes, each node only communicates with {\em one} neighbor {\em each} iteration and hence saves remarkable communications. Decentralized SGD can typically achieve $1.3\sim2\times$ training time speedup without performance degradation \cite{lian2017can,assran2019stochastic,kong2021consensus}. 

Existing Decentralized SGD methods \cite{lian2017can,lian2018asynchronous,tang2018d,lin2021quasi,assran2019stochastic} and their momentum accelerated variants \cite{yu2019linear,singh2020squarm, gao2020periodic,balu2020decentralized} primarily utilize  small-batch in their algorithm design. However, recent hardware advances make it feasible to store large batches in memory and compute their gradients timely. Furthermore, the total batch size naturally grows when more computing nodes participate into training. These two primary reasons lead to the recent  exploration of large-batch deep training algorithms. 


\begin{figure}[t!]
    \centering
    \includegraphics[width=0.48\textwidth]{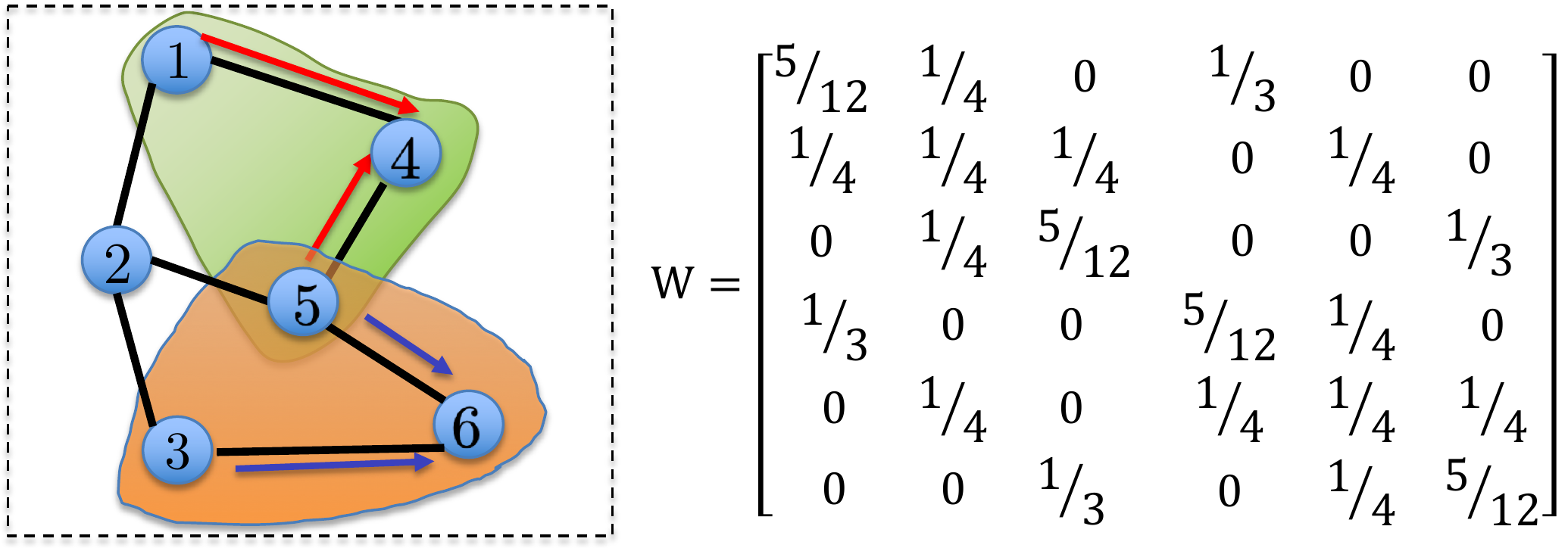}
    \vspace{-1mm}
    \caption{\small Illustration of decentralized methods. Nodes receive information from neighbors; they do not relay information. For example, nodes $4$ and $6$ collect information from their neighbors $\{1,5\}$ and $\{3,5\}$, respectively.
    Other nodes do the same but not depicted. Topology connectivity can be represented in a matrix, as shown in the right figure, see more details in Sec.~\ref{sec:DmSGD}. }
    \label{fig:decen_comm} \vspace{-2mm}
    \vspace{-2mm}
\end{figure}

In fact, large-batch training has been extensively studied in Parallel SGD. Pioneering works \cite{goyal2017accurate,you2017large,you2019largeb} find large-batch can significantly speed up Parallel SGD. First, the computation of a large-batch  gradient can fully utilize the computational resources (e.g., the computing nodes, the CUDA cores and GPU memories within each node). Second, large-batch gradient will result in a reduced variance and hence enables a much larger learning rate. 
With newly-proposed layer-wise adaptive rate scaling (LARS) \cite{you2017large} and its variant \cite{you2019largeb}, large-batch Parallel momentum SGD (PmSGD) can cut down the training time of BERT and Resnet-50 from days to hours \cite{you2018imagenet}.

\begin{table}[t]
\centering
\vskip 0.15in
\begin{tabular}{ccccc}
\toprule
\textbf{Dataset}           & \multicolumn{2}{c}{\textbf{Cifar-10}} & \multicolumn{2}{c}{\textbf{ImageNet}} \\
\textbf{Batch-size} & 2K    & 8K    & 2K    & 32K     \\ \midrule
\textbf{PmSGD}     & 91.6\%    & 89.2\%     & 76.5\% & 75.3\%   \\
\textbf{DmSGD}      & 91.5\%  & 88.3\%    & 76.5\%     & 74.9\%  \\ 
\bottomrule \\
\end{tabular}
\vskip -0.1in
\caption{Top-1 validation accuracy comparison between PmSGD and DmSGD under the small-batch and large-batch settings. No layer-wise adaptive rate scaling is used in any of these algorithms. All hypter-parameters are exactly the same. More experimental details can be referred to Appendix \ref{app-table-introduction}} 
\vskip -0.2in
\label{talbe:motivation}
\end{table}

This naturally motivates us to study {\em how Decentralized momentum SGD (DmSGD) performs with large batch-size}. To this end, we compared PmSGD and DmSGD over Cifar-10 (Resnet-20) and ImageNet (Resnet-50) with both small and large batches.
Their performances are listed in Table \ref{talbe:motivation}. While DmSGD achieves the same accuracy as PmSGD with small batch-size, it has far more performance degradation in the large-batch scenario. This surprising observation, which reveals that the extension of DmSGD to large-batch is non-trivial, raises two fundamental questions: 
\begin{itemize}
    \item Why does DmSGD suffer from severe  performance degradation in the large-batch scenario?
    \vspace{-2mm}
    \item Is there a way to enhance the accuracy performance of large-batch DmSGD so that it can match with or even beat large-batch PmSGD?
    \vspace{-2mm}
\end{itemize}
This paper focuses on these questions and provides affirmative answers. In particular, our main contributions are:
\begin{itemize}
    \item We find large-batch DmSGD has severe performance degradation compared to large-batch PmSGD and have clarified the reason behind this phenomenon.  
    It is discovered that the momentum term can significantly amplify the inconsistency bias in DmSGD.  When batch-size is large and the gradient noise is hence dramatically reduced, such inconsistency bias gets dominant and thus degrades DmSGD's performance notably. 
    \vspace{-2mm}
    \item We propose DecentLaM, a novel \underline{decent}ralized \underline{la}rge-batch \underline{m}omentum SGD to remove the momentum-incurred bias in DmSGD. We establish its convergence rate for both non-convex and strongly-convex scenarios. Our theoretical results show that DecentLaM has superior performance to  existing decentralized momentum methods, and such superiority gets more evident as batch size gets larger.
    \vspace{-2mm}
    \item Experimental results on a variety of computer vision tasks and models show that DecentLaM outperforms various existing baselines such as DmSGD, DA/AWC/D$^2$-DmSGD, SlowMo, PmSGD, and PmSGD$+$LARS in terms of the training accuracy.
    \vspace{-2mm}
\end{itemize}

The rest of this paper is organized as follows: We briefly summarize related works in Sec.~\ref{sec:related_works} and review DmSGD in Sec.~\ref{sec:DmSGD}. We identify the issue that causes performance degradation in DmSGD (Sec.~\ref{sec:DmSG-bias}) and propose DecentLaM to resolve it (Sec.~\ref{sec:decentLaM}). The convergence analysis and experiments are established in Sec.~\ref{sec:DecentLaM_convergence} and  Sec.~\ref{sec:experiment}, respectively. 

\section{Related Works}
\label{sec:related_works}
\vspace{1mm}
\noindent \textbf{Decentralized deep training.}
Decentralized optimization can be tracked back to \cite{tsitsiklis1986distributed}. Decentralized gradient descent \cite{nedic2009distributed,yuan2016convergence}, diffusion \cite{chen2012diffusion,sayed2014adaptation} and dual averaging \cite{duchi2011dual} are among the first decentralized algorithms that target on general optimization problems arise from  signal processing and control communities. In the context of deep  learning, Decentralized SGD (DSGD) has gained a lot of attentions recently. DSGD will incur model inconsistency among computing nodes. However, it is established in \cite{lian2017can} that DSGD can reach the same linear speedup as vanilla Parallel SGD in terms of convergence rate. After that, \cite{assran2019stochastic} comes out to extend DSGD to directed topologies. A recent work \cite{koloskova2020unified} proposes a unified framework to analyze DSGD  with changing topologies and local updates. \cite{lian2018asynchronous,luo2020prague} extend DSGD to the asynchronous setting. Various communication-efficient techniques can be further integrated into DSGD such as periodic  updates \cite{stich2019local,koloskova2020unified,yu2019linear}, communication compression \cite{alistarh2017qsgd,bernstein2018signsgd,koloskova2019decentralized,koloskova2019decentralized2,tang2019doublesqueeze}, and lazy communication \cite{chen2018lag,liu2019communication}. 

\vspace{1mm}
\noindent \textbf{Momentum SGD training.} Momentum SGD  have been extensively studied due to their empirical success in deep learning. The works \cite{loizou2020momentum,yuan2016influence,gitman2019understanding,liu2020improved,yu2019linear} establish that momentum SGD converges at least as fast as SGD. Momentum SGD for overparameterized models are shown to converge faster than SGD asymptotically \cite{sebbouh2020convergence}.  The exploration in decentralized momentum SGD (DmSGD) is relatively limited.  \cite{assran2019stochastic} proposes a widely-used DmSGD approach in which a local momentum SGD step is updated first before the partial averaging is conducted (see Algorithm \ref{Algorithm: DmSGD}). This approach is extended by  \cite{singh2020squarm,gao2020periodic} to involve communication quantization and periodic local updates. Another work (Doubly-averaging DmSGD, or DA-DmSGD) \cite{yu2019linear} imposes an additional partial averaging over momentum to increase stability. \cite{wang2019slowmo} proposes a slow
momentum (SlowMo) framework, where each node periodically synchronize and
perform a momentum update. \cite{balu2020decentralized} proposes a new variant in which the partial-averaging step is mixed up with the local momentum SGD update. All these decentralize momentum methods were studied and tested with {\em small} batch sizes. 

\vspace{1mm}
\noindent \textbf{Large-batch training.} The main challenge to use large-batch lies in the generalization performance degradation. Recent works in large-batch training centered on adaptive learning rate strategies to enhance accuracy performance. For example, Adam and its variants \cite{kingma2014adam, reddi2019convergence} adjust the learning rate based on the gradient variance, 
and \cite{goyal2017accurate} utilizes learning rate warm-up and linear scaling to boost the performance in large-batch scenario. The layer-wise adaptive rate scaling \cite{you2017large,you2019large}  can reduce the training time of Resnet-50 and BERT from days to hours. However, the study of large-batch training in decentralized algorithms is quite limited. This paper does not focus on the adaptive rate strategy for decentralized algorithms. Instead, we target on clarifying why existing momentum methods result in an intrinsic convergence bias, and how to update momentum properly to remove such bias and hence improve the performance in the large-batch scenario.

\vspace{1mm}
\noindent \textbf{Decentralized methods on heterogeneous data.} Decentralized large-batch training within data-centers shares the same essence with decentralized training on heterogeneous data for EdgeAI applications. In large-batch training, the stochastic bias caused by gradient noise gets significantly reduced and the inconsistency bias caused by data heterogeneity will become dominant. \cite{tang2018d,lu2019gnsd,yuan2020influence,xin2020improved} proposed decentralized stochastic primal-dual algorithms to remedy the influence of data heterogeneity. However, none of these algorithms, due to their difficulty to be integrated with momentum acceleration, show strong effective empirical performances in commenly-used deep learning models such as ResNet-50 or EfficientNet. A concurrent work \cite{lin2021quasi} proposes Quasi-Global momentum (QG-DmSGD), which locally approximates the global optimization direction, to mitigate the affects of heterogeneous data.
It is worth noting that while DecentLaM is proposed for the large-batch setting within data-centers, it is also suitable for  EdgeAI applications where inconsistency bias resulted from heterogeneous data dominates.


\section{Decentralized Momentum SGD}
\label{sec:DmSGD}
\noindent \textbf{Problem.} Suppose $n$ computing nodes collaborate to solve the distributed optimization problem:
\begin{align}\label{dist-opt}
\min_{x \in \mathbb{R}^d}\ f(x)=\frac{1}{n}\sum_{i=1}^n [f_i(x): = \mathbb{E}_{\xi_i \sim D_i} F(x;\xi_i)]
\end{align}
where $f_i(x)$ is local to node $i$, and random variable $\xi_i$ denotes the local data that follows distribution $D_i$. Each node $i$ can locally evaluate stochastic gradient $\nabla F(x;\xi_i)$; it must communicate to access information from other nodes.

\vspace{1mm}
\noindent \textbf{Notation.} We let $[n] := \{1,\cdots,n\}$, and $\mathds{1} \in \mathbb{R}^d$ be a vector with each element being $1$. We let $\lambda_i(A)$ denote the $i$-th largest eigenvalue of matrix $A$.

\vspace{1mm}
\noindent \textbf{Network topology and weights.} Decentralized methods are based on partial averaging within neighborhood that is defined by the network topology. We assume all computing nodes are connected by an undirected network topology. Such connected  topology can be of any shape, but its degree and connectivity will affect the communication efficiency and convergence rate of the decentralized algorithm. For a given topology, we define $w_{ij}$, the weight to scale information flowing from node $j$ to node $i$, as follows:
\begin{align}\label{wij}
w_{ij}
\begin{cases}
> 0 & \mbox{if node $j$ is connected to $i$, or $i=j$;} \\
= 0 & \mbox{otherwise.}\vspace{-2mm}
\end{cases}
\end{align}
We further define $\mathcal{N}_i:=\{j|w_{ij} > 0\}$ as the set of neighbors of node $i$ which also includes node $i$ itself. We define {\em weight matrix} $W := [w_{ij}]_{i,j=1}^{n} \in \mathbb{R}^{n\times n}$ to stack all weights into a matrix. Such matrix $W$ will characterize the sparsity and connectivity of the underlying network topology. An example of the topology and its associated weight matrix $W$ is illustrated in Fig.~\ref{fig:decen_comm}.

\vspace{1mm}
\noindent \textbf{Partial averaging.}
With weights $\{w_{ij}\}$ and the set of neighbors $\mathcal{N}_i$, the neighborhood partial averaging operation of node $i$ can be expressed as 
\begin{align}\label{partial-ave}
\hspace{-10mm}
\mbox{Partial averaging:}\quad x_i^{+} \leftarrow \sum_{j\in \mathcal{N}_i}w_{ij} x_j.
\end{align}
Partial averaging has much lower communication overheads. When the network topology is sparse, partial averaging typically incurs $O(1)$ latency plus $O(1)$ bandwidth cost, which are independent of the number of computing nodes $n$. Consequently, decentralized methods are more communication efficient than those based on global averaging.

\vspace{1mm}
\noindent \textbf{Decentralized SGD (DSGD).} Given a connected network topology and weights $\{w_{ij}\}$, each node $i$ in DSGD will iterate in parallel as follows: 
\begin{align}
x_i^{(k+\frac{1}{2})} &= x_i^{(k)} - \gamma \nabla F(x_i^{(k)}; \xi_i^{(k)})\quad \mbox{(local update)} \label{dsgd-1}\\
x_i^{(k+1)} &= \sum_{j\in \mathcal{N}_i} w_{ij}\, x_j^{(k+\frac{1}{2})} \hspace{0.8cm} \mbox{(partial averaging)} \label{dsgd-2}
\end{align}
where $x_i^{(k)}$ is the local model of node $i$ at iteration $k$, $\xi_i^{(k)}$ is the realization of $\xi_i$ at iteration $k$, and $\gamma$ is the learning rate. When the network topology is fully connected and $w_{ij} = 1/n$, DSGD will reduce to the  Parallel SGD algorithm. 

\begin{algorithm}[t]
 	\DontPrintSemicolon
 	\KwRequire{Initialize $\gamma$, $x^{(0)}_{i}$; let  $m^{(0)}_{i}\hspace{-0.3mm}=\hspace{-0.3mm}0, \beta\hspace{-0.3mm}\in\hspace{-0.3mm}(0,1)$}

    \vspace{1mm}
 	\For{$k=0, 1,2,...,T-1$, every node $i$}{
 		Sample $\xi^{(k)}_{i}$ and update $\tilde{g}_i^{(k)} \hspace{-0.8mm}=\hspace{-0.8mm} {\nabla}F(x^{(k)}_{i}\hspace{-0.4mm};\xi^{(k)}_{i})$ \;
 		
 		$m^{(k+1)}_i = \beta m^{(k)}_{i} + \tilde{g}_i^{(k)} \hspace{0.7cm} \triangleright \mbox{\footnotesize{momentum update}}$ \;
 		
 		$x^{(k+\frac{1}{2})}_i = x^{(k)}_{i} - \gamma m^{(k+1)}_i \hspace{0.38 cm} \triangleright \mbox{\footnotesize{local model update}}$ \;
 		
 		$x^{(k+1)}_i  = \sum_{j\in \mathcal{N}_i} w_{ij} x^{(k+\frac{1}{2})}_j \hspace{1.1mm} \triangleright \mbox{\footnotesize{partial average}}$\;
 	}
 	\caption{DmSGD}
 	\label{Algorithm: DmSGD}
 \end{algorithm}

\vspace{1mm}
\noindent \textbf{Decentralized momentum SGD (DmSGD).} Being the momentum accelerated extension of DSGD, DmSGD has been widely-used in existing literatures \cite{lian2018asynchronous,assran2019stochastic,singh2020squarm, gao2020periodic,yu2019linear,balu2020decentralized}. The primary version of DmSGD is listed in Algorithm \ref{Algorithm: DmSGD}. 
When small-batch is used, DmSGD will achieve $1.3\sim2\times$ speedup in training time compared to PmSGD without visibly loss of generalization performance. 

\vspace{1mm}
\noindent \textbf{Assumptions.} We introduce several standard assumptions to facilitate future analysis:

\vspace{0.5mm}
\noindent \textbf{A.1} Each $f_i(x)$ is $L$-smooth, i.e., $\|\nabla f_i(x) - \nabla f_i(y)\| \le L \|x - y\|$ for any $x,y\in \mathbb{R}^d$.

\vspace{0.5mm}
\noindent \textbf{A.2} The random sample  $\xi_i^{(k)}$ is independent of each other for any $k$ and $i$. We also assume each stochastic gradient is unbiased and has bounded variance, i.e., $\mathbb{E}[{\nabla}F(x;\xi_{i})] = \nabla f_i(x)$ and $\mathbb{E}\|{\nabla}F(x;\xi_{i}) - \nabla f_i(x) \|^2 \le  \sigma^2$.

\vspace{0.5mm}
\noindent \textbf{A.3} The network topology is strongly connected, and the weight matrix is symmetric and satisfies $W\mathds{1} = \mathds{1}$.

\vspace{1mm} 
\noindent Assumption A.3 indicates  $\sum_{j\in\mathcal{N}_i}w_{ij} = 1$ for $i\in[n]$, which is critical to guarantee the partial averaging \eqref{partial-ave} to  converge to the global averaging asymptotically. The weight matrix satisfying Assumption A.3 can be easily constructed, see \cite[Table~14.1]{sayed2014adaptation}.

\section{DmSGD Incurs Severe Inconsistency Bias}

Table \ref{talbe:motivation} shows that large-batch DmSGD has severe performance degradation compared to PmSGD in the large-batch scenario.  This section targets to explore the reason behind this phenomenon. To highlight the insight, we assume each $f_i(x)$ to be strongly convex in this section, and $x^\star$ to be the global solution to problem \eqref{dist-opt}. We let matrix $\vx = [x_1, \cdots, x_n]^T \in \mathbb{R}^{n\times d}$ to stack all local models across the network, and matrix $\nabla f(\vx) = [\nabla f(x_1), \cdots, \nabla f(x_n)]^T \in \mathbb{R}^{n\times d}$ to stack all local (accurate) gradients.

\vspace{1mm}
\noindent \textbf{Limiting bias of decentralized methods.} The convergence of decentralized methods such as DSGD and DmSGD will suffer from two sources of bias:
\begin{itemize}
    \item \textbf{Stochastic bias} is caused by the utilization of stochastic gradient in algorithms.
    
    \item \textbf{Inconsistency bias} is caused by the data  inconsistency between nodes. 
\end{itemize}

\noindent In addition, these two bias are orthogonal to each other, i.e., 
\begin{align}
\boxed{\lim_{k\to \infty} \sum_{i=1}^n \mathbb{E}\|x_i^{(k)} - x^\star\|^2 = \mbox{sto. bias} + \mbox{inconsist. bias}} \nonumber 
\end{align}
An example illustrating the stochastic and inconsistency bias of DSGD is established in Appendix \ref{app-limit-bias-dsgd}.

\vspace{1mm}
\noindent \textbf{Inconsistency bias dominates large-batch scenario.} In the large-batch scenario, the gradient nosie will be notably reduced. In an extreme case where a full-batch gradient is utilized, the stochastic bias becomes zero. This leads to
\begin{proposition}\label{prop-limit-bias}
The inconsistency bias dominates the convergence of {\em large-batch} decentralized algorithms.
\end{proposition}

\vspace{1mm}
\noindent \textbf{DmSGD's inconsistency bias: intuition.}
\label{sec:DmSG-bias}
The inconsistency bias can be achieved by letting gradient noise be zero. To this end, we let $g_i^{(k)} = \mathbb{E}[{\nabla}F(x^{(k)}_{i}\hspace{-0.4mm};\xi^{(k)}_{i})] = \nabla f_i(x_i^{(k)})$ be the full-batch gradient.
Note that both model $x_i$ and gradient $g_i$ are deterministic due to the removal of gradient noise. By substituting the momentum update and local model update into the partial averaging step in Algorithm \ref{Algorithm: DmSGD}, we can write DmSGD (see Appendix \ref{app-reform-dsgd}) into: 
\begin{align}\label{xsdhs7}
\vx^{(k+1)} &= \underbrace{W \Big( \vx^{(k)} - \gamma \nabla f(\vx^{(k)}) \Big)}_{\rm DSGD} + \underbrace{\beta \Big( \vx^{(k)} - W \vx^{(k-1)}\Big)}_{\rm momentum}.
\end{align}
Since the momentum term in DmSGD \eqref{xsdhs7} {\em cannot vanish} as $k$ increases, it will impose an extra inconsistency bias to DSGD. A simple numerical simulation on a full-batch linear regression problem confirms such conclusion. It is observed in Fig.~\ref{fig:ls-dsgd-dmsgd} that DmSGD converges faster but suffers from a larger bias than DSGD. 
\begin{figure}[t!]
    \centering
    \includegraphics[width=0.35\textwidth]{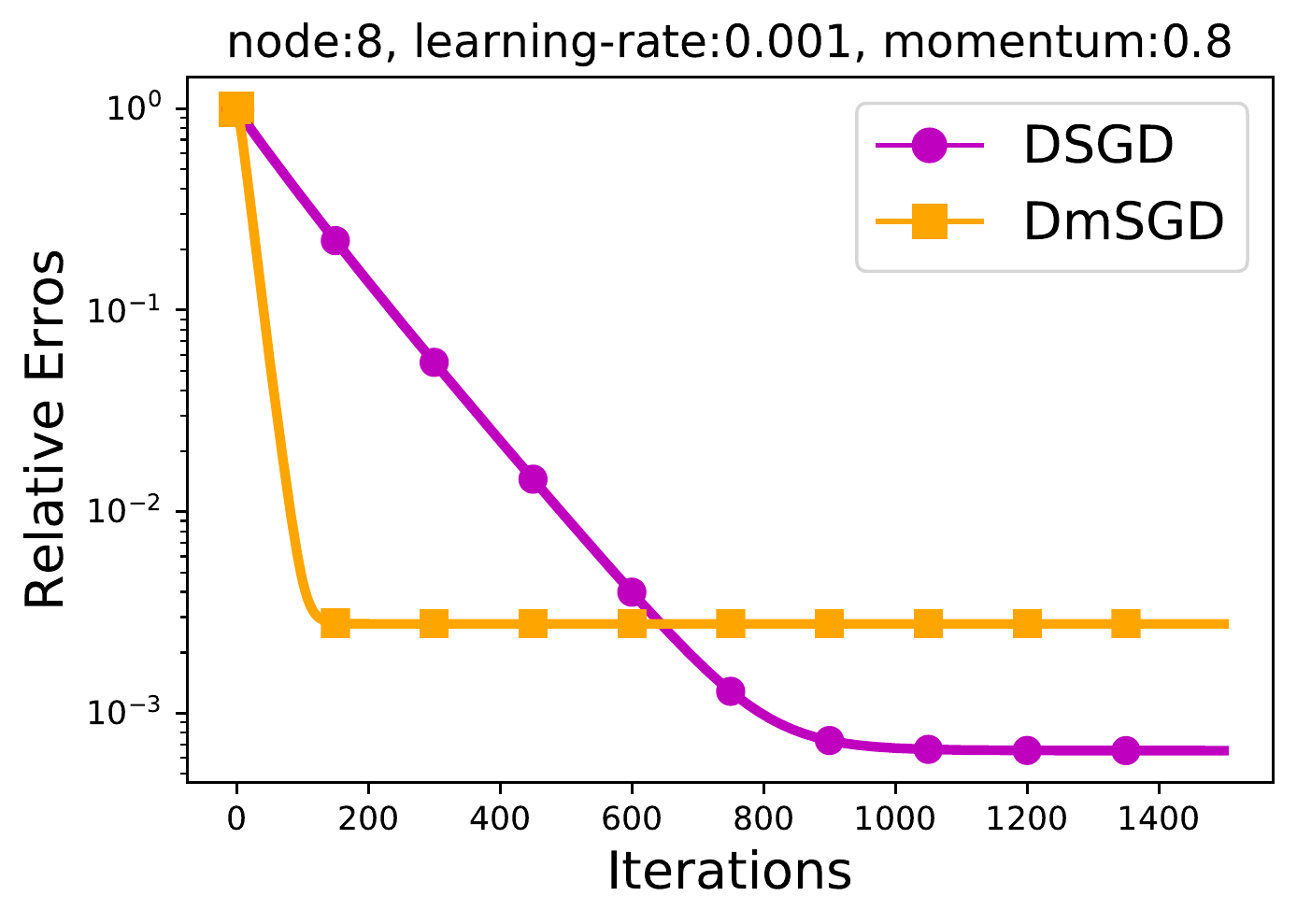}
    \caption{Convergence comparison between DSGD and DmSGD for a full-batch linear regression problem. Detailed experimental setting is in Appendix \ref{app-exp-linear-regression}.}
    \label{fig:ls-dsgd-dmsgd} \vspace{-2mm}
    \vspace{-2mm}
\end{figure}

\vspace{1mm}
\noindent \textbf{DmSGD's inconsistency bias: magnitude.}
The following proposition quantitatively evaluates the magnitude of DmSGD's inconsistency bias. We let $\rho = \max\{|\lambda_2(W)|,|\lambda_n(W)|\}$, which characterize how well the network is connected. Since $W$ is doubly stochastic (Assumption A.3), it holds that $\rho \in (0,1)$ (see Appendix \ref{app-pre}). A well connected network will have $\rho \to 0$.

\begin{proposition}\label{prop-dmsdg-bias}
Under Assumptions A.1 and A.3, if each $f_i(x)$ is further assumed to be strongly-convex, and the full-batch gradient $\nabla f_i(x)$ is accessed per iteration,  then DmSGD \eqref{xsdhs7} has the following inconsistency bias:
\begin{align}\label{DmSGD-hb}
\lim_{k\to \infty} \sum_{i=1}^n \|x_i^{(k)} - x^\star\|^2 = O\Big( \frac{\gamma^2 b^2}{(1-\beta)^2(1-\rho)^2} \Big),
\end{align}
where $b^2 = (1/n)\sum_{i=1}^n \|\nabla f_i(x^\star)\|^2$ denotes the data inconsistency between nodes, and $\beta$ is the momentum coefficient. (Proof is in Appendix \ref{app-inconsist-bias-dsgd})
\end{proposition}

\noindent 
Note that we do not take expectation over $\sum_{i=1}^n \|x_i^{(k)} - x^\star\|^2$ in \eqref{DmSGD-hb} because no gradient noise exists in recursion \eqref{xsdhs7}. Recall from Appendix \ref{app-limit-bias-dsgd} that DSGD has an inconsistency bias on the order of $O(\gamma^2 b^2/(1-\rho)^2)$. Comparing it with \eqref{DmSGD-hb}, it is observed that the momentum term in DmSGD is essentially amplifying the inconsistency bias by a margin of $1/(1-\beta)^2$. This can explain why DmSGD converges less accurate than DSGD as illustrated in Fig.~\ref{fig:ls-dsgd-dmsgd}. Noting that $\beta$ is typically set close to $1$ in practice, DmSGD can suffer from a significantly large inconsistency bias. Since inconsistency bias dominates the large-batch scenario (see Proposition \ref{prop-limit-bias}),  DmSGD is thus observed to have  severely deteriorated performance as illustrated in Table \ref{talbe:motivation}.

\section{Improving Inconsistency Bias: DecentLaM}\label{sec:decentLaM}
To improve large-batch DmSGD's accuracy, we have to reduce the influence of momentum on inconsistency bias.

\vspace{1mm}
\noindent \textbf{DSGD interprets as standard SGD.} Without loss of generality, we assume each node $i$ can access the accurate gradient $\nabla f_i(x)$ in DSGD recursions \eqref{dsgd-1}--\eqref{dsgd-2}. In this scenario, DSGD can be rewritten into a compact form:
\begin{align}\label{dsgd-compact}
\vx^{(k+1)} = W\big(\vx^{(k)} - \gamma \nabla f(\vx^{(k)})\big).
\end{align}
To simplify the derivation, we assume the weight matrix $W$, in addition to satisfying Assumption A.3, is positive-definite. With this assumption, $W$ can be eigen-decomposed into $W = U\Lambda U^T$ where $U$ is an orthogonal matrix and $\Lambda$ is a positive definite matrix. We define $W^{\frac{1}{2}} = U \Lambda^{\frac{1}{2}}U^T$ so that $W^{\frac{1}{2}}$ is also positive-definite and $W = W^{\frac{1}{2}}\times W^{\frac{1}{2}}$. We next introduce $\vx = W^{\frac{1}{2}} \bs$ and hence $\bs = W^{-\frac{1}{2}} \vx$. Left-multiplying $W^{-\frac{1}{2}}$ to both sides of \eqref{dsgd-compact}, 
\begin{align}
\bs^{(k+1)} 
&= W\bs^{(k)} - \gamma W^{\frac{1}{2}} \nabla f(W^{\frac{1}{2}} \bs^{(k)}) \nonumber \\
&= W\bs^{(k)} - \gamma \nabla_{\bs} f(W^{\frac{1}{2}} \bs^{(k)}) \nonumber \\
&= \bs^{(k)} - \gamma \big( \nabla_{\bs} f(W^{\frac{1}{2}} \bs^{(k)}) + \frac{1}{\gamma}(I-W)\bs^{(k)} \big) \label{dsgd-s-recursion}
\end{align}
In other words, DSGD \eqref{dsgd-compact} (with constant $\gamma$) can be interpreted as a standard SGD \eqref{dsgd-s-recursion} (in terms of $\bs$) to solve:
\begin{align}\label{s-prob}
\min_{\bs} \quad f(W^{\frac{1}{2}}\bs) + \frac{1}{2\gamma} \|\bs\|^2_{I-W},
\end{align}
and $\vx^{(k)}$ can be achieved by $\vx^{(k)} = W^{\frac{1}{2}}\bs^{(k)}$ per iteration. When recursion \eqref{dsgd-s-recursion} approaches the solution $\bs^o$ to problem \eqref{s-prob}, DSGD \eqref{dsgd-compact} will achieve the fixed point $\vx^o = W^{\frac{1}{2}}\bs^{o}$. The inconsistency bias of DSGD is essentially the distance between $\vx^o$ and $\vx^\star = [x^\star,\cdots, x^\star]\in \RR^{n\times d}$ where $x^\star$ is the global solution to problem \eqref{dist-opt}. Since  $\vx^o\neq \vx^\star$, such inconsistency bias generally exists in DSGD. 

\begin{figure}[t!]
    \centering
    \includegraphics[width=0.35\textwidth]{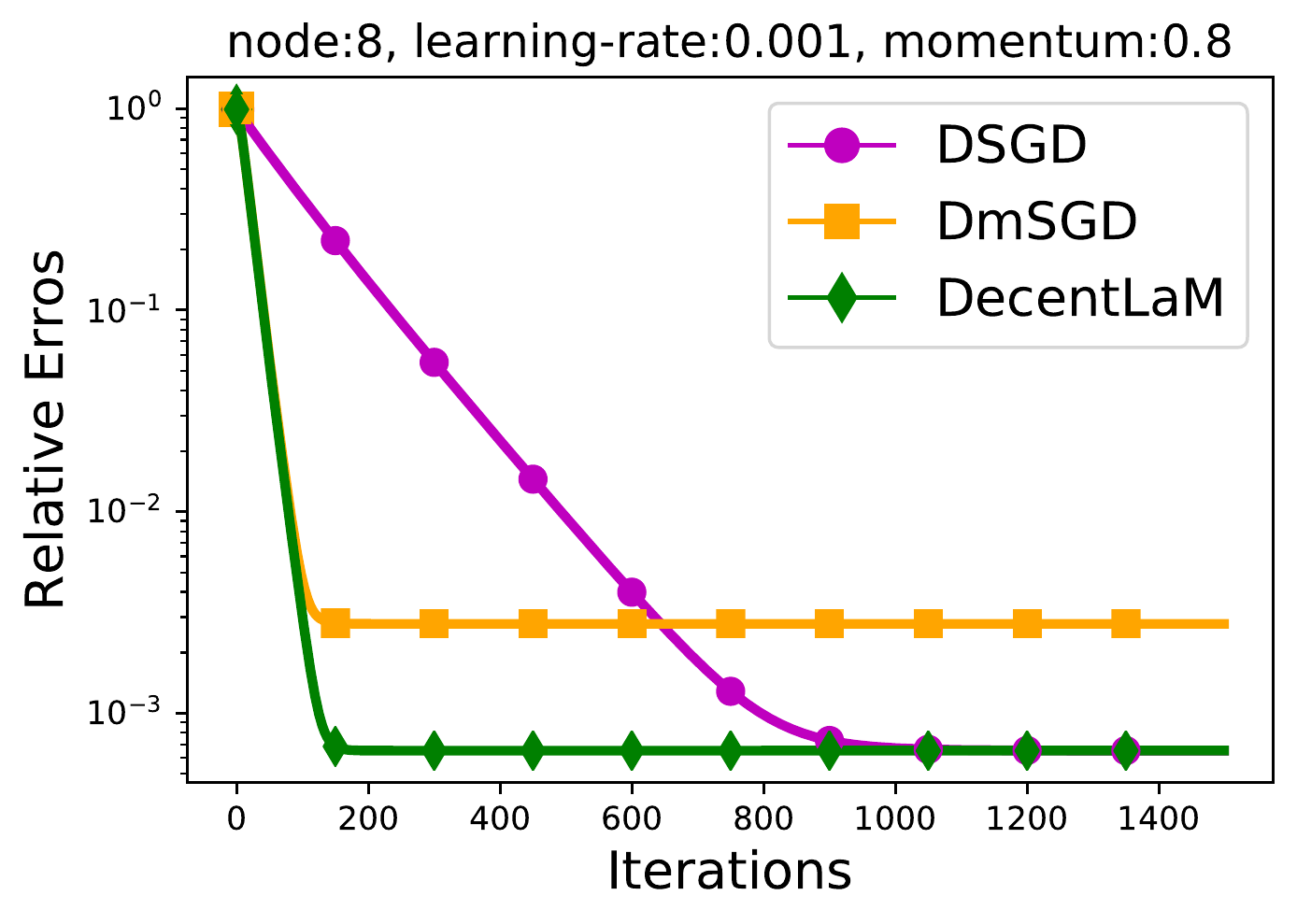}
    \caption{Convergence comparison between DSGD, DmSGD and decentLaM for a full-batch linear regression problem.}
    \label{fig:ls-dsgd-dmsgd-decentLaM} \vspace{-2mm}
\end{figure}

\begin{remark}
The interpretation of DSGD \eqref{dsgd-compact}, which is  usually named as the adaptation-then-combination (ATC) version of DSGD, as standard SGD \eqref{dsgd-s-recursion} is novel. Existing literatures \cite{yuan2016convergence,balu2020decentralized} can bridge the adaptation-with-combination (AWC) version of DSGD, i.e., $\vx^{(k+1)} = W\vx^{(k)} - \gamma \nabla f(\vx^{(k)})$, with standard SGD, but not for ATC-DSGD. ATC-DSGD can enable larger learning rates and has a better convergence performance than AWC-DSGD \cite[Secs.~10.6 and 11.5]{sayed2014adaptation}, \cite{yuan2018exact2,li2017decentralized}, and is hence more widely used in deep training  \cite{lian2018asynchronous,assran2019stochastic,tang2018d,yu2019linear,gao2020periodic,singh2020squarm,lin2021quasi,koloskova2019decentralized}.
\end{remark}

\vspace{1mm}
\noindent \textbf{DecentLaM development.} 
It is established in Proposition \ref{prop-dmsdg-bias} that the momentum term in DmSGD can significantly amplify the inconsistency bias especially when $\beta \to 1$. 
This section proposes DecentLaM to remove the negative influence of momentum in DmSGD.  

Since DSGD can be interpreted as a standard SGD, we can easily integrate it to the existing momentum acceleration techniques  \cite{loizou2020momentum,yuan2016influence,gitman2019understanding,liu2020improved,yu2019linear}. By letting $\vg_s^{(k)}$ be the standard gradient of problem \eqref{s-prob} at iteration $k$, the standard momentum SGD approach to solving problem \eqref{s-prob} is 
\begin{align}
\vg_s^{(k)} &=  \nabla_{\bs} f(W^{\frac{1}{2}} \bs^{(k)}) + \frac{1}{\gamma}(I-W)\bs^{(k)} \label{decentLaM-s-1}\\
\vm^{(k+1)}_{s} &= \beta \vm^{(k)}_{s} + \vg_s^{(k)} \label{decentLaM-s-2}\\
\bs^{(k+1)} &= \bs^{(k)} - \gamma \vm^{(k+1)}_{s} \label{decentLaM-s-3}
\end{align}
After achieving $\bs^{(k+1)}$, we get $\vx^{(k+1)} = W^{\frac{1}{2}} \bs^{(k+1)}$. Recursions \eqref{decentLaM-s-1}--\eqref{decentLaM-s-3} are  named as DecentLaM, and we simulate their convergence behaviour under the same setting as in Fig.~\ref{fig:ls-dsgd-dmsgd}. It is observed in  Fig.~\ref{fig:ls-dsgd-dmsgd-decentLaM} that the proposed algorithm converges as fast as DmSGD but to a more accurate solution. Since DecentLaM is with an improved inconsistency bias, it better fits into  the large-batch scenario.

\vspace{1mm}
\noindent \textbf{Efficient implementation.} To apply DecentLaM \eqref{decentLaM-s-1}-\eqref{decentLaM-s-3} to train deep neural networks,  we need two additional modifications: (i) the full-batch gradient $\nabla f_i(x)$ needs to be replaced with the stochastic gradient $\nabla F(x;\xi_i)$; (ii) the matrix $W^{\frac{1}{2}}$ should be removed from the algorithm update. To this end, by letting  $\vg=W^{\frac{1}{2}}\vg_{s}$, $\vm = W^{\frac{1}{2}}\vm_{s}$, and $\tilde{\vg}$ be the stochastic version of $\vg$, \eqref{decentLaM-s-1}-\eqref{decentLaM-s-3} can be rewritten as 
\begin{align}
\tilde{\vg}^{(k)} &= W \nabla F(\vx^{(k)};\bxi^{(k)}) + \frac{1}{\gamma}(I-W)\vx^{(k)} \label{decentLaM-x-1}\\
\vm^{(k+1)} &= \beta \vm^{(k)} + \tilde{\vg}^{(k)} \label{decentLaM-x-2}\\
\vx^{(k+1)} &= \vx^{(k)} - \gamma \vm^{(k+1)} \label{decentLaM-x-3}
\end{align}
where we have replaced $\nabla f(\vx^{(k)})$ with $\nabla F(\vx^{(k)};\bxi^{(k)})$ in \eqref{decentLaM-x-1} to enable the usage of stochastic gradient. The above  algorithm is almost the same as the vanilla momentum SGD within a single computing node, with the exception to construct $\vg^{(k)}$ as in \eqref{decentLaM-x-1}. The decentralized implementation of \eqref{decentLaM-x-1}-\eqref{decentLaM-x-3} is listed in Algorithm \ref{Algorithm: DecentLaM}, where 
\begin{align}\label{DecentLaM-gi}
\tilde{g}_i^{(k)} \hspace{-0.8mm}=\hspace{-0.8mm} \frac{1}{\gamma} x_i^{(k)} \hspace{-0.8mm}-\hspace{-0.8mm} \frac{1}{\gamma} \sum_{j\in\mathcal{N}_i}w_{ij}\big( x_j^{(k)} \hspace{-0.8mm}-\hspace{-0.8mm} \gamma \nabla F(x_j^{(k)}; \xi_j^{(k)})\big).
\end{align}
The momentum update and local model update in Algorithm \ref{Algorithm: DecentLaM} can be conducted via the mSGD optimizer provided by PyTorch or TensorFlow. As shown in Fig.~\ref{fig:workflow-decentLaM}, the dashed part, which is different from the DmSGD workflow, enables an efficient wait-free backpropagation (WFBP) implementation that can overlap communication and computation.

\begin{algorithm}[t!]
 	\DontPrintSemicolon
 	\KwRequire{Initialize $\gamma$, $x^{(0)}_{i}$; let  $m^{(0)}_{i}\hspace{-0.3mm}=\hspace{-0.3mm}0, \beta\hspace{-0.3mm}\in\hspace{-0.3mm}(0,1)$}

    \vspace{1mm}
 	\For{$k=0, 1,2,...,T-1$, every node $i$}{
 		Sample $\xi^{(k)}_{i}$ and update $\tilde{g}_i^{(k)}$ according to \eqref{DecentLaM-gi} \;
 		
 		$m^{(k+1)}_i = \beta m^{(k)}_{i} + \tilde{g}_i^{(k)} \hspace{0.6cm} \triangleright \mbox{\footnotesize{momentum update}}$ \;
 		
 		$x^{(k+1)}_i = x^{(k)}_{i} - \gamma m^{(k+1)}_i \hspace{0.35 cm} \triangleright \mbox{\footnotesize{local model update}}$ \;
 	}
 	\caption{DecentLaM}
 	\label{Algorithm: DecentLaM}
 \end{algorithm}



\begin{figure}[b!]
    \centering
    \includegraphics[width=0.5\textwidth]{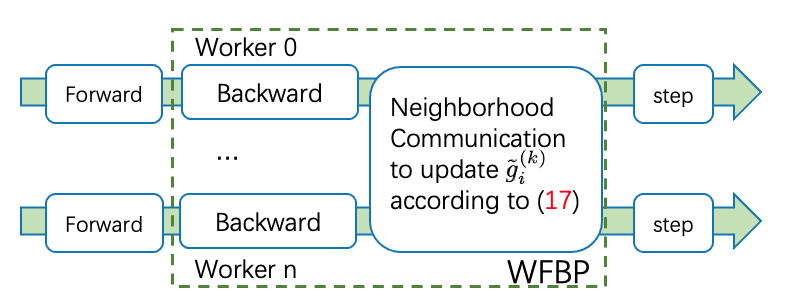}
    \caption{The workflow of DecentLaM in a training iteration.}
    \label{fig:workflow-decentLaM} \vspace{-2mm}
\end{figure}
 
\vspace{1mm}
\noindent \textbf{DecentLaM's inconsistency bias.} The following proposition quantitatively evaluates the magnitude of DecentLaM's inconsistency bias (Proof is in Appendix \ref{app-inconsist-bias-decentLaM}):
\begin{proposition}\label{prop-decentlam-bias}
Under the same assumptions as Proposition \ref{prop-dmsdg-bias}, DecentLaM has an inconsistency bias as follows:
\begin{align}\label{DecentLaM-hb}
\lim_{k\to \infty} \sum_{i=1}^n \|x_i^{(k)} - x^\star\|^2 = O\Big(\frac{\gamma^2 b^2}{(1-\rho)^2}\Big),
\end{align}
\end{proposition}
\begin{remark}\label{remark-decentLam-1}
Comparing with DmSGD's inconsistency bias \eqref{DmSGD-hb}, it is observed that DecentLaM completely removes the negative effects of momentum; it improves DmSGD's inconsistency bias by a margin of $1/(1-\beta)^2$. When $b^2$ is large or $\beta$ is close to $1$, such improvement is remarkable. 
\end{remark}
\begin{remark}\label{remark-decentLam-2}
Comparing with DSGD's inconsistency bias $O(\gamma^2 b^2/(1-\rho)^2)$, it is observed that DecentLaM has exactly the same inconsistency bias as DSGD. However, the momentum term in DecentLaM will significantly speed up its convergence rate. The simulation results in Fig.~\ref{fig:ls-dsgd-dmsgd-decentLaM} are consistent with the conclusions discussed in Remarks \ref{remark-decentLam-1}-\ref{remark-decentLam-2}.
\end{remark}

\section{Convergence analysis of DecentLaM}
\label{sec:DecentLaM_convergence}
Sec.~\ref{sec:decentLaM} proposes a novel algorithm DecentLaM and explains why it better fits into the large-batch scenario. This section establishes the formal convergence analysis in the non-convex and strongly-convex scenarios, respectively. 

\subsection{Non-convex scenario}
We introduce a standard data inconsistency assumption for the non-convex scenario:

\vspace{1mm}
\noindent \textbf{Assumption A.4} There exits a constant $\hat{b} > 0$ such that $\frac{1}{n}\sum_{i=1}^n\|\nabla f_i(x) - \nabla f(x)\|^2 \le \hat{b}^2$ for any $x$, where $f(x)$ is the global cost function defined in \eqref{dist-opt}.

\vspace{1mm}
\noindent Assumption A.4 is much more relaxed than the commonly used assumption that each $\nabla f_i(x)$ is upper bounded (i.e., $\|\nabla f_i(x)\| \le M$) \cite{gao2020periodic,singh2020squarm,balu2020decentralized}. When  data distribution $D_i$ in each computing node $i$ is identical to each other, it holds that $f_i(x) = f_j(x)$ for any $i$ and $j$, which implies $\hat{b}=0$. 

\begin{theorem}\label{thm-decentLaM--nc}
Under Assumptions A.1--A.4, if learning rate $\gamma =  \min\{\frac{(1-\beta)^2}{5\sqrt{\beta+\beta^2}L}, \frac{(1-\beta)^2}{(5-\beta+2\beta^2)L}, \frac{(1-\beta)^2}{12L\beta^2}, \frac{1-\rho}{20\sqrt{\rho}L}\}$, $W$ is positive-definite, and $\beta + \frac{16\beta^2}{(1-\beta)(1-\rho)^2}  \le \frac{3+\rho}{4}$,  the DecentLaM algorithm in Algorithm \ref{Algorithm: DecentLaM} will converge as follows  (Proof is in Appendix \ref{app-section-non-convex}, and $\bar{x}^{(k)} = \frac{1}{n}\sum_{i=1}^n x_i^{(k)}$):
\begin{align}\label{decentLaM-conv-nc}
& \frac{1}{T}\sum_{k=0}^{T-1}\hspace{-0.5mm}\mathbb{E}\|\frac{1}{n}\sum_{i=1}^n \nabla f_i(\bar{x}^{(k)})\|^2  \\
=& O\Big( \underbrace{\frac{1-\beta}{\gamma T}}_{\rm convg.\ rate} + \underbrace{\frac{\gamma \sigma^2}{n(1-\beta)} + \frac{ \gamma^2\sigma^2}{1-\rho}}_{\rm sto.\ bias} + \underbrace{\frac{ \gamma^2 \hat{b}^2}{(1-\rho)^2}}_{\rm inconsist. bias}\Big) \nonumber
\end{align}
\end{theorem}
Expression \eqref{decentLaM-conv-nc} indicates that DecentLaM will converge sublinearly to a limiting bias. The limiting bias (which can be achieved by letting $T\to \infty$) can be separated into the stochastic bias and the inconsistency bias, which is consistent with our discussion in Sec.~\ref{sec:DmSG-bias}. It is also observed that the inconsistency bias term in \eqref{decentLaM-conv-nc} is unrelated to the momentum coefficient $\beta$, which implies, as discussed in Sec.~\ref{sec:decentLaM}, that DecentLaM has removed the negative affects of momentum. 
Theorem~\ref{thm-decentLaM--nc} has restrictions on the positive-definiteness of $W$ and the momentum parameter $\beta$ so that the analysis can be simplified. However, they are not needed in the real practice and not used in any of our experiments in Sec.~\ref{sec:experiment}. We find the concurrent work \cite{lin2021quasi} also imposes a restriction on momentum parameter to simply the analysis. 

\begin{corollary}\label{coro-decentLaM-nc}
Under the same assumptions as in  Theorem~\ref{thm-decentLaM--nc}, if learning rate $\gamma = O((1-\beta)/\sqrt{T/n})$, DecentLaM will converge at rate $O(1/\sqrt{n T})$ (Proof is in Appendix \ref{app-section-non-convex}), which achieves the same linear speedup as PmSGD.
\end{corollary}

\subsection{Strongly-convex scenario}

\noindent \textbf{Assumption A.5} Each $f_i(x)$ is $\mu$-strongly convex, i.e., $\langle \nabla f_i(x) - \nabla f_i(y), x- y \rangle \ge \mu \|x-y\|^2$ for any $x,y\in \mathbb{R}^d$.

\vspace{1mm}
Under Assumption A.5, the global solution to problem \eqref{dist-opt} is unique. Furthermore, we do need to make the data heterogeneity assumption A.4 in the strongly-convex scenario. Instead, we will use $b^2 = \frac{1}{n}\sum_{i=1}^n \| \nabla f_i(x^\star)\|^2$ to gauge the data heterogeneity

\begin{theorem}\label{thm-decentLaM--sc}
Under Assumption A.1--A.3 and A.5, if $\gamma =  \min\{\frac{(1-\beta)^2}{27L},\frac{(1-\rho)(1-\beta)(\kappa+1)}{5L},\frac{1-\rho}{\sqrt{1728(\kappa+1)}}\}$, $W$ is  positive definite, and $\beta + \frac{16\beta^2}{(1-\beta)(1-\rho)^2}  \le \frac{3+\rho}{4}$,  DecentLaM in Algorithm \ref{Algorithm: DecentLaM} will converge as (Proof is in Appendix \ref{app-thm-decentLaM--sc}):
\begin{align}\label{decentLaM-conv-sc}
&\ \frac{1}{n H_T}\sum_{k=0}^{T} \sum_{i=1}^n h_k(\mathbb{E} f_i(\bar{x}^{(k)}) - f^*)\nonumber \\
=& O\Big(\underbrace{\frac{(1-\beta)}{\gamma}(1-\frac{\gamma}{(1-\beta)})^{T}}_{\rm convg.\ rate} \nonumber \\
&\quad +\underbrace{\frac{\gamma \sigma^2}{n(1-\beta)} + \frac{\gamma^2 \sigma^2}{1-\rho}}_{\rm sto.\ bias} +\underbrace{\frac{\gamma^2 {b}^2}{(1-\rho)^2}}_{\rm inconsist.bias}\Big)
\end{align}
where $h_k$ is some positive weight (see Appendix \ref{app-thm-decentLaM--sc}) and $H_T = \sum_{k=0}^T h_k$. Quantity $f^\star = \min_{x} f(x)$.
\end{theorem}
\noindent 
The inconsistency bias in \eqref{decentLaM-conv-sc} is independent of the momentum $\beta$, which is consistent with Proposition \ref{prop-decentlam-bias} . It implies that DecentLaM achieves the same  $O(\gamma^2b^2/(1-\rho)^2)$ inconsistency bias as DSGD (see \eqref{dsgd-convergence}).

Theorem~\ref{thm-decentLaM--sc} establishes the convergence of DecentLaM with a constant learning rate. When learning rate is decaying, DecentLaM converges  exactly to the global solution:
\begin{corollary}\label{coro-sc}
Under the same assumptions as Theorem~\ref{thm-decentLaM--sc}, if learning rate $\gamma = {O}(\frac{(1-\beta)\ln(nT^2)}{T})$, DecentLaM will converge as follows (Proof is in Appendix \ref{app-thm-decentLaM--sc}):
\begin{align}\label{decentLaM-conv-sc-2}
\frac{1}{n H_T}\sum_{k=0}^{T} \sum_{i=1}^n h_k(\mathbb{E} f_i(\bar{x}^{(k)}) - f^*) = \tilde{O}\Big( \frac{1}{nT} \Big)
\end{align}
where $\tilde{O}(\cdot)$ hides all logarithm factors. 
\end{corollary}

\noindent Corollary~\ref{coro-sc} implies that DecentLaM achieves the same linear speedup as PmSGD in the strongly-convex scenario \cite{liu2020improved}. However, it saves more communication per iteration due to the partial averaging (see Sec.~\ref{sec:DmSGD}).

\subsection{Comparison with existing methods}
As we discussed in Sec.~\ref{sec:DmSG-bias}, the  inconsistency bias will dominate the convergence performance as batch-size gets large. For this reason, we list the inconsistency bias comparison between DecentLaM and existing DmSGD variants in Table \ref{table:algoriothm-comparison}. By removing the influence of momentum, DecentLaM has the smallest inconsistency bias. This implies DecentLaM can perform better in the large-batch scenario.   

\begin{table}[h]
\centering
\begin{tabular}{rcc}
\toprule
\multicolumn{1}{l}{} & \textbf{Strongly-convex} & \textbf{Non-convex}                           \\ \midrule
DmSGD \cite{gao2020periodic}                & N.A.                                                               & $O\big(\frac{\gamma^2 M^2}{(1-\beta)^2}\big)$ \vspace{0.5mm}\\
DmSGD \cite{singh2020squarm}                & $O\big(\frac{\gamma^{5/2} M^2}{(1-\beta)^6}\big)$                                                               & $O\big(\frac{\gamma^2 M^2}{(1-\beta)^4}\big)$ \vspace{0.5mm}\\
DmSGD  (eq.\eqref{DmSGD-hb})              & $O\big(\frac{\gamma^2 b^2}{(1-\beta)^2}\big)$  & N.A                                           \\
DA-DmSGD \cite{yu2019linear}             & N.A.                                                               & $O\big(\frac{\gamma^2 \hat{b}^2}{(1-\beta)^2}\big)$ \vspace{0.5mm}\\
AWC-DmSGD \cite{balu2020decentralized}           & $O\big(\frac{\gamma^2 M^2}{(1-\beta)^2}\big)$                      & $O\big(\frac{\gamma^2 M^2}{(1-\beta)^4}\big)$ \vspace{0.5mm}\\
SlowMo \cite{wang2019slowmo}           & N.A                     & N.A   \vspace{1.5mm}\\
\textbf{DecentLaM (Ours)}            & $\boldsymbol{O(\gamma^2 b^2)}$                                                  & $\boldsymbol{O(\gamma^2 \hat{b}^2)}$                             \\ \bottomrule
\end{tabular}
\vspace{2mm}
\caption{Inconsistency bias comparison between various decentralized momentum algorithms. Constant $M$ is the gradient's upper bound, which is typically much larger than the data inconsistency $b$ or $\hat{b}$. SlowMo \cite{wang2019slowmo} only examined the data-homogeneous scenario. It did not clarify what the inconsistency bias is. }
\label{table:algoriothm-comparison}
\end{table}

After completing this work, we become aware of the concurrent work QG-DmSGD \cite{lin2021quasi} that can also achieve $O(\gamma^2 \hat{b}^2)$ for non-convex scenario. However, QG-DmSGD is based on a strategy to mimic the global momentum, which is different from the core idea to develop DecentLaM. Furthermore, QG-DmSGD is designed for highly data-heterogeneous scenario in EdgeAI applications while DecentLaM is for large-batch training within data-centers. As a result, our experimental tasks and algorithm settings are very different. In addition, our analysis is also different from QG-DmSGD. With a delicately designed Lyapunov function, we can cover the convergence analysis for both non-convex and strongly-convex scenarios. In contrast, the analysis in QG-DmSGD is for the non-convex scenario. Finally, we theoretically uncovered the reason why traditional DmSGD \cite{yu2019linear,gao2020periodic,singh2020squarm} has degraded performance in Sec.~\ref{sec:DmSG-bias} while QG-DmSGD only justifies it empirically. 

\begin{figure*}[t!]
\centering
\includegraphics[width=0.245\textwidth]{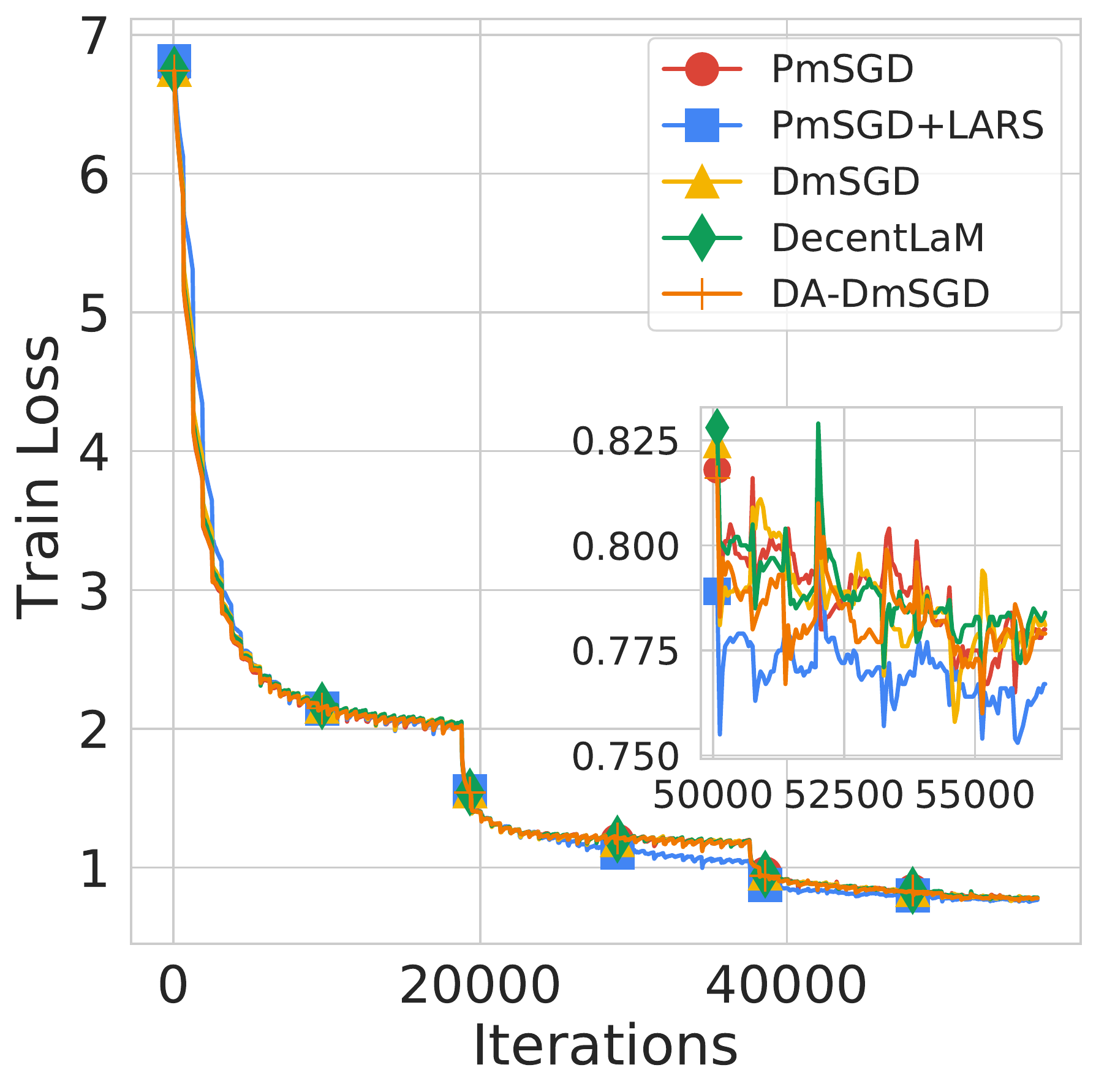} 
\includegraphics[width=0.245\textwidth]{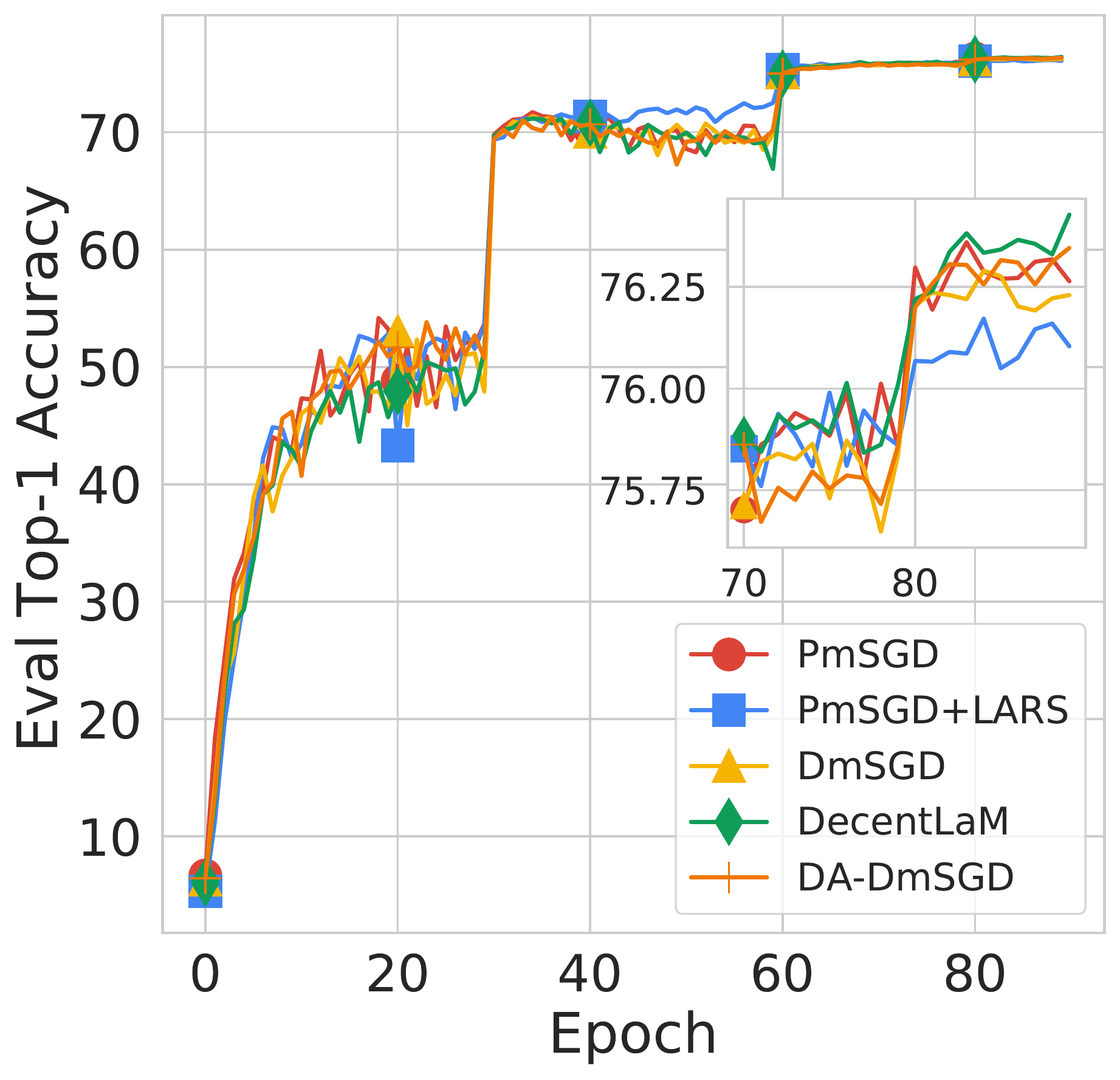} 
\includegraphics[width=0.245\textwidth]{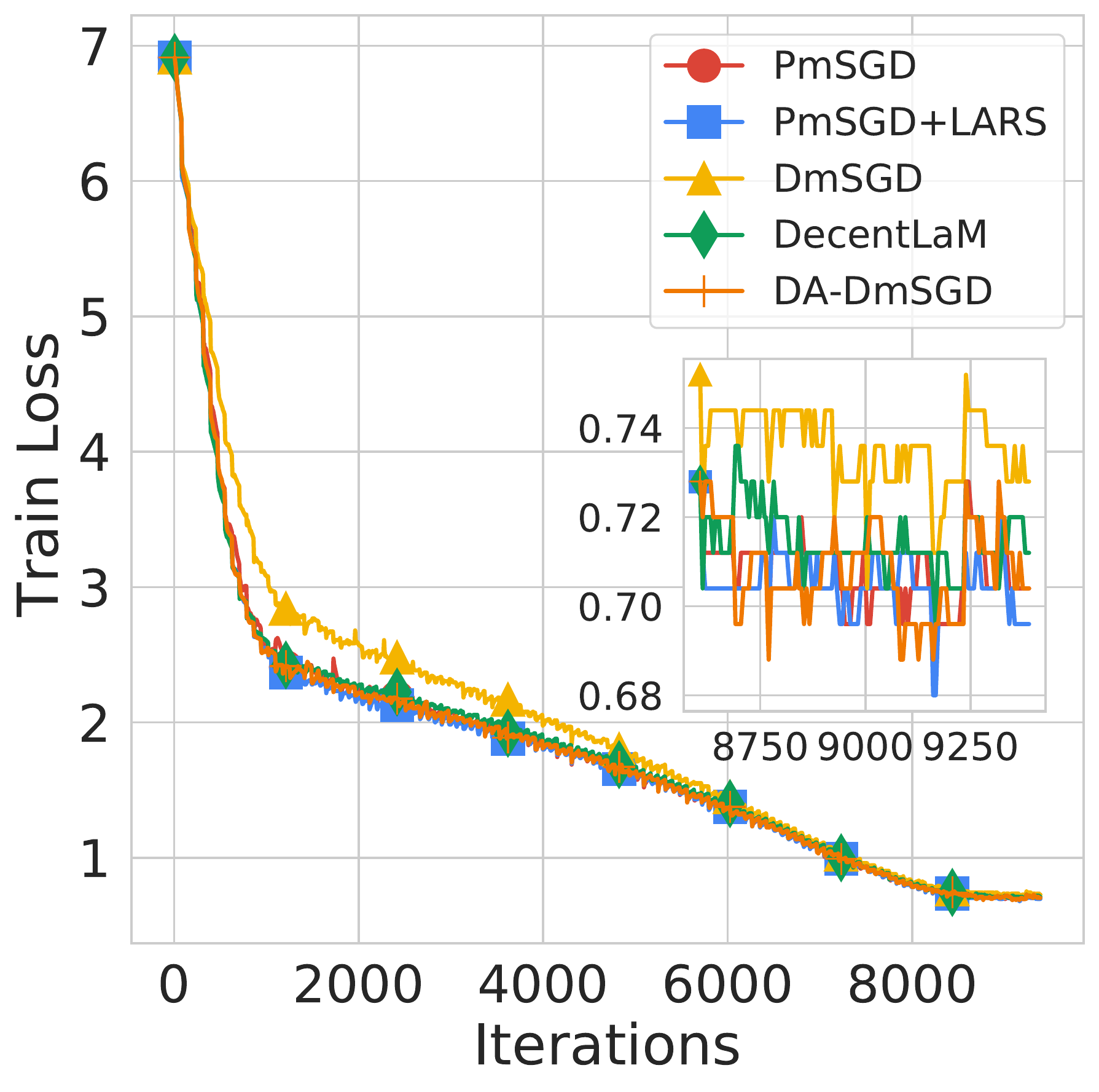} 
\includegraphics[width=0.245\textwidth]{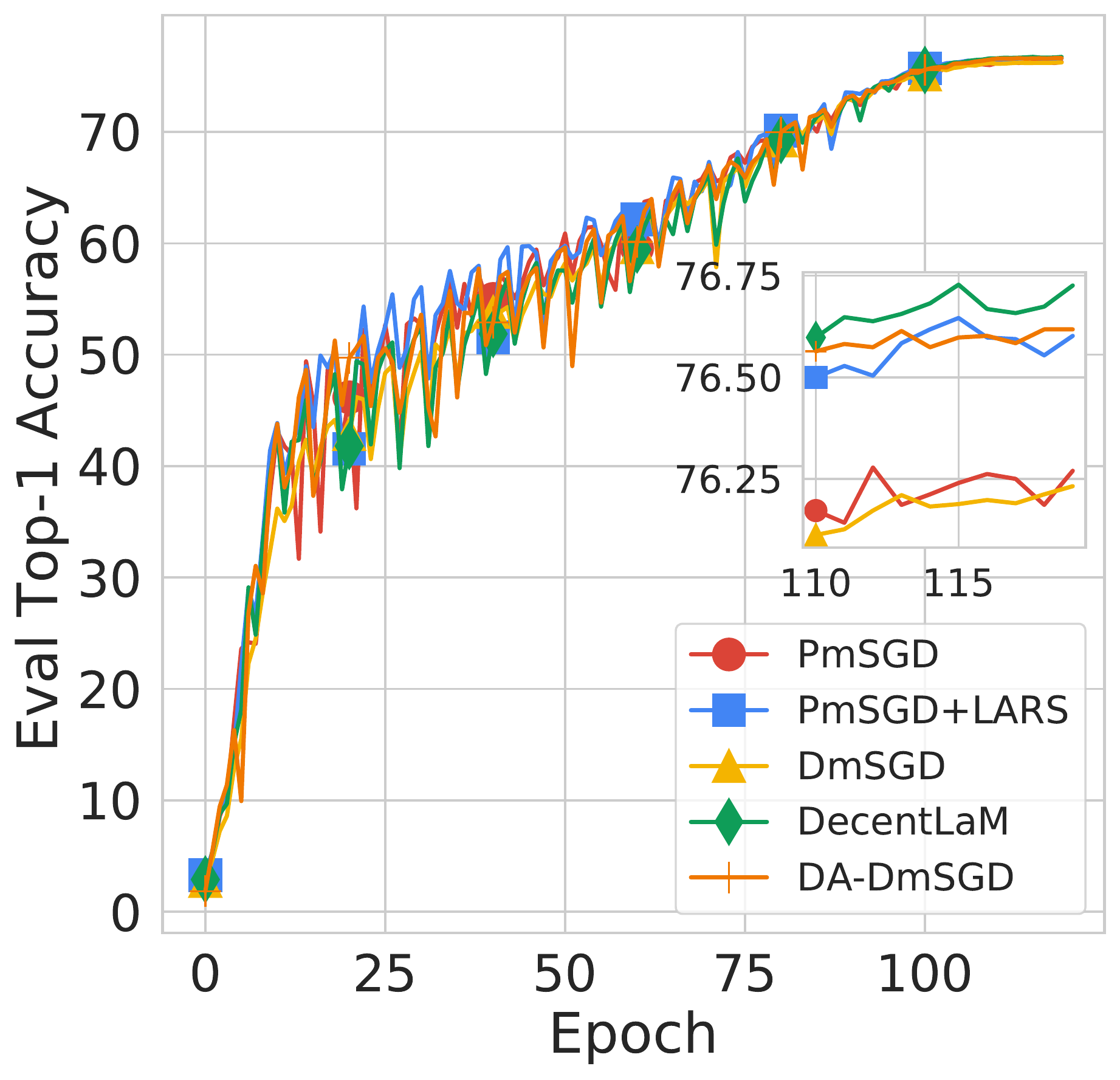}

\caption{\small Convergence results on the ImageNet with respect to training loss and validation top-1 accuracy. The batch size is 2k for left-side two figures and 16k for right-side two figures. For clarity, we do not depict the curves of other baselines.}
\label{Fig:Imagenet-16k}
\end{figure*}

\section{Experiments}
\label{sec:experiment}

In this section, we systematically compare the proposed method, DecentLaM, with well-known state-of-the-art methods on typical large-scale computer vision tasks: image classification and object detection. In particular, we compare DecentLaM with the following algorithms. \textbf{PmSGD} is the standard Parallel momentum SGD algorithm in which a global synchronization across all nodes are required per iteration. \textbf{PmSGD + LARS} exploits a layer-wise adaptive rate scaling strategy \cite{you2017large} to boost  performance with large batch-size. \textbf{DmSGD} \cite{assran2019stochastic,gao2020periodic,singh2020squarm} exploits partial averaging to save communications. Its recursions are listed in Algorithm \ref{Algorithm: DmSGD}. \textbf{DA-DmSGD} \cite{yu2019linear} incurs additional partial averaging over the momentum to increase stability; it has double partial averages per iteration. \textbf{AWC-DmSGD} \cite{balu2020decentralized} accelerates AWC-DSGD with momentum. \textbf{SlowMo} \cite{wang2019slowmo} periodically synchronize each node and adds an additional slow momentum update. \textbf{QG-DmSGD} \cite{lin2021quasi} is a concurrent work to DecentLaM that adds a quasi-global momentum to DSGD. We utilize the QG-DmSGD Heavy-ball variant. \textbf{D$^2$} \cite{tang2018d} is a primal-dual approach that can completely remove the inconsistency bias theoretically. The vanilla D$^2$ algorithm does not perform well in our experiments. We refer to another form of D$^2$ in \cite{yuan2020influence} and add momentum acceleration to the local update step. We name this algorithm as \textbf{D$^2$-DmSGD}. All baseline algorithms are tested with the recommended hyper-parameters in their papers. Note that DecentLaM requires {\em symmetric} (but not necessarily positive definite) weight matrix to guarantee empirical convergence. 




We implement all the aforementioned algorithms with PyTorch \cite{paszke2019pytorch} 1.6.0 using NCCL 2.8.3 (CUDA 10.2) as the communication backend. For PmSGD, we used PyTorch's native Distributed Data Parallel (DDP) module. For the implementation of decentralized methods, we utilize \textbf{BlueFog} \cite{bluefog}, which is a high-performance decentralized deep training framework, to facilitate the topology organization, weight matrix generation, and efficient partial averaging. We also follow DDP's design to enable computation and communication overlap. Each server contains 8 V100 GPUs in our cluster and is treated as one node. The inter-node network fabrics are 25 Gbps TCP as default, which is a common distributed training platform setting. To eliminate the effect of topology with different size in decentralized algorithms, we used 8 nodes (i.e. $8\times 8=64$ GPUs) in all decentralized training and changed the batch size of every single GPU respectively.

\subsection{Image Classification} 

\begin{table}[t!]
\vskip 0.15in
\begin{center}
\begin{small}
\begin{sc}
\begin{tabular}{ccccc}
\toprule
    \multirow{2}{*}{method} & \multicolumn{4}{c}{Batch Size}  \\
     & 2k    & 8k        & 16k & 32k    \\ 
\midrule
PmSGD & 76.32 & 76.08 & 76.27 & 75.27   \\ 
PmSGD+LARS\cite{you2017large}  & 76.16 & 75.95 & 76.65 & 75.63 \\
DmSGD\cite{assran2019stochastic,gao2020periodic,singh2020squarm} & 76.27 & 76.01 & 76.23 & 74.97 \\
DA-DmSGD\cite{yu2019linear} & 76.35 & $\mathbf{76.19}$ & 76.62 & 75.51 \\
AWC-DmSGD\cite{balu2020decentralized} & 76.29 & 75.96 & 76.31 & 75.37 \\
SlowMo\cite{wang2019slowmo} & 76.30 & 75.47 & 75.53 & 75.33 \\
QG-DmSGD\cite{lin2021quasi} & 76.23 & 75.96 & 76.60 & 75.86 \\
D$^2$-DmSGD\cite{tang2018d} & 75.44 & 75.30 & 76.16 & 75.44 \\
DecentLaM (Ours) & $\mathbf{76.43}$ & $\mathbf{76.19}$ & $\mathbf{76.73}$ & $\mathbf{76.22}$ \\
\bottomrule
\end{tabular}
\end{sc}
\end{small}
\end{center}
\vskip -0.1in
\caption{Top-1 validation accuracy of aforementioned methods when training ResNet-50 model with different batch sizes.}
\vskip -0.2in
\label{table-imagenet-resnet50-acc}
\end{table}

\noindent \textbf{Implementation.} We conduct a series of image classification experiments with the ImageNet-1K \cite{deng2009imagenet} dataset, which consists of 1,281,167 training images and 50,000 validation images in 1000 classes. We train classification models 
with different batch sizes to verify our theoretical findings. As suggested in \cite{goyal2017accurate}, we treat batch size equal or less than 8k as small-batch setting and the training protocol in \cite{goyal2017accurate} is used. In details, we train total 90 epochs. The learning rate is warmed up in the first 5 epochs and is decayed by a factor of 10 at 30, 60 and 80 epochs. For large-batch setting (batch size larger than 8K), we train total 120 epochs. The learning rate is warmed up in the first 20 epochs and is decayed in a cosine annealing scheduler as suggested in \cite{you2019large}. The choice of hyper-parameter setting is not cherry-picked and we try to keep our PmSGD's baseline around 76\% while \cite{goyal2017accurate}'s setting has severe performance drop even for large-batch PmSGD. The momentum SGD optimizer is  with linear scaling by default. 
Experiments are trained in the mixed precision using Pytorch native amp module. 

\begin{table*}[t]
\begin{center}
\begin{small}
\begin{sc}
\setlength{\tabcolsep}{1.8mm}{
\begin{tabular}{cccccccccccccccc}
\toprule
    model & \multicolumn{3}{c}{ResNet-18} & \multicolumn{3}{c}{ResNet-34} & \multicolumn{3}{c}{ResNet-50} & \multicolumn{3}{c}{MobileNet-v2} & \multicolumn{3}{c}{EfficientNet} \\
    Batch Size & 2k & 8k & 16k & 2k & 8k & 16k & 2k & 8k & 16k & 2k & 8k & 16k & 2k & 8k & 16k \\
    \midrule
    PmSGD & 70.0 & 69.5 & 68.3 & $\mathbf{73.8}$ & 72.8 & 72.9  & 76.3 & 76.1 & 76.3 & 70.6 & 71.1 & 69.5 & 77.6 & 77.0 & 78.1 \\ 
    PmSGD+LARS & 70.2 & 69.9 & $\mathbf{70.6}$ & 73.6 & $\mathbf{73.2}$ & 73.3 & 76.2 & 76.0 & $\mathbf{76.7}$ & 70.7 & 71.3 & $\mathbf{72.3}$ & 77.7 & 77.2 & 78.2 \\
    DmSGD & 69.9 & 69.3 & 68.7 & 73.1 & 72.7 & 72.4 & 76.3 & 76.0 & 76.2 & 70.5 & 71.0 & 72.1 & 77.5 & 76.9 & 77.5 \\
    DA-DmSGD & 70.2 & 70.0 & 70.2 & 73.3 & 73.1 & 73.1 & $\mathbf{76.4}$ & $\mathbf{76.2}$ & 76.6 & 70.4 & $\mathbf{71.4}$ & 72.1 & 77.6 & $\mathbf{77.4}$ & 77.8 \\
    AWC-DmSGD & 69.5 & 69.4 & 70.4 & 73.2 & 72.7 & 73.2 & 76.3 & 76.0 & 76.3 & 70.6 & 70.8 & 71.9 & 77.4 & 76.9 & 77.7 \\
    SlowMo & 70.5 & 70.0 & 69.9 & 73.5 & 73.0 & 73.0 & 76.3 & 75.5 & 75.5 & 68.7 & 69.1 & 67.8 & 77.4 & 77.2 & 77.1 \\
    QG-DmSGD & 70.2 & $\mathbf{70.5}$ & ${70.3}$ & $\mathbf{73.8}$ & 73.1 & 73.3 & 76.2 & 76.0 & 76.6 & $\mathbf{70.8}$ & 70.6 & 72.0 & 76.8 & 76.5 & 76.7 \\
    D$^2$-DmSGD & 69.1 & 68.9 & 70.1 & 72.7 & 72.3 & 73.1 & 75.4 & 75.3 & 76.2 & 70.2 & 70.6 & 71.8 & 76.6 & 76.4 & 76.5 \\
    DecentLaM & $\mathbf{70.3}$ & 69.9 & 70.5 & 73.4 & 73.1 & $\mathbf{73.4}$ & $\mathbf{76.4}$ & $\mathbf{76.2}$ & $\mathbf{76.7}$ & 70.3 & $\mathbf{71.4}$ & 72.2 & $\mathbf{77.8}$ & 77.2 & $\mathbf{78.3}$ \\

\bottomrule
\end{tabular}}
\end{sc}
\end{small}
\end{center}
\vskip -0.1in
\caption{Top-1 validation accuracy comparison with different models and batch sizes on ImageNet dataset.}
\label{table-imagenet-backbone-acc}
\end{table*}

\vspace{1mm}
\noindent \textbf{Performance with different batch-sizes}.
We first compare the top-1 accuracy of the aforementioned methods when training ResNet-50 \cite{he2016deep} model ($\sim$25.5M parameters)  with different batch-sizes on ImageNet. The network topology is set as the symmetric exponential topology (see  \ref{app-topo}). The results are listed in Table  \ref{table-imagenet-resnet50-acc}. It is observed that:
\begin{itemize}
    \item DecentLaM always has better accuracy than other decentralized algorithms, and its superiority gets evident as batch size grows. DecentLaM outperforms  other baselines by a large margin in the 32K batch size scenario due to its improved inconsistency bias derived in Proposition \ref{prop-decentlam-bias}, Theorems~\ref{thm-decentLaM--sc} and \ref{thm-decentLaM--nc}. 

    \item DecentLaM even outperforms PmSGD and PmSGD + LARS for each batch-size. One conjecture is that its model inconsistency between nodes caused by the partial averaging helps the algorithm escape from shallow local minimums. As a result, DecentLaM can be better as well as faster than DmSGD (with LARS).

    \item DmSGD, DA/AWC-DmSGD, and SlowMo have a severe degraded performance in the 32K batch size scenario. It is because the momentum have amplified their inconsistency bias, see Proposition \ref{prop-dmsdg-bias} and the results listed in Table \ref{table:algoriothm-comparison}. D$^2$-DmSGD's performance also drops significantly while it can  theoretical remove all inconsistency bias. QG-DmSGD has relatively small performance degradation, but it still performs worse than DecentLaM. 
     
\end{itemize}

We also depict the training loss and top-1 validation accuracy curves on ResNet-50 with 2K and 16K batch sizes in Fig.~\ref{Fig:Imagenet-16k}. When batch size is 2K, it is observed that DecentLaM achieves roughly the same training loss as DmSGD. However, as batch size grows to 16K, DecentLaM achieves visibly smaller training loss than DmSGD. This observation illustrates that DecentLaM improves  inconsistency bias which can 
boost its convergence when batch-size is large.




\begin{table}[t!]
\vskip 0.15in
\begin{center}
\begin{small}
\begin{sc}
\begin{tabular}{ccc}
\toprule
    \multirow{2}{*}{topology} & \multicolumn{2}{c}{Batch Size}  \\
     & 16k & 32k \\
    \midrule
    Ring & 76.65 & 76.34 \\
    Mesh & 76.54 & 76.47 \\
    Symmetric Exponential & 76.73 & 76.22 \\
    bipartite random match & 76.53 & 76.11 \\
\bottomrule
\end{tabular}
\end{sc}
\end{small}
\end{center}
\vskip -0.05in
\caption{DecentLaM has consistent performance with different network topologies on ImageNet (ResNet-50).}
\label{table-imagenet-resnet50-topo}
\vskip -0.2in
\end{table}

\vspace{1mm}
\noindent \textbf{Performance with different models.} We now validate whether DecentLaM is effective to different neural network architectures. Table  \ref{table-imagenet-backbone-acc} compares the top-1 accuracy with  backbones widely-used in image classification tasks including ResNet \cite{he2016deep}, MobileNetv2 \cite{sandler2018mobilenetv2} and EfficientNet \cite{tan2019efficientnet}. The accuracy of PmSGD is slightly different from those reported in \cite{sandler2018mobilenetv2, tan2019efficientnet} because we do not utilize tricks such as AutoAugment \cite{cubuk2018autoaugment} that are orthogonal to our methods. It is observed in Table \ref{table-imagenet-backbone-acc} that either PmSGD + LARS or DecentLaM achieves best performance with large batch-size (16K) across different models. However, DecentLaM has runtime speedup due to its efficient neighborhood-communication. For small-batch settings, DA-DmSGD and QG-DmSGD can also achieve the best accuracy for some models.

\vspace{1mm}
\noindent \textbf{Performance with different topologies.} We now examine how DecentLaM is robust to different topologies. To this end, we first organize all computing nodes into ring, mesh, symmetric exponential, or the bipartite random match topology, see Appendix \ref{app-topo} for details of these topologies. Next we test the performance of DecentLaM on ResNet-50 with these topologies and the results are in Table \ref{table-imagenet-resnet50-topo}. It is observed that DecentLaM has a consistent performance with different topologies. The ring topology is sparser than symmetric exponential topology. Interestly, it is observed to have a better accuracy in the 32K batch-size setting. It is conjectured that the ring topology can help escape from shallow local minimums when batch-size is large. We leave the justification as the future work. 

\vspace{1mm}
\noindent \textbf{Training time comparison}. The end-to-end training speed varies across different models and network bandwidth conditions. For the sake of brevity, we compare the runtime of PmSGD, DmSGD and DecentLaM in training ResNet-50 (ImageNet) with different batch sizes and network bandwidths. The speed-up of DmSGD/DecentLaM over PmSGD is consistent for other backbones and tasks. 
In Fig.~\ref{Fig:Imagenet-speed}, DecentLaM and DmSGD have equivalent runtime because they are based on the same partial averaging operation. However, they can achieve 1.2$\sim$1.9$\times$speed-up compared to PmSGD, which is consistent with \cite{assran2019stochastic}.

\begin{figure}[t!]
\vskip 0.2in
\centering
\includegraphics[width=0.235\textwidth]{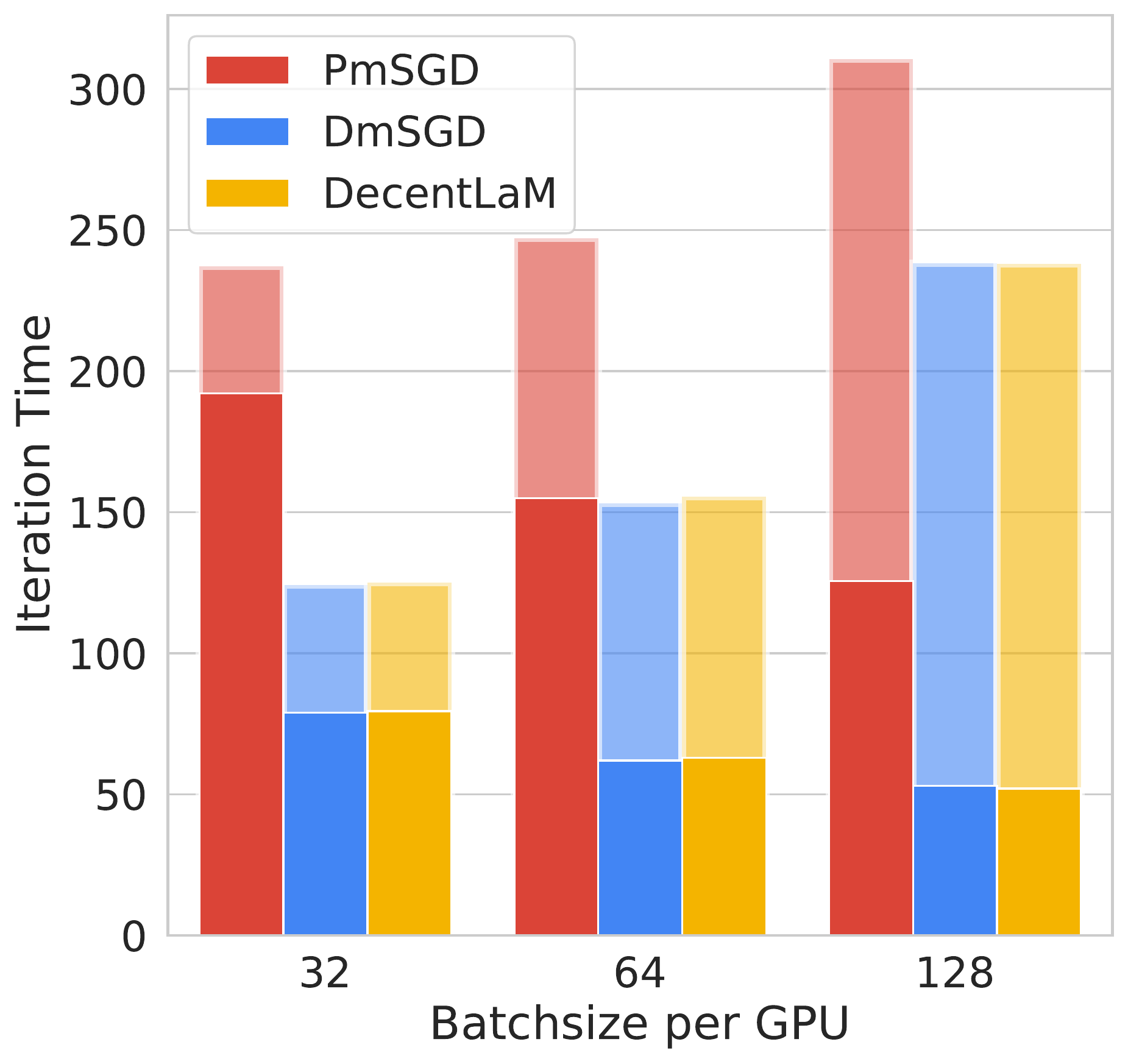} 
\includegraphics[width=0.235\textwidth]{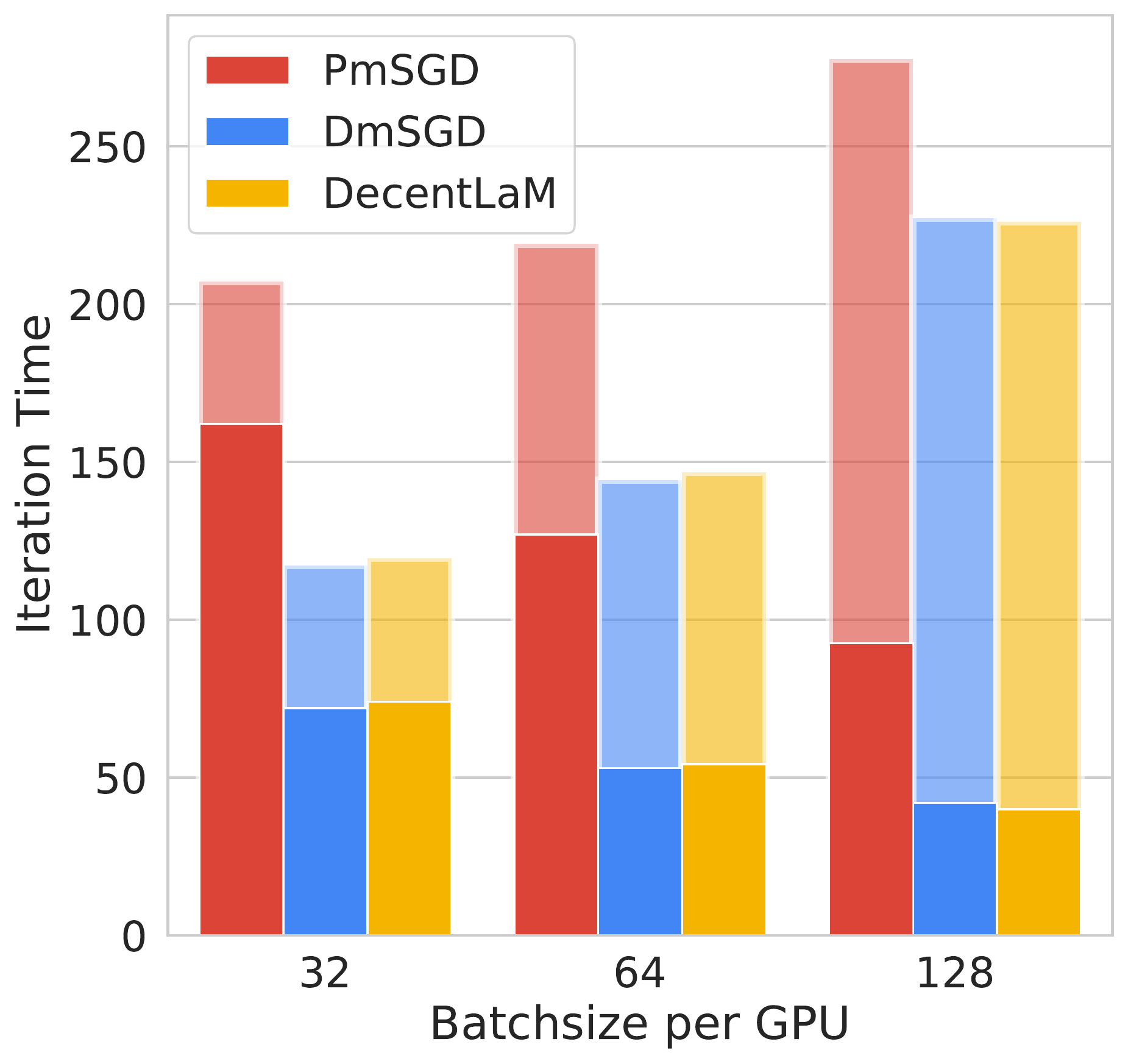} 

\caption{\small Runtime comparison on ResNet-50 with different batch sizes and network bandwidth (Left: 10Gbps; Right: 25Gbps). Each column indicates the averaged iteration runtime of 500 iterations. The thick part highlights the communication overhead.}
\label{Fig:Imagenet-speed}
\vskip -0.2in
\end{figure}

\subsection{Object Detection} 

We compared the aforementioned methods with well-known detection models, e.g. Faster-RCNN \cite{ren2016faster} and RetinaNet \cite{lin2017focal} on popular PASCAL VOC \cite{everingham2010pascal} and COCO \cite{lin2014microsoft} datasets. We adopt the MMDetection  \cite{chen2019mmdetection} framework as the building blocks and utilize ResNet-50 with FPN \cite{lin2017feature} as the  backbone network. We choose mean Average Precision (mAP) as the evaluation metric for both datesets. We used 8 GPUs (which are connected by the symmetric exponential topology) and set the total batch size as 256 in all detection experiments. Table \ref{table-detection-map} shows the performance of each algorithm. Compared with classification tasks, the total batch size in object detection cannot be set too large (typical object detection task sets batch size as $2$ within each GPU). For this reason, our proposed method just reaches a slightly higher mAP than other methods.

\begin{table}[t!]
\vskip 0.15in
\begin{center}
\begin{small}
\begin{sc}
\begin{tabular}{ccccc}
\toprule
    Dataset & \multicolumn{2}{c}{PASCAL VOC} & \multicolumn{2}{c}{COCO} \\
    Model & R-Net & F-RCNN & R-Net & F-RCNN \\

\midrule
PmSGD & 79.0 & 80.3 & 36.2 & 36.5   \\ 
PmSGD+LARS & 78.5 & 79.8 & 35.7 & 36.2 \\
DmSGD & 79.1 & 80.5 & 36.1 & 36.4  \\
DA-DmSGD & 79.0 & 80.5 & 36.4 & 37.0 \\
DecentLaM  & $\mathbf{79.3}$ & $\mathbf{80.7}$ & $\mathbf{36.6}$ & $\mathbf{37.1}$ \\

\bottomrule
\end{tabular}
\end{sc}
\end{small}
\end{center}
\vskip -0.1in
\caption{Comparision of aforementioned methods with different models on PASCAL VOC and COCO datasets. R-Net and F-RCNN refer to RetinaNet and Faster-RCNN respectively.}
\label{table-detection-map}
\vskip -0.2in
\end{table}

\section{Conclusion} 
\label{sec:conclusion}
We investigated the performance degradation in large-batch DmSGD and proposed DecentLaM to remove the momentum-incurred bias. Theoretically, we demonstrate the convergence improvement in smooth convex and non-convex scenarios. Empirically, experimental results of various models and tasks validate our theoretical findings.

{\small
\bibliographystyle{ieee_fullname}
\bibliography{references}

\begin{thebibliography}{10}\itemsep=-1pt

\bibitem{bluefog}
{BlueFog: A High-performance Decentralized Training Framework for Deep
  Learning}.
\newblock \url{https://github.com/Bluefog-Lib/bluefog}.

\bibitem{alistarh2017qsgd}
Dan Alistarh, Demjan Grubic, Jerry Li, Ryota Tomioka, and Milan Vojnovic.
\newblock Qsgd: Communication-efficient sgd via gradient quantization and
  encoding.
\newblock In {\em Advances in Neural Information Processing Systems}, pages
  1709--1720, 2017.

\bibitem{assran2019stochastic}
Mahmoud Assran, Nicolas Loizou, Nicolas Ballas, and Mike Rabbat.
\newblock Stochastic gradient push for distributed deep learning.
\newblock In {\em International Conference on Machine Learning (ICML)}, pages
  344--353, 2019.

\bibitem{balu2020decentralized}
Aditya Balu, Zhanhong Jiang, Sin~Yong Tan, Chinmay Hedge, Young~M Lee, and
  Soumik Sarkar.
\newblock Decentralized deep learning using momentum-accelerated consensus.
\newblock {\em arXiv preprint arXiv:2010.11166}, 2020.

\bibitem{bernstein2018signsgd}
Jeremy Bernstein, Jiawei Zhao, Kamyar Azizzadenesheli, and Anima Anandkumar.
\newblock signsgd with majority vote is communication efficient and fault
  tolerant.
\newblock {\em arXiv preprint arXiv:1810.05291}, 2018.

\bibitem{chen2012diffusion}
Jianshu Chen and Ali~H Sayed.
\newblock Diffusion adaptation strategies for distributed optimization and
  learning over networks.
\newblock {\em IEEE Transactions on Signal Processing}, 60(8):4289--4305, 2012.

\bibitem{chen2019mmdetection}
Kai Chen, Jiaqi Wang, Jiangmiao Pang, Yuhang Cao, Yu Xiong, Xiaoxiao Li,
  Shuyang Sun, Wansen Feng, Ziwei Liu, Jiarui Xu, et~al.
\newblock Mmdetection: Open mmlab detection toolbox and benchmark.
\newblock {\em arXiv preprint arXiv:1906.07155}, 2019.

\bibitem{chen2018lag}
Tianyi Chen, Georgios Giannakis, Tao Sun, and Wotao Yin.
\newblock {LAG}: Lazily aggregated gradient for communication-efficient
  distributed learning.
\newblock In {\em Advances in Neural Information Processing Systems}, pages
  5050--5060, 2018.

\bibitem{cubuk2018autoaugment}
Ekin~D Cubuk, Barret Zoph, Dandelion Mane, Vijay Vasudevan, and Quoc~V Le.
\newblock Autoaugment: Learning augmentation policies from data.
\newblock {\em arXiv preprint arXiv:1805.09501}, 2018.

\bibitem{deng2009imagenet}
Jia Deng, Wei Dong, Richard Socher, Li-Jia Li, Kai Li, and Li Fei-Fei.
\newblock Imagenet: A large-scale hierarchical image database.
\newblock In {\em IEEE Conference on Computer Vision and Pattern Recognition
  (CVPR)}, pages 248--255. Ieee, 2009.

\bibitem{duchi2011dual}
John~C Duchi, Alekh Agarwal, and Martin~J Wainwright.
\newblock Dual averaging for distributed optimization: Convergence analysis and
  network scaling.
\newblock {\em IEEE Transactions on Automatic control}, 57(3):592--606, 2011.

\bibitem{everingham2010pascal}
Mark Everingham, Luc Van~Gool, Christopher~KI Williams, John Winn, and Andrew
  Zisserman.
\newblock The pascal visual object classes (voc) challenge.
\newblock {\em International journal of computer vision}, 88(2):303--338, 2010.

\bibitem{gao2020periodic}
Hongchang Gao and Heng Huang.
\newblock Periodic stochastic gradient descent with momentum for decentralized
  training.
\newblock {\em arXiv preprint arXiv:2008.10435}, 2020.

\bibitem{gitman2019understanding}
Igor Gitman, Hunter Lang, Pengchuan Zhang, and Lin Xiao.
\newblock Understanding the role of momentum in stochastic gradient methods.
\newblock {\em arXiv preprint arXiv:1910.13962}, 2019.

\bibitem{goyal2017accurate}
Priya Goyal, Piotr Doll{\'a}r, Ross Girshick, Pieter Noordhuis, Lukasz
  Wesolowski, Aapo Kyrola, Andrew Tulloch, Yangqing Jia, and Kaiming He.
\newblock Accurate, large minibatch sgd: Training imagenet in 1 hour.
\newblock {\em arXiv preprint arXiv:1706.02677}, 2017.

\bibitem{he2016deep}
Kaiming He, Xiangyu Zhang, Shaoqing Ren, and Jian Sun.
\newblock Deep residual learning for image recognition.
\newblock In {\em IEEE Conference on Computer Vision and Pattern Recognition
  (CVPR)}, pages 770--778, 2016.

\bibitem{kingma2014adam}
Diederik~P Kingma and Jimmy Ba.
\newblock Adam: A method for stochastic optimization.
\newblock {\em arXiv preprint arXiv:1412.6980}, 2014.

\bibitem{koloskova2019decentralized2}
Anastasia Koloskova, Tao Lin, Sebastian~U Stich, and Martin Jaggi.
\newblock Decentralized deep learning with arbitrary communication compression.
\newblock In {\em International Conference on Learning Representations}, 2019.

\bibitem{koloskova2020unified}
Anastasia Koloskova, Nicolas Loizou, Sadra Boreiri, Martin Jaggi, and
  Sebastian~U Stich.
\newblock A unified theory of decentralized sgd with changing topology and
  local updates.
\newblock In {\em International Conference on Machine Learning (ICML)}, pages
  1--12, 2020.

\bibitem{koloskova2019decentralized}
Anastasia Koloskova, Sebastian Stich, and Martin Jaggi.
\newblock Decentralized stochastic optimization and gossip algorithms with
  compressed communication.
\newblock In {\em International Conference on Machine Learning}, pages
  3478--3487, 2019.

\bibitem{kong2021consensus}
Lingjing Kong, Tao Lin, Anastasia Koloskova, Martin Jaggi, and Sebastian~U
  Stich.
\newblock Consensus control for decentralized deep learning.
\newblock {\em arXiv preprint arXiv:2102.04828}, 2021.

\bibitem{li2014scaling}
Mu Li, David~G Andersen, Jun~Woo Park, Alexander~J Smola, Amr Ahmed, Vanja
  Josifovski, James Long, Eugene~J Shekita, and Bor-Yiing Su.
\newblock Scaling distributed machine learning with the parameter server.
\newblock In {\em 11th $\{$USENIX$\}$ Symposium on Operating Systems Design and
  Implementation ($\{$OSDI$\}$ 14)}, pages 583--598, 2014.

\bibitem{li2017decentralized}
Z. Li, W. Shi, and M. Yan.
\newblock A decentralized proximal-gradient method with network independent
  step-sizes and separated convergence rates.
\newblock {\em IEEE Transactions on Signal Processing}, July 2019.
\newblock early acces. Also available on arXiv:1704.07807.

\bibitem{lian2017can}
Xiangru Lian, Ce Zhang, Huan Zhang, Cho-Jui Hsieh, Wei Zhang, and Ji Liu.
\newblock Can decentralized algorithms outperform centralized algorithms? a
  case study for decentralized parallel stochastic gradient descent.
\newblock In {\em Advances in Neural Information Processing Systems}, pages
  5330--5340, 2017.

\bibitem{lian2018asynchronous}
Xiangru Lian, Wei Zhang, Ce Zhang, and Ji Liu.
\newblock Asynchronous decentralized parallel stochastic gradient descent.
\newblock In {\em International Conference on Machine Learning}, pages
  3043--3052, 2018.

\bibitem{lin2021quasi}
Tao Lin, Sai~Praneeth Karimireddy, Sebastian~U Stich, and Martin Jaggi.
\newblock Quasi-global momentum: Accelerating decentralized deep learning on
  heterogeneous data.
\newblock {\em arXiv preprint arXiv:2102.04761}, 2021.

\bibitem{lin2017feature}
Tsung-Yi Lin, Piotr Doll{\'a}r, Ross Girshick, Kaiming He, Bharath Hariharan,
  and Serge Belongie.
\newblock Feature pyramid networks for object detection.
\newblock In {\em Proceedings of the IEEE conference on computer vision and
  pattern recognition}, pages 2117--2125, 2017.

\bibitem{lin2017focal}
Tsung-Yi Lin, Priya Goyal, Ross Girshick, Kaiming He, and Piotr Doll{\'a}r.
\newblock Focal loss for dense object detection.
\newblock In {\em Proceedings of the IEEE international conference on computer
  vision}, pages 2980--2988, 2017.

\bibitem{lin2014microsoft}
Tsung-Yi Lin, Michael Maire, Serge Belongie, James Hays, Pietro Perona, Deva
  Ramanan, Piotr Doll{\'a}r, and C~Lawrence Zitnick.
\newblock Microsoft coco: Common objects in context.
\newblock In {\em European conference on computer vision}, pages 740--755.
  Springer, 2014.

\bibitem{liu2020improved}
Yanli Liu, Yuan Gao, and Wotao Yin.
\newblock An improved analysis of stochastic gradient descent with momentum.
\newblock {\em arXiv preprint arXiv:2007.07989}, 2020.

\bibitem{liu2019communication}
Yaohua Liu, Wei Xu, Gang Wu, Zhi Tian, and Qing Ling.
\newblock Communication-censored admm for decentralized consensus optimization.
\newblock {\em IEEE Transactions on Signal Processing}, 67(10):2565--2579,
  2019.

\bibitem{loizou2020momentum}
Nicolas Loizou and Peter Richt{\'a}rik.
\newblock Momentum and stochastic momentum for stochastic gradient, newton,
  proximal point and subspace descent methods.
\newblock {\em Computational Optimization and Applications}, 77(3):653--710,
  2020.

\bibitem{lu2019gnsd}
Songtao Lu, Xinwei Zhang, Haoran Sun, and Mingyi Hong.
\newblock Gnsd: A gradient-tracking based nonconvex stochastic algorithm for
  decentralized optimization.
\newblock In {\em 2019 IEEE Data Science Workshop (DSW)}, pages 315--321. IEEE,
  2019.

\bibitem{luo2020prague}
Qinyi Luo, Jiaao He, Youwei Zhuo, and Xuehai Qian.
\newblock Prague: High-performance heterogeneity-aware asynchronous
  decentralized training.
\newblock In {\em Proceedings of the Twenty-Fifth International Conference on
  Architectural Support for Programming Languages and Operating Systems}, pages
  401--416, 2020.

\bibitem{nedic2009distributed}
Angelia Nedic and Asuman Ozdaglar.
\newblock Distributed subgradient methods for multi-agent optimization.
\newblock {\em IEEE Transactions on Automatic Control}, 54(1):48--61, 2009.

\bibitem{paszke2019pytorch}
Adam Paszke, Sam Gross, Francisco Massa, Adam Lerer, James Bradbury, Gregory
  Chanan, Trevor Killeen, Zeming Lin, Natalia Gimelshein, Luca Antiga, et~al.
\newblock Pytorch: An imperative style, high-performance deep learning library.
\newblock In {\em Advances in Neural Information Processing Systems (NeurIPS)},
  pages 8024--8035, 2019.

\bibitem{patarasuk2009bandwidth}
Pitch Patarasuk and Xin Yuan.
\newblock Bandwidth optimal all-reduce algorithms for clusters of workstations.
\newblock {\em Journal of Parallel and Distributed Computing}, 69(2):117--124,
  2009.

\bibitem{reddi2019convergence}
Sashank~J Reddi, Satyen Kale, and Sanjiv Kumar.
\newblock On the convergence of adam and beyond.
\newblock {\em arXiv preprint arXiv:1904.09237}, 2019.

\bibitem{ren2016faster}
Shaoqing Ren, Kaiming He, Ross Girshick, and Jian Sun.
\newblock Faster r-cnn: towards real-time object detection with region proposal
  networks.
\newblock {\em IEEE transactions on pattern analysis and machine intelligence},
  39(6):1137--1149, 2016.

\bibitem{sandler2018mobilenetv2}
Mark Sandler, Andrew Howard, Menglong Zhu, Andrey Zhmoginov, and Liang-Chieh
  Chen.
\newblock Mobilenetv2: Inverted residuals and linear bottlenecks.
\newblock In {\em Proceedings of the IEEE conference on computer vision and
  pattern recognition}, pages 4510--4520, 2018.

\bibitem{sayed2014adaptation}
Ali~H Sayed.
\newblock Adaptation, learning, and optimization over networks.
\newblock {\em Foundations and Trends in Machine Learning},
  7(ARTICLE):311--801, 2014.

\bibitem{sebbouh2020convergence}
Othmane Sebbouh, Robert~M Gower, and Aaron Defazio.
\newblock On the convergence of the stochastic heavy ball method.
\newblock {\em arXiv preprint arXiv:2006.07867}, 2020.

\bibitem{singh2020squarm}
Navjot Singh, Deepesh Data, Jemin George, and Suhas Diggavi.
\newblock Squarm-sgd: Communication-efficient momentum sgd for decentralized
  optimization.
\newblock {\em arXiv preprint arXiv:2005.07041}, 2020.

\bibitem{stich2019local}
Sebastian~Urban Stich.
\newblock Local sgd converges fast and communicates little.
\newblock In {\em International Conference on Learning Representations (ICLR)},
  2019.

\bibitem{tan2019efficientnet}
Mingxing Tan and Quoc Le.
\newblock Efficientnet: Rethinking model scaling for convolutional neural
  networks.
\newblock In {\em International Conference on Machine Learning}, pages
  6105--6114. PMLR, 2019.

\bibitem{tang2018d}
Hanlin Tang, Xiangru Lian, Ming Yan, Ce Zhang, and Ji Liu.
\newblock $ d^2$: Decentralized training over decentralized data.
\newblock In {\em International Conference on Machine Learning}, pages
  4848--4856, 2018.

\bibitem{tang2019doublesqueeze}
Hanlin Tang, Chen Yu, Xiangru Lian, Tong Zhang, and Ji Liu.
\newblock Doublesqueeze: Parallel stochastic gradient descent with double-pass
  error-compensated compression.
\newblock In {\em International Conference on Machine Learning}, pages
  6155--6165. PMLR, 2019.

\bibitem{tsitsiklis1986distributed}
John Tsitsiklis, Dimitri Bertsekas, and Michael Athans.
\newblock Distributed asynchronous deterministic and stochastic gradient
  optimization algorithms.
\newblock {\em IEEE transactions on automatic control}, 31(9):803--812, 1986.

\bibitem{wang2019slowmo}
Jianyu Wang, Vinayak Tantia, Nicolas Ballas, and Michael Rabbat.
\newblock Slowmo: Improving communication-efficient distributed sgd with slow
  momentum.
\newblock In {\em International Conference on Learning Representations}, 2019.

\bibitem{xin2020improved}
Ran Xin, Usman~A Khan, and Soummya Kar.
\newblock An improved convergence analysis for decentralized online stochastic
  non-convex optimization.
\newblock {\em arXiv preprint arXiv:2008.04195}, 2020.

\bibitem{you2017large}
Yang You, Igor Gitman, and Boris Ginsburg.
\newblock Large batch training of convolutional networks.
\newblock {\em arXiv preprint arXiv:1708.03888}, 2017.

\bibitem{you2019large}
Yang You, Jonathan Hseu, Chris Ying, James Demmel, Kurt Keutzer, and Cho-Jui
  Hsieh.
\newblock Large-batch training for lstm and beyond.
\newblock In {\em Proceedings of the International Conference for High
  Performance Computing, Networking, Storage and Analysis}, pages 1--16, 2019.

\bibitem{you2019largeb}
Yang You, Jing Li, Sashank Reddi, Jonathan Hseu, Sanjiv Kumar, Srinadh
  Bhojanapalli, Xiaodan Song, James Demmel, Kurt Keutzer, and Cho-Jui Hsieh.
\newblock Large batch optimization for deep learning: Training bert in 76
  minutes.
\newblock {\em arXiv preprint arXiv:1904.00962}, 2019.

\bibitem{you2018imagenet}
Yang You, Zhao Zhang, Cho-Jui Hsieh, James Demmel, and Kurt Keutzer.
\newblock Imagenet training in minutes.
\newblock In {\em Proceedings of the 47th International Conference on Parallel
  Processing}, pages 1--10, 2018.

\bibitem{yu2019linear}
Hao Yu, Rong Jin, and Sen Yang.
\newblock On the linear speedup analysis of communication efficient momentum
  sgd for distributed non-convex optimization.
\newblock In {\em International Conference on Machine Learning}, pages
  7184--7193. PMLR, 2019.

\bibitem{yuan2020influence}
Kun Yuan, Sulaiman~A Alghunaim, Bicheng Ying, and Ali~H Sayed.
\newblock On the influence of bias-correction on distributed stochastic
  optimization.
\newblock {\em IEEE Transactions on Signal Processing}, 2020.

\bibitem{yuan2016convergence}
Kun Yuan, Qing Ling, and Wotao Yin.
\newblock On the convergence of decentralized gradient descent.
\newblock {\em SIAM Journal on Optimization}, 26(3):1835--1854, 2016.

\bibitem{yuan2016influence}
Kun Yuan, Bicheng Ying, and Ali~H Sayed.
\newblock On the influence of momentum acceleration on online learning.
\newblock {\em The Journal of Machine Learning Research}, 17(1):6602--6667,
  2016.

\bibitem{yuan2018exact2}
Kun Yuan, Bicheng Ying, Xiaochuan Zhao, and Ali~H Sayed.
\newblock Exact diffusion for distributed optimization and learning—part ii:
  Convergence analysis.
\newblock {\em IEEE Transactions on Signal Processing}, 67(3):724--739, 2018.

\end{thebibliography}
}

\newpage
\onecolumn
\appendix
\allowdisplaybreaks
\section{Preliminary}
\label{app-pre}
\noindent \textbf{Notation.} We first introduce necessary notations as follows.
\begin{itemize}
	\item $\vx^{(k)} = [(x_1^{(k)})^T; (x_2^{(k)})^T; \cdots; (x_n^{(k)})^T]\in \mathbb{R}^{n\times d}$
	\item $\nabla F(\vx^{(k)};\bxi^{(k)}) = [\nabla F_1(x_1^{(k)};\xi_1^{(k)})^T; \cdots; \nabla F_n(x_n^{(k)};\xi_n^{(k)})^T]\in \mathbb{R}^{n\times d}$
	\item $\nabla f(\vx^{(k)}) = [\nabla f_1(x_1^{(k)})^T; \nabla f_2(x_2^{(k)})^T; \cdots; \nabla f_n(x_n^{(k)})^T]\in \mathbb{R}^{n\times d}$ 
	\item $f(\vx^{(k)}) = \sum_{i=1}^n f(x_i^{(k)})$ and $f(x^{(k)}) = \sum_{i=1}^n f(x^{(k)})$
	\item $\bar{\vx}^{(k)} = [(\bar{x}^{(k)})^T; (\bar{x}^{(k)})^T; \cdots; (\bar{x}^{(k)})^T]\in \mathbb{R}^{n\times d}$ where $\bar{x}^{(k)} = \frac{1}{n}\sum_{i=1}^n \x_i^{(k)}$
	\item ${\vx}^{\star} = [({x}^{\star})^T; ({x}^{\star})^T; \cdots; ({x}^{\star})^T]\in \mathbb{R}^{n\times d}$ where ${x}^{\star}$ is the global solution to problem \eqref{dist-opt}.
	\item $W=[w_{ij}]\in \mathbb{R}^{n\times n}$ is the weight matrix.
	\item $\mathds{1}_n = \mathrm{col}\{1,1,\cdots, 1\} \in \RR^n$.
	\item Given two matrices $\vx, \vy \in \RR^{n\times d}$, we define inner product $\langle \vx, \vy \rangle = \mathrm{tr}(\vx^T \vy)$, the Frobenius norm $\|\vx\|^2 = \langle \vx, \vx \rangle$, and the $\|\vx\|_2$ as $\vx$'s $\ell_2$ norm. Furthermore, for a positive semi-definite matrix $A\in \RR^{n\times n}$, we define $\langle \vx, \vy \rangle_A = \mathrm{tr}(\vx^T A\vy)$ and $\|\vx\|^2_A=\langle \vx,\vx\rangle_A$ for simplicity.
	\item Given $W\in \RR^{n\times n}$, we let $\|W\|_2 = \sigma_{\max}(W)$ where   $\sigma_{\max}(\cdot)$ denote the maximum sigular value.
\end{itemize}

\vspace{2mm}
\noindent \textbf{GmSGD in matrix notation.} For ease of analysis, we rewrite the recursion of GmSGD in Algorithm \ref{Algorithm: DmSGD} with matrix notation: 
\begin{align}
\vm^{(k+1)} &= \beta \m^{(k)} + \nabla F(\vx^{(k)};\bxi^{(k)})  \label{dmssgd-matrix-1} \\
\vx^{(k+1)} &= W(\vx^{(k)} - \gamma \vm^{(k+1)}) \label{dmssgd-matrix-2}
\end{align}

\vspace{2mm}
\noindent \textbf{DecentLaM in matrix notation.} We can also rewrite DecentLaM in Algorithm \ref{Algorithm: DecentLaM} with  matrix notation: 
\begin{align}
\tilde{\vg}^{(k)} &=\frac{1}{\gamma}\vx^{(k)}-\frac{1}{\gamma}W(\vx^{(k)}-\gamma \nabla F(\vx^{(k)},\bxi^{(k)})) \label{eqn:decentlam-00} \\
\vm^{(k+1)}&=\beta\vm^{(k)}+\tilde{\vg}^{(k)}\label{eqn:decentlam-11}\\
\vx^{(k+1)}&=\vx^{(k)}-\gamma \vm^{(k+1)}\label{eqn:decentlam-22}
\end{align}
Moreover, we define $\vg^{(k)} = \mathbb{E}[\tilde{\vg}^{(k)}] = \frac{1}{\gamma}\vx^{(k)}-\frac{1}{\gamma}W(\vx^{(k)}-\gamma \nabla f(\vx^{(k)}))$

\vspace{2mm}
\noindent \textbf{Smoothness.} Since each $f_i(x)$ is assumed to be $L$-smooth in Assumption A.1, it holds that $f(x) = \frac{1}{n}\sum_{i=1}^n f_i(x)$ is also $L$-smooth. As a result, the following inequality holds for any $x, y \in \mathbb{R}^d$:
\begin{align}
f(x) - f(y) - \frac{L}{2}\|x - y\|^2 &\le  \langle \nabla f(y), x- y \rangle \label{sdu-2}
\end{align}


\noindent \textbf{Network weight matrix.} Suppose a symmetric matrix $W\in \RR^{n\times n}$ satisfies Assumption A.3, and $\lambda_j$ denotes its $j$-th largest eigenvalue. It holds that $1 = \lambda_1 > \lambda_2 \ge \cdots \ge \lambda_n > -1$. As a result, it holds that 
\begin{align}\label{network-inequaliy}
\|W\|_2 = 1, \quad \mbox{and} \quad \rho = \|W - \frac{1}{n}\mathds{1}\mathds{1}^T\|_2 = \max\{|\lambda_2|, |\lambda_n|\} \in (0,1)
\end{align}
If $W$ satisfying Assumption A.3 is further assumed to be positive-definite, it holds that $1 = \lambda_1 > \lambda_2 \ge \cdots \ge \lambda_n > 0$. 

\vspace{1mm}
\noindent \textbf{Submultiplicativity of the Frobenius norm.} Given matrices $W\in \RR^{n\times n}$ and $\vy\in \RR^{n\times d}$, it holds that 
\begin{align}\label{submulti}
\|W\vy\| \le \|W\|_2 \|\vy\|.
\end{align}
To verify it, by letting $y_j$ be the $j$-th column of $\vy$, we have $\|W\vy\|^2 = \sum_{j=1}^d \|Wy_j\|_2^2 \le \sum_{j=1}^d \|W\|_2^2 \|y_j\|_2^2=\|W\|_2^2\|\vy\|^2$.

\section{Reformulation of DmSGD and DecentLaM}

\subsection{Reformulation of DmSGD}
\label{app-reform-dsgd}
In this section we show how DmSGD algorithm \ref{Algorithm: DmSGD} can be rewritten as \eqref{xsdhs7}. To this end, we rewrite \eqref{dmssgd-matrix-2} as 
\begin{align}\label{xcnxcn8}
\beta \vx^{(k)} = W( \beta \vx^{(k-1)} - \gamma \beta \vm^{(k)}).
\end{align}
Subtracting \eqref{xcnxcn8} from \eqref{dmssgd-matrix-2}, we have
\begin{align}
\vx^{(k+1)} - \beta \vx^{(k)} = W\big( \vx^{(k)} - \beta \vx^{(k-1)} - \gamma (\vm^{(k+1)} - \beta \vm^{(k)}) \big) \overset{\eqref{dmssgd-matrix-1}}{=} W \big( \vx^{(k)} - \beta \vx^{(k-1)} - \gamma \nabla F(\vx^{(k)};\bxi^{(k)})\big)
\end{align}
which is equivalent to 
\begin{align}\label{xcnxcnxc7}
\vx^{(k+1)} = \underbrace{W \big(\vx^{(k)} - \gamma \nabla F(\vx^{(k)};\bxi^{(k)})\big)}_{\rm DSGD} + \underbrace{\beta (\vx^{(k)}  - W \vx^{(k-1)})}_{\rm momentum}.
\end{align}
When a full-batch gradient is used, the above recursion becomes 
\begin{align}\label{xcnxcnxc7-no-noise}
\vx^{(k+1)} = {W \big(\vx^{(k)} - \gamma \nabla f(\vx^{(k)})\big)} + \beta (\vx^{(k)}  - W \vx^{(k-1)})
\end{align}
which is essentially recursion \eqref{xsdhs7} in the matrix notation. 

\subsection{Reformulation of DecentLaM}
\label{app-reform-decentLaM}

In this section we show another formulation of DecentLaM Algorithm \ref{Algorithm: DecentLaM}. To this end, we rewrite \eqref{eqn:decentlam-22} as
\begin{align}\label{xn23we76sd6}
\beta \vx^{(k)} = \beta \vx^{(k-1)} - \gamma \beta \vm^{(k)}.
\end{align}
Subtracting \eqref{xn23we76sd6} from \eqref{eqn:decentlam-22}, we have 
\begin{align}
\vx^{(k+1)} - \beta \vx^{(k)} &= \vx^{(k)} - \beta \vx^{(k-1)} - \gamma (\vm^{(k+1)} - \beta \vm^{(k)}) \nonumber \\ & \overset{\eqref{eqn:decentlam-11}}{=} \vx^{(k)} - \beta \vx^{(k-1)} - \gamma \tilde{\g}^{(k)}  \nonumber \\
& \overset{\eqref{eqn:decentlam-00}}{=} W(\vx^{(k)}-\gamma \nabla F(\vx^{(k)},\bxi^{(k)})) - \beta \vx^{(k-1)}
\end{align}
which is  equivalent to 
\begin{align}\label{2465sbv}
\vx^{(k+1)} = \underbrace{W(\vx^{(k)}-\gamma \nabla F(\vx^{(k)},\bxi^{(k)}))}_{\rm DSGD} + \underbrace{\beta (\vx^{(k)} - \vx^{(k-1)})}_{\rm momentum}
\end{align}
When a full-batch gradient is used, the above recursion becomes 
\begin{align}\label{2465sbv-no-noise}
\vx^{(k+1)} = {W \big(\vx^{(k)} - \gamma \nabla f(\vx^{(k)})\big)} + \beta (\vx^{(k)}  - \vx^{(k-1)}).
\end{align}

\section{Limiting Bias of Decentralized Algorithms}

\subsection{Limiting bias of DSGD}
\label{app-limit-bias-dsgd}
In this section we illustrate the stochastic bias and inconsistency bias in the DSGD algorithm. It is established in \cite{yuan2020influence} that DSGD in the strongly-convex scenario will converge as follows:
\begin{align}\label{dsgd-convergence}
\frac{1}{n}\sum_{i=1}^n \mathbb{E}\|x_i^{(k)} - x^\star\|^2 = O\Big( \underbrace{(1-\gamma \mu)^k}_{\rm convg.\ rate} + \underbrace{\frac{\gamma \sigma^2}{n} + \frac{\gamma^2 \sigma^2}{1-\rho}}_{\rm sto.\ bias} + \underbrace{\frac{\gamma^2 b^2}{(1-\rho)^2}}_{\rm inconsis.\ bias} \Big).
\end{align}
where $\sigma^2$ is the variance of gradient noise, and $b^2$ is the data inconsistency (see the definition in Proposition \ref{prop-dmsdg-bias}). When learning rate is constant and iteration $k$ goes to infinity, DSGD will converge with limiting bias. i.e., 
\begin{align}
\mbox{Limiting bias} = \limsup_{k\to \infty} \frac{1}{n}\sum_{i=1}^n \mathbb{E}\|x_i^{(k)} - x^\star\|^2 = O\Big( \underbrace{\frac{\gamma \sigma^2}{n} + \frac{\gamma^2 \sigma^2}{1-\rho}}_{\rm sto.\ bias} + \underbrace{\frac{\gamma^2 b^2}{(1-\rho)^2}}_{\rm inconsis.\ bias} \Big)
\end{align}
As we discussed in Sec.~\ref{sec:DmSG-bias}, the limiting bias can be divided into two categories: stochastic bias and inconsistency bias. The stochastic bias is caused by the gradient noise. In the large-batch scenario in which the gradient noise $\sigma^2$ gets  significantly reduced, the inconsistency bias will dominate the magnitude of DSGD's limiting bias. 

\subsection{Inconsistency bias of DmSGD (Proof of Proposition \ref{prop-dmsdg-bias})}
\label{app-inconsist-bias-dsgd}

In this section we will prove Proposition \ref{prop-dmsdg-bias}. To achieve the inconsistency bias, we let $\vx_{\rm m}$ be the fixed point of $\vx^{(k)}$, i.e., $\vx^{(k)} \to \vx_{\rm m}$. From recursion \eqref{xcnxcnxc7-no-noise}, it is derived that $\vx_{\rm m}$ satisfies 
\begin{align}\label{dmsgd-bais-fixed-point-recursion}
(1-\beta)(I-W)\vx_{\rm m} = -\gamma W\nabla f(\vx_{\rm m}).
\end{align}

\vspace{1mm}
\noindent \textbf{Bound of $\|{\vx}_{\rm m} - \bar{\vx}_{\rm m}\| $.} Letting $\bar{\vx}_{\rm m} = \frac{1}{n}\mathds{1}\mathds{1}^T \vx_{\rm m}$, it holds that $(I-W)\bar{\vx}_{\rm m} = 0$ because $W\mathds{1} = \mathds{1}$ (see Assumption A.3). Substituting $(I-W)\bar{\vx}_{\rm m} = 0$ into \eqref{dmsgd-bais-fixed-point-recursion}, we have 
\begin{align}\label{dmsgd-bais-fixed-point-recursion-xbar}
(1-\beta)(I-W)(\vx_{\rm m} - \bar{\vx}_{\rm m}) = -\gamma W\nabla f(\vx_{\rm m}).
\end{align}
Since $W$ is symmetric and satisfies $W\mathds{1}=\mathds{1}$ (see Assumption A.3), we can eigen-decompose it as 
\begin{align}\label{xnsd65}
W = \underbrace{[\frac{1}{\sqrt{n}}\mathds{1} \quad U_1]}_{U}
\underbrace{\left[
\begin{array}{cc}
1 & 0 \\
0 & \Lambda_1
\end{array}
\right]}_{\Lambda}
\underbrace{\left[
\begin{array}{c}
\frac{1}{\sqrt{n}}\mathds{1}^T \\
U_1^T
\end{array}
\right]}_{U^T} 
\end{align}
where $U$ is the orthonormal matrix, and $\Lambda_1 = \mathrm{diag}\{\lambda_2,\cdots, \lambda_n\}$ is a diagonal matrix. With \eqref{xnsd65}, we have
\begin{align}\label{xcnwhzi}
\|(I-W)(\vx_{\rm m} - \bar{\vx}_{\rm m})\|^2 &= \|U(I-\Lambda)U^T(\vx_{\rm m} - \bar{\vx}_{\rm m})\|^2 \nonumber \\
& \overset{(a)}{=} \|(I-\Lambda)U^T(\vx_{\rm m} - \bar{\vx}_{\rm m})\|^2 \nonumber \\
& \overset{(b)}{=} \|(I-\Lambda_1)U_1^T(\vx_{\rm m} - \bar{\vx}_{\rm m})\|^2 \nonumber \\
&\ge (1-\lambda_2)^2 \|U_1^T(\vx_{\rm m} - \bar{\vx}_{\rm m})\|^2 \nonumber \\
&\overset{(c)}{=} (1-\lambda_2)^2 \|U^T(\vx_{\rm m} - \bar{\vx}_{\rm m})\|^2 \nonumber \\
&\overset{(d)}{=} (1-\lambda_2)^2 \|\vx_{\rm m} - \bar{\vx}_{\rm m}\|^2
\end{align}
where (a) and (d) hold because $U$ is orthonormal, and (b) and (c) hold because $\|U^T(\vx_{\rm m} - \bar{\vx}_{\rm m})\|^2 \overset{\eqref{xnsd65}}{=} \|\frac{1}{\sqrt{n}}\mathds{1}^T(\vx_{\rm m} - \bar{\vx}_{\rm m})\|^2 + \|U_1^T(\vx_{\rm m} - \bar{\vx}_{\rm m})\|^2 = \|U_1^T(\vx_{\rm m} - \bar{\vx}_{\rm m})\|^2$. With \eqref{dmsgd-bais-fixed-point-recursion-xbar} and \eqref{xcnwhzi}, we have
\begin{align}\label{xcn23e6sdgf}
(1-\beta)(1-\lambda_2) \|\vx_{\rm m} - \bar{\vx}_{\rm m}\| &\le (1-\beta)\|(I-W)(\vx_{\rm m} - \bar{\vx}_{\rm m})\| \nonumber \\
&= \gamma\|W\nabla f(\vx_{\rm m})\| \nonumber \\
&\overset{\eqref{network-inequaliy}}{\le} \gamma \|\nabla f(\vx_{\rm m})\| \nonumber \\
&\le \gamma \|\nabla f(\vx_{\rm m}) - \nabla f(\bar{\vx}_{\rm m})\| + \gamma\|\nabla f(\bar{\vx}_{\rm m}) - \nabla f({\vx}^\star)\| + \gamma\|\nabla f({\vx}^\star)\| \nonumber \\
&\le \gamma L\|\vx_{\rm m} - \bar{\vx}_{\rm m}\| + \sqrt{n} \gamma L \|\bar{x}_m - x^\star\| + \sqrt{n}\gamma b
\end{align}
where $x^\star$ is the global solution to problem \eqref{dist-opt}, $\bar{x}_{\rm m} = \frac{1}{n}\mathds{1}^T \vx_{\rm m}$, and $b^2 = \frac{1}{n}\sum_{i=1}^n\|\nabla f_i(x^\star)\|^2$. 

\vspace{1mm}
\noindent \textbf{Bound of $\|\bar{x}_m - x^\star\|.$} Left-multiplying $\frac{1}{n}\mathds{1}^T$ to both sides of \eqref{dmsgd-bais-fixed-point-recursion}, we achieve $\frac{1}{n}\mathds{1}^T \nabla f(\vx_{\rm m}) = 0$. With this fact we have
\begin{align}
\|\bar{x}_{\rm m} - x^\star\| &= \|\bar{x}_{\rm m} - x^\star - \gamma \big(\frac{1}{n}\mathds{1}^T \nabla F(\vx_{\rm m}) - \frac{1}{n}\mathds{1}^T\nabla F(\vx^\star)\big)\| \nonumber \\
&= \|\bar{x}_{\rm m} - x^\star - \gamma \big(\frac{1}{n}\mathds{1}^T \nabla F(\bar{\vx}_{\rm m}) - \frac{1}{n}\mathds{1}^T\nabla F(\vx^\star)\big)\| + \gamma \|\frac{1}{n}\mathds{1}^T\nabla F(\vx_{\rm m}) - \frac{1}{n}\mathds{1}^T\nabla F(\bar{\vx}_{\rm m})\| \nonumber \\
&\overset{(a)}{\le} (1-\frac{\gamma \mu}{2})\|\bar{x}_{\rm m} - x^\star\| + \frac{\gamma L}{\sqrt{n}} \|\vx_{\rm m} - \bar{\vx}_{\rm m}\|
\end{align}
where (a) holds because $\frac{1}{n}\sum_{i=1}^n f_i(x)$ is $L$-smooth and $\mu$-strongly convex (see Assumptions A.1 and A.5). We thus have
\begin{align}\label{xn237asdgb}
\sqrt{n}\|\bar{x}_{\rm m} - x^\star\| \le \frac{2L}{\mu}\|\vx_{\rm m} - \bar{\vx}_{\rm m}\|.
\end{align}

\vspace{1mm}
\noindent \textbf{Proof of Proposition \ref{prop-dmsdg-bias}.} Substituting \eqref{xn237asdgb} to \eqref{xcn23e6sdgf}, we achieve
\begin{align}
(1-\beta)(1-\lambda_2) \|\vx_{\rm m} - \bar{\vx}_{\rm m}\| &\le  \Big(\gamma L + \frac{2\gamma L^2}{\mu} \Big)\|\vx_{\rm m} - \bar{\vx}_{\rm m}\| + \sqrt{n} \gamma b \nonumber \\
&\le \frac{3\gamma L^2}{\mu} \|\vx_{\rm m} - \bar{\vx}_{\rm m}\| + \sqrt{n} \gamma b
\end{align}
If $\gamma \le \frac{\mu(1-\beta)(1-\lambda)}{6L^2}$, the above inequality becomes 
\begin{align}\label{xnwe8sdx}
\|\vx_{\rm m} - \bar{\vx}_{\rm m}\| \le \frac{2\sqrt{n}\gamma b}{(1-\beta)(1-\lambda_2)}.
\end{align}
With \eqref{xn237asdgb} and \eqref{xnwe8sdx}, we have 
\begin{align}\label{c676}
\|\vx_{\rm m} - \vx^\star\| \le \|\vx_{\rm m} - \bar{\vx}_{\rm m}\| + \sqrt{n}\|\bar{x}_{\rm m} - x^\star\| \le (1 + \frac{2L}{\mu}) \frac{2\sqrt{n}\gamma b}{(1-\beta)(1-\lambda_2)} \le (1 + \frac{2L}{\mu}) \frac{2\sqrt{n}\gamma b}{(1-\beta)(1-\rho)},
\end{align}
where the last inequality holds because $\rho = \max\{ |\lambda_2|, |\lambda_n|\}$. Inequality \eqref{c676} leads to 
\begin{align}\label{asdfasdsfdaaa09}
\lim_{k\to \infty} \frac{1}{n}\sum_{i=1}^n \|x_i^{(k)} - x^\star\|^2 = \frac{1}{n}\|\vx_{\rm m} - \vx^\star\|^2 \overset{\eqref{c676}}{=}  O\Big( \frac{\gamma^2 b^2}{(1-\rho)^2(1-\beta)^2} \Big).
\end{align}
This concludes the proof of Proposition \ref{prop-dmsdg-bias}. 

\subsection{Inconsistency bias of DecentLaM (Proof of Proposition \ref{prop-decentlam-bias})}
\label{app-inconsist-bias-decentLaM}

We let $\vx_{\rm L}$ be the fixed point of the DecentLaM iterate  $\vx^{(k)}$, i.e., $\vx^{(k)} \to \vx_{\rm L}$. From DecentLaM recursion \eqref{2465sbv-no-noise}, we have 
\begin{align}\label{xcnweu2376zz0}
    (I-W)\vx_{\rm L} = - \gamma W \nabla f(\vx_{\rm L}).
\end{align}
By following the arguments in  \eqref{dmsgd-bais-fixed-point-recursion-xbar}--\eqref{asdfasdsfdaaa09}, we can prove Proposition \ref{prop-decentlam-bias}. 

\section{Fundamental Supporting Lemmas}
\label{app-support-lemmas}

In this section, we establish the key lemmas to facilitate the  convergence analysis in Appendices \ref{app-section-non-convex} and \ref{app-thm-decentLaM--sc}. This section assumes the weight matrix $W$ to be positive-definite to simplify the derivations.

It is shown in Sec.~\ref{sec:decentLaM} that DecentLaM can be interpreted as a standard momentum SGD algorithm to solve problem \eqref{s-prob}. This paper will conduct all analysis based on the recursions \eqref{decentLaM-s-1}-\eqref{decentLaM-s-3}. To this end, we define 
\begin{equation} \label{eqn:x-s}
\nabla_\bs \cF(\bs;\bxi)=W^\frac{1}{2}\nabla F(W^\frac{1}{2}\bs;\bxi) + \frac{1}{\gamma}(I-W)\bs, \quad \mbox{and} \quad \mathbb{E}[\nabla_\bs \cF(\bs;\bxi)] = W^\frac{1}{2}\nabla f(W^\frac{1}{2}\bs) + \frac{1}{\gamma}(I-W)\bs.
\end{equation}
Furthermore, the relations between $\bs$ and $\vx$ are $\vx = W^{\frac{1}{2}}\bs$ and $\bs = W^{-\frac{1}{2}} \vx$. For ease of analysis, we transform  \eqref{decentLaM-s-1}-\eqref{decentLaM-s-3} into the following recursions
\begin{align}
&\vm_{\bs}^{(k+1)}=\beta \vm_{\bs}^{(k)}+(1-\beta )\nabla_{s} \cF ( \bs^{(k)};\bxi^{(k)})\label{eqn:decentlam-7}\\
    &\bs^{(k+1)}=\bs^{(k)}-\frac{\gamma}{1-\beta}\vm_{\bs}^{(k+1)}\label{eqn:decentlam-8}
\end{align}
Comparing \eqref{decentLaM-s-1}--\eqref{decentLaM-s-3} and \eqref{eqn:decentlam-7}--\eqref{eqn:decentlam-8}, one can verify the sequence $\{\bs^{(k)}\}$ generated by two set of recursions are exactly the same when $\vm_{\bs}^{(0)} = 0$. Quantity $\vm_s$ in \eqref{eqn:decentlam-7}--\eqref{eqn:decentlam-8} is essentially the scaled version (with coefficient $1-\beta$) of that in \eqref{decentLaM-s-1}--\eqref{decentLaM-s-3}. 
Left-multiplying $\frac{1}{n}\mathds{1}^T$ to both sides in \eqref{eqn:decentlam-7} and \eqref{eqn:decentlam-8}, and defining $\bar{s}=\frac{1}{n}\mathds{1}^T\bs$ and $\bar{m}_s=\frac{1}{n}\mathds{1}^T\vm_s$, we achieve 
\begin{align}
&\bar{m}_{s}^{(k+1)}=\beta \bar{m}_{s}^{(k)}+(1-\beta )\frac{1}{n} \mathds{1}^T \nabla F(W^{\frac{1}{2}}{\bs}^{(k)}; \bxi^{(k)})\label{eqn:decentlam-7-0}\\
&\bar{s}^{(k+1)}=\bar{s}^{(k)}-\frac{\gamma}{1-\beta}\bar{m}_{s}^{(k+1)}\label{eqn:decentlam-8-0}
\end{align}
We introduce $\bar{\bs}^{(k)} = [(\bar{s}^{(k)})^T; \cdots; (\bar{s}^{(k)})^T]\in \mathbb{R}^{n\times d}$ and $\bar{\vm}_s^{(k)} = [(\bar{m}_s^{(k)})^T; \cdots; (\bar{m}_s^{(k)})^T]\in \mathbb{R}^{n\times d}$. Since $\bar{\vx} = \frac{1}{n}\mathds{1}\mathds{1}^T \vx = \frac{1}{n}\mathds{1}\mathds{1}^T W^{\frac{1}{2}} \bs = \frac{1}{n}\mathds{1}\mathds{1}^T \bs = \bar{\bs}$, we conclude that $\bar{x} = \bar{s}$. 

\subsection{Supporting lemmas}
In the following lemma, we bound the average of momentum $\bar{m}_s^{(k)} = \frac{1}{n}\mathds{1}^T\vm_{\bs}^{(k)}$. 
\begin{lemma}[\sc Bound of $\mathbb{E}\|\bar{m}_s^{(k)}\|^2$] \label{lm-bound-ms}
Under Assumption A.2, it holds for any random variable $\vy\in \mathbb{R}^{n\times d}$ that 
\begin{align}
\mathbb{E}\|\bar{m}_s^{(k)}\|^2 &\le 4(1-\beta^k)^2 \mathbb{E}\|\frac{1}{n}\mathds{1}^T\nabla f(\vy)\|^2   + \frac{2(1-\beta)(1-\beta^{2(k+1)})\sigma^2}{n} \nonumber \\
&\quad + 4\mathbb{E}\|(1-\beta)\sum_{i=0}^{k-1}\beta^{k-1-i} [\frac{1}{n}\mathds{1}^T\nabla f(W^{\frac{1}{2}}\bs^{(i)}) - \frac{1}{n}\mathds{1}^T\nabla f(\vy)]\|^2
\end{align}
\end{lemma}
\begin{proof}
Since $\vm^{(0)}=0$, we have $\vm_s^{(0)}=0$ and $\bar{m}_s^{(0)} = 0$. Keep iterating recursion \eqref{eqn:decentlam-7-0}, we have 
$$\bar{m}_s^{(k+1)} = (1-\beta)\sum_{i=0}^k\beta^{k-i} \frac{1}{n}\mathds{1}^T\nabla F(W^{\frac{1}{2}}\bs^{(i)};\bxi^{(i)}) = (1-\beta)\sum_{i=0}^k\beta^{k-i} \frac{1}{n}\mathds{1}^T[\nabla f(W^{\frac{1}{2}}\bs^{(i)}) +  \nabla F(W^{\frac{1}{2}}\bs^{(i)};\bxi^{(i)}) - \nabla f(W^{\frac{1}{2}}\bs^{(i)})]$$
Note that 
\begin{align}\label{xb23zz}
\mathbb{E}\|(1-\beta)\sum_{i=0}^k\beta^{k-i} \frac{1}{n}\mathds{1}^T[\nabla F(W^{\frac{1}{2}}\bs^{(i)};\bxi^{(i)}) - \nabla f(W^{\frac{1}{2}}\bs^{(i)})]\|^2 \overset{(a)}{\le} \frac{(1-\beta)^2\sigma^2}{n}\sum_{i=0}^k \beta^{2(k-i)} \overset{(b)}{\le} \frac{(1-\beta)(1-\beta^{2(k+1)})\sigma^2}{n} 
\end{align}
where (a) holds because of Assumption A.2 and (b) holds since $\sum_{i=0}^k \beta^{2(k-i)} = \frac{1 - \beta^{2(k+1)}}{1-\beta^2} = \frac{1 - \beta^{2(k+1)}}{(1-\beta)(1+\beta)} \le \frac{1-\beta^{2(k+1)}}{1-\beta}$. Therefore, 
\begin{align}\label{p2378-0}
\mathbb{E}\|\bar{m}_s^{(k)}\|^2 &\le 2 \mathbb{E}\|(1\hspace{-0.5mm}-\hspace{-0.5mm}\beta)\hspace{-1mm} \sum_{i=0}^{k-1}\beta^{k-1-i} \frac{1}{n}\mathds{1}^T\nabla f(W^{\frac{1}{2}}\bs^{(i)})\|^2 \hspace{-0.5mm}+\hspace{-0.5mm} 2\mathbb{E}\|(1\hspace{-0.5mm}-\hspace{-0.5mm}\beta)\hspace{-1mm} \sum_{i=0}^{k-1}\beta^{k-1-i} \frac{1}{n}\mathds{1}^T[\nabla F(W^{\frac{1}{2}}\bs^{(i)};\bxi^{(i)}) \hspace{-0.5mm}-\hspace{-0.5mm} \nabla f(W^{\frac{1}{2}}\bs^{(i)})]\|^2 \nonumber \\
&\overset{\eqref{xb23zz}}{\le} 2 \mathbb{E}\|(1-\beta)\sum_{i=0}^{k-1}\beta^{k-1-i} \frac{1}{n}\mathds{1}^T\nabla f(W^{\frac{1}{2}}\bs^{(i)})\|^2 + \frac{2(1-\beta)(1-\beta^{2(k+1)})\sigma^2}{n} 
\end{align}
On the other hand, we can verify for any random variable $\vy$ that 
\begin{align}\label{p23782-0}
&\hspace{-1cm} \mathbb{E}\|\frac{1-\beta}{1-\beta^k}\sum_{i=0}^{k-1}\beta^{k-1-i} \frac{1}{n}\mathds{1}^T\nabla f(W^{\frac{1}{2}}\bs^{(i)})\|^2  \nonumber \\
&\le  2 \mathbb{E}\|\frac{1}{n}\mathds{1}^T\nabla f(\vy)\|^2  + 2 \mathbb{E}\|\frac{1-\beta}{1-\beta^k}\sum_{i=0}^{k-1}\beta^{k-1-i} [\frac{1}{n}\mathds{1}^T\nabla f(W^{\frac{1}{2}}\bs^{(i)}) - \frac{1}{n}\mathds{1}^T\nabla f(\vy)]\|^2.
\end{align}
Combining inequalities \eqref{p2378-0} and \eqref{p23782-0}, we achieve
\begin{align}
\mathbb{E}\|\bar{m}_s^{(k)}\|^2 &\overset{\eqref{p2378-0}}{\le} 2(1-\beta^k)^2 \mathbb{E}\|\frac{1-\beta}{1-\beta^k}\sum_{i=0}^{k-1}\beta^{k-1-i} \frac{1}{n}\mathds{1}^T\nabla f(W^{\frac{1}{2}}\bs^{(i)})\|^2 + \frac{2(1-\beta)(1-\beta^{2(k+1)})\sigma^2}{n} \nonumber \\
&\overset{\eqref{p23782-0}}{\le} 4 (1-\beta^k)^2 \mathbb{E}\|\frac{1}{n}\mathds{1}^T\nabla f(\vy)\|^2   + \frac{2(1-\beta)(1-\beta^{2(k+1)})\sigma^2}{n} \nonumber \\ 
&\quad + 4  \mathbb{E}\|(1-\beta)\sum_{i=0}^{k-1}\beta^{k-1-i} [\frac{1}{n}\mathds{1}^T\nabla f(W^{\frac{1}{2}}\bs^{(i)}) - \frac{1}{n}\mathds{1}^T\nabla f(\vy)]\|^2 
\end{align}
\end{proof}
\noindent Recalling from \eqref{eqn:decentlam-8-0} that $\bar{s}^{(k+1)} - \bar{s}^{(k)} = - \frac{\gamma}{1-\beta} \bar{m}_s^{(k+1)}$, we can  derive from Lemma \ref{lm-bound-ms} (by setting $\vy = W^{\frac{1}{2}}\bs^{(k-1)}$) that
\begin{align}\label{bound-bar-s-diff}
\mathbb{E}\|\bar{s}^{(k+1)} - \bar{s}^{(k)}\|^2 &\le \frac{4\gamma^2}{(1-\beta)^2} \mathbb{E}\|\frac{1}{n}\mathds{1}^T\nabla f(W^{\frac{1}{2}}\bs^{(k)})\|^2   \hspace{-0.8mm}+\hspace{-0.8mm} \frac{2\gamma^2(1-\beta^{2(k+1)})\sigma^2}{(1-\beta)n} \nonumber \\
&\quad \hspace{-0.8mm}+\hspace{-0.8mm} \frac{4\gamma^2}{(1-\beta)^2}  \mathbb{E}\|(1-\beta)\sum_{i=0}^{k}\beta^{k-i} [\frac{1}{n}\mathds{1}^T\nabla f(W^{\frac{1}{2}}\bs^{(i)}) \hspace{-0.8mm}-\hspace{-0.8mm} \frac{1}{n}\mathds{1}^T\nabla f(W^{\frac{1}{2}}\bs^{(k)})]\|^2 \nonumber \\
&\le \frac{4\gamma^2}{(1-\beta)^2} \mathbb{E}\|\frac{1}{n}\mathds{1}^T\nabla f(W^{\frac{1}{2}}\bs^{(k)})\|^2   \hspace{-0.8mm}+\hspace{-0.8mm} \frac{2\gamma^2\sigma^2}{(1-\beta)n} \nonumber \\
&\quad \hspace{-0.8mm}+\hspace{-0.8mm} \frac{4\gamma^2\beta^2 }{(1-\beta)^2}  \mathbb{E}\|(1-\beta)\sum_{i=0}^{k-1}\beta^{k-1-i} [\frac{1}{n}\mathds{1}^T\nabla f(W^{\frac{1}{2}}\bs^{(i)}) \hspace{-0.8mm}-\hspace{-0.8mm} \frac{1}{n}\mathds{1}^T\nabla f(W^{\frac{1}{2}}\bs^{(k)})]\|^2
\end{align}

Next we introduce an auxiliary sequence  $\bar{t}^{(k)}$ defined as
\begin{equation}\label{eqn:t-def-2-0}
\bar{t}^{(k)}=\begin{cases}
\bar{s}^{(0)}&k=0,\\
\frac{1}{1-\beta}\bar{s}^{(k)}-\frac{\beta}{1-\beta}\bar{s}^{(k-1)}&k\geq 1.
\end{cases}
\end{equation}
The introduction of $\bar{t}^{(k)}$ is inspired from \cite{yu2019linear}. It is delicately designed to enjoy the following property:
\begin{lemma}\label{lem:ttt-0}
	$\bar{t}^{(k)}$ defined in \eqref{eqn:t-def-2-0} satisfies 
	\begin{align}
	\bar{t}^{(k+1)}&=\bar{t}^{(k)}-\frac{\gamma}{n(1-\beta)} \mathds{1}^T \nabla F(W^\frac{1}{2}\bs^{(k)};\bxi^{(k)}) \label{t-recursion-0}\\
	\bar{t}^{(k)} - \bar{s}^{(k)} &= - \frac{\gamma \beta}{(1-\beta)^2}\bar{m}_s^{(k)} \label{t-recursion-1}
	\end{align}
\end{lemma}
\begin{proof}
	We prove the lemma by direct calculation. When $k=1$,
	$$
	\bar{t}^{(1)}-\bar{t}^{(0)}=\frac{1}{1-\beta} \bar{s}^{(1)}-\frac{\beta}{1-\beta} \bar{s}^{(0)}-\bar{s}^{(0)}=\frac{1}{1-\beta}\left(\bar{s}^{(1)}-\bar{s}^{(0)}\right)=-\frac{\gamma}{n(1-\beta)} \mathds{1}^T \nabla F(W^\frac{1}{2}\bs^{(0)};\bxi^{(0)}).
	$$
	where the last equality holds because of recursions \eqref{eqn:decentlam-7-0}-\eqref{eqn:decentlam-8-0}. 
	For $k \geq 1$, we have
	$$
	\begin{aligned}
	\bar{t}^{(k+1)}-\bar{t}^{(k)} &\overset{\eqref{eqn:t-def-2-0}}{=}\frac{1}{1-\beta}\left(\bar{s}^{(k+1)}-\bar{s}^{(k)}\right)-\frac{\beta}{1-\beta}\left(\bar{s}^{(k)}-\bar{s}^{(k-1)}\right) \\
	&\overset{\eqref{eqn:decentlam-8-0}}{=}\frac{1}{1-\beta}\left(-\frac{\gamma}{1-\beta}\bar{m}_{\bs}^{(k+1)}\right)-\frac{\beta}{1-\beta}\left(-\frac{\gamma}{1-\beta}\bar{m}_{\bs}^{(k)}\right) \\
	&=-\frac{\gamma}{1-\beta}\left(\frac{1}{1-\beta}\bar{m}_{\bs}^{(k+1)}-\frac{\beta}{1-\beta}\bar{m}_{\bs}^{(k)}\right) \\
	&\overset{\eqref{eqn:decentlam-7-0}}{=}-\frac{\gamma}{n(1-\beta)} \mathds{1}^T \nabla F(W^\frac{1}{2}\bs^{(k)};\bxi^{(k)}).
	\end{aligned}
	$$
	which completes the proof of \eqref{t-recursion-0}. Moreover, it holds for $k\ge 1$ that 
	\begin{align}
	\bar{t}^{(k)} - \bar{s}^{(k)} \overset{\eqref{eqn:t-def-2-0}}{=} \frac{\beta}{1-\beta}\big( \bar{s}^{(k)} - \bar{s}^{(k-1)} \big) \overset{\eqref{eqn:decentlam-8-0}}{=} - \frac{\gamma \beta}{(1-\beta)^2}\bar{m}_s^{(k)}.
	\end{align}
\end{proof}

\subsection{Descent Lemma}
The following lemma establishes how $\mathbb{E}[f(\bar{t}^{(k)})]$ evolves with iteration $k$.  
\begin{lemma}[\sc Descent Lemma]
Under Assumptions A.1-A.3, it holds that 
\begin{align}\label{x762bzz}
\mathbb{E}[f(\bar{t}^{(k+1)})] &\le \mathbb{E}[f(\bar{t}^{(k)})]  - \frac{\gamma}{2(1-\beta)}\mathbb{E}\| \nabla f( \bar{s}^{(k)}) \|^2  + \frac{\gamma L^2}{2n(1-\beta)}\mathbb{E}\|\bar{\bs}^{(k)} - \bs^{(k)}\|^2 +  \frac{\gamma^2 L\sigma^2}{(1-\beta)n} \nonumber \\
&\quad - \Big(\frac{\gamma}{2(1-\beta)} - \frac{\gamma^2 L}{(1-\beta)^3} - \frac{L\gamma^2}{2(1-\beta)^2} - \frac{\gamma^2 \beta^2 L}{(1-\beta)^3}\Big)\mathbb{E}\|\frac{1}{n}\mathds{1}^T \nabla f(W^{\frac{1}{2}}\bs^{(k)})\|^2 \nonumber \\
&\quad + \frac{\gamma^2 \beta^2 L}{(1-\beta)^3} \mathbb{E}\|(1-\beta)\sum_{i=0}^{k-1}\beta^{k-1-i} [\frac{1}{n}\mathds{1}^T\nabla f(W^{\frac{1}{2}}\bs^{(i)}) - \frac{1}{n}\mathds{1}^T\nabla f(W^{\frac{1}{2}}\bs^{(k)})]\|^2.
\end{align} \label{lm-descent-general}
\end{lemma}
\begin{proof}
It holds from Lemma \ref{lem:ttt-0} that 
\begin{align}\label{oxhn}
\bar{t}^{(k+1)} = \bar{t}^{(k)} - \frac{\gamma}{1-\beta} \frac{1}{n}\mathds{1}^T \nabla F(W^{\frac{1}{2}}\bs^{(k)};\bxi^{(k)})
\end{align}
Since $f(x)$ is $L$-smooth, it holds that 
\begin{align}\label{eqn:16fdsf-new}
&\mathbb{E}_{\bxi^{(k)}}\left[f\left(\bar{t}^{(k+1)}\right)\right] \nonumber \\
& \leq f\left(\bar{t}^{(k)}\right) + \mathbb{E}_{\bxi^{(k)}}\left[\left\langle\nabla f\left( \bar{t}^{(k)}\right), \bar{t}^{(k+1)}-\bar{t}^{(k)}\right\rangle\right] + \frac{L}{2} \mathbb{E}_{\bxi^{(k)}}\left[\left\|\bar{t}^{(k+1)}-\bar{t}^{(k)}\right\|^{2}\right] \nonumber \\ 
&\overset{\eqref{oxhn}}{=} f\left(\bar{t}^{(k)}\right) -  \underbrace{\frac{\gamma}{1-\beta} \left[\left\langle\nabla f\left( \bar{t}^{(k)}\right), \frac{1}{n}\mathds{1}^T \nabla f(W^{\frac{1}{2}}\bs^{(k)})\right\rangle\right]}_{= A} + \underbrace{\frac{L\gamma^2}{2(1-\beta)^2} \mathbb{E}_{\bxi^{(k)}}\left[\left\|\frac{1}{n}\mathds{1}^T \nabla F(W^{\frac{1}{2}}\bs^{(k)};\bxi^{(k)})\right\|^{2}\right]}_{=B}
\end{align}
We first bound term $A$. Note that 
\begin{align}\label{eq-A}
A = \underbrace{- \frac{\gamma}{1-\beta} \left\langle\nabla f\left( \bar{t}^{(k)}\right) - \nabla f\left( \bar{s}^{(k)}\right), \frac{1}{n}\mathds{1}^T \nabla f(W^{\frac{1}{2}}\bs^{(k)})\right\rangle}_{=A_1} - \underbrace{\frac{\gamma}{1-\beta} \left\langle \nabla f\left( \bar{s}^{(k)}\right), \frac{1}{n}\mathds{1}^T \nabla f(W^{\frac{1}{2}}\bs^{(k)})\right\rangle}_{=A_2}
\end{align}
With Cauchy-Schwarz inequality, we can bound $A_1$ as 
\begin{align}\label{bound-A1}
A_1 &\le \frac{\gamma}{1-\beta}\Big( \frac{\eta}{2} \|\nabla f\left( \bar{t}^{(k)}\right) - \nabla f\left( \bar{s}^{(k)}\right)\|^2 + \frac{1}{2\eta}\|\frac{1}{n}\mathds{1}^T \nabla f(W^{\frac{1}{2}}\bs^{(k)})\|^2\Big)  \nonumber \\
&= \frac{(1-\beta)L}{4}\| \bar{t}^{(k)} - \bar{s}^{(k)} \|^2 + \frac{\gamma^2 L}{(1-\beta)^3}\|\frac{1}{n}\mathds{1}^T \nabla f(W^{\frac{1}{2}}\bs^{(k)})\|^2.
\end{align}
where the last equality holds by letting $\eta = \frac{(1-\beta)^2}{2L\gamma}$. 
The bound of $A_2$ is 
\begin{align}\label{bound-A2}
A_2 &= -\frac{\gamma}{2(1-\beta)}\big( \| \nabla f( \bar{s}^{(k)}) \|^2 + \|\frac{1}{n}\mathds{1}^T \nabla f(W^{\frac{1}{2}}\bs^{(k)})\|^2 - \|\nabla f( \bar{s}^{(k)}) -  \frac{1}{n}\mathds{1}^T \nabla f(W^{\frac{1}{2}}\bs^{(k)})\|^2\big) \nonumber \\
&\overset{(a)}{\le} -\frac{\gamma}{2(1-\beta)}\| \nabla f( \bar{s}^{(k)}) \|^2 - \frac{\gamma}{2(1-\beta)}\|\frac{1}{n}\mathds{1}^T \nabla f(W^{\frac{1}{2}}\bs^{(k)})\|^2 + \frac{\gamma L^2}{2n(1-\beta)}\|W^{\frac{1}{2}}(\bar{\bs}^{(k)} - \bs^{(k)})\|^2 \nonumber \\
&\le -\frac{\gamma}{2(1-\beta)}\| \nabla f( \bar{s}^{(k)}) \|^2 - \frac{\gamma}{2(1-\beta)}\|\frac{1}{n}\mathds{1}^T \nabla f(W^{\frac{1}{2}}\bs^{(k)})\|^2 + \frac{\gamma L^2}{2n(1-\beta)}\|\bar{\bs}^{(k)} - \bs^{(k)}\|^2.
\end{align}
where (a) holds because $\nabla f(\bar{s}^{(k)}) = \frac{1}{n}\mathds{1}^T\nabla f(\bar{\bs}^{(k)}) = \frac{1}{n}\mathds{1}^T\nabla f(W^{\frac{1}{2}}\bar{\bs}^{(k)})$.
Substituting \eqref{bound-A1} and \eqref{bound-A2} into \eqref{eq-A}, we achieve
\begin{align}\label{xzbsd67236}
A &\le \frac{(1-\beta)L}{4}\| \bar{t}^{(k)} - \bar{s}^{(k)} \|^2 - \frac{\gamma}{2(1-\beta)}\| \nabla f( \bar{s}^{(k)}) \|^2 \nonumber \\
&\quad - \Big(\frac{\gamma}{2(1-\beta)} - \frac{\gamma^2 L}{(1-\beta)^3}\Big)\|\frac{1}{n}\mathds{1}^T \nabla f(W^{\frac{1}{2}}\bs^{(k)})\|^2 + \frac{\gamma L^2}{2n(1-\beta)}\|\bar{\bs}^{(k)} - \bs^{(k)}\|^2.
\end{align}
Next we bound term $B$. 
\begin{align}\label{xh23vs6}
\mathbb{E}_{\bxi^{(k)}}\left\|\frac{1}{n}\mathds{1}^T \nabla F(W^{\frac{1}{2}}\bs^{(k)};\bxi^{(k)})\right\|^{2} \le \|\frac{1}{n}\mathds{1}^T \nabla f(W^{\frac{1}{2}}\bs^{(k)})\|^2 + \frac{\sigma^2}{n}.
\end{align}
Furthermore, recall from \eqref{t-recursion-1} that 
\begin{align}\label{23bnx7}
\bar{t}^{(k)} - \bar{s}^{(k)} = - \frac{\gamma \beta}{(1-\beta)^2} \bar{m}_s^{(k)}
\end{align}
Substituting \eqref{xzbsd67236}--\eqref{23bnx7} into \eqref{eqn:16fdsf-new} and taking expectations over $\{\bxi^{(\ell)}\}_{\ell=0}^{k-1}$, we achieve
\begin{align}\label{x236sdvbsd9}
\mathbb{E}[f(\bar{t}^{(k+1)})] &\le \mathbb{E}[f(\bar{t}^{(k)})] + \frac{\gamma^2 \beta^2 L}{4(1-\beta)^3}\mathbb{E}\|\bar{m}_s^{(k)} \|^2 - \frac{\gamma}{2(1-\beta)}\mathbb{E}\| \nabla f( \bar{s}^{(k)}) \|^2  + \frac{\gamma^2 L\sigma^2}{2(1-\beta)^2n} \nonumber \\
&\quad - \Big(\frac{\gamma}{2(1-\beta)} - \frac{\gamma^2 L}{(1-\beta)^3} - \frac{L\gamma^2}{2(1-\beta)^2}\Big)\mathbb{E}\|\frac{1}{n}\mathds{1}^T \nabla f(W^{\frac{1}{2}}\bs^{(k)})\|^2 + \frac{\gamma L^2}{2n(1-\beta)}\mathbb{E}\|\bar{\bs}^{(k)} - \bs^{(k)}\|^2.
\end{align}
Substituting the bound of $\mathbb{E}\|\bar{m}_s^{(k)} \|^2$ in Lemma \ref{lm-bound-ms} (set $\vy = W^{\frac{1}{2}}\bs^{(k)}$) into the above inequality, we achieve \eqref{x762bzz}.
\end{proof}

\subsection{Consensus Lemma}
It is observed that a consensus error term $\mathbb{E}\|\bs^{(k)} - \bar{\bs}^{(k)}\|^2$ exists in \eqref{x762bzz}. In this subsection, we establish its upper bound. 

\subsubsection{Non-convex scenario}

\begin{lemma} [\sc Descent Lemma for $ \mathbb{E}\|\bs^{(k)} - \bar{\bs}^{(k)}\|^2$]  \label{lemma-s-descent}
Under Assumptions A.1-A.4, if learning rate $\gamma$ is sufficiently small such that $\gamma \le \frac{1-\rho}{4 \sqrt{\rho} L}$, it holds that 
\begin{align}\label{s-consensus}
\mathbb{E}\|\bs^{(k+1)} - \bar{\bs}^{(k+1)}\|^2 \le&\ \frac{1+\rho}{2} \mathbb{E}\|\bs^{(k)} - \bar{\bs}^{(k)}\|^2 + \frac{2\gamma^2 \beta^2}{(1-\beta)^2(1-\rho)}\mathbb{E}\|\vm_s^{(k)} - \bar{\vm}_s^{(k)}\|^2 \nonumber \\
&\hspace{1cm} + \frac{6n \rho \gamma^2}{1-\rho} \mathbb{E}\|\nabla f(\bar{s}^k)\|^2 + \frac{6 n \rho\gamma^2\hat{b}^2}{1-\rho} + n\rho \gamma^2\sigma^2
\end{align}
\end{lemma}

\begin{proof}
\noindent Substituting \eqref{eqn:decentlam-7} into \eqref{eqn:decentlam-8}, we achieve:
\begin{align}\label{x72ds6}
\bs^{(k+1)} = W\bs^{(k)} - \frac{\gamma \beta}{1-\beta}\vm_s^{(k)} - \gamma W^{\frac{1}{2}} \nabla F(W^{\frac{1}{2}} \bs^{(k)};\bxi^{(k)})
\end{align}
By multiplying $R=\frac{1}{n}\mathds{1}\mathds{1}^T$ to both sides of the above recursion, we achieve
\begin{align}\label{xbs6sd5}
\bar{\bs}^{(k+1)} = \bar{\bs}^{(k)} - \frac{\gamma \beta}{1-\beta} \bar{\vm}_s^{(k)} - \gamma R \nabla  F(W^{\frac{1}{2}} \bs^{(k)};\bxi^{(k)}).
\end{align}
Subtracting \eqref{xbs6sd5} from \eqref{x72ds6}, we obtain
\begin{align}
\bs^{(k+1)} - \bar{\bs}^{(k+1)} = (W - R)(\bs^{(k)} - \bar{\bs}^{(k)}) - \frac{\gamma \beta}{1-\beta} (\vm_s^{(k)} - \bar{\vm}_s^{(k)}) - \gamma (W^{\frac{1}{2}}-R) \nabla  F(W^{\frac{1}{2}} \bs^{(k)};\bxi^{(k)}).
\end{align}
By taking mean-square-expectation, we have
\begin{align}\label{237sdh}
&\ \mathbb{E}\|\bs^{(k+1)} - \bar{\bs}^{(k+1)}\|^2 \nonumber \\
\overset{(a)}{\le}& \mathbb{E} \|(W - R)(\bs^{(k)} - \bar{\bs}^{(k)}) - \frac{\gamma \beta}{1-\beta} (\vm_s^{(k)} - \bar{\vm}_s^{(k)}) - \gamma (W^{\frac{1}{2}}-R) \nabla  f(W^{\frac{1}{2}} \bs^{(k)})\|^2 + n \rho \gamma^2 \sigma^2 \nonumber \\
\overset{(b)}{\le}&\ \frac{1}{\eta}\|W - R\|^2_2\, \mathbb{E}\|\bs^{(k)} - \bar{\bs}^{(k)}\|^2 + \frac{1}{1-\eta} \Big( \frac{2\gamma^2 \beta^2}{(1-\beta)^2}\mathbb{E}\|\vm_s^{(k)} - \bar{\vm}_s^{(k)}\|^2 + 2\gamma^2 \rho \mathbb{E}\|\nabla  f(W^{\frac{1}{2}} \bs^{(k)})\|^2 \Big) + n \rho \gamma^2 \sigma^2\nonumber \\
\overset{(c)}{\le}&\ \rho \mathbb{E}\|\bs^{(k)} - \bar{\bs}^{(k)}\|^2 + \frac{2\gamma^2 \beta^2}{(1-\beta)^2(1-\rho)}\mathbb{E}\|\vm_s^{(k)} - \bar{\vm}_s^{(k)}\|^2 + \frac{2 \rho \gamma^2}{1-\rho} \mathbb{E}\|\nabla  f(W^{\frac{1}{2}} \bs^{(k)})\|^2 + n \rho \gamma^2\sigma^2
\end{align}
where (a) holds because of Assumption A.2 and $\|W^{\frac{1}{2}}-R\|_2^2\le \rho$, (b) holds because of the Jensen's inequality for any $\eta \in (0,1)$, and (c) holds because $\|W - R\|^2_2 \le \rho^2$ and $\eta = \rho$.
Note that 
\begin{align}\label{grad-f-Ws}
\mathbb{E}\|\nabla  f(W^{\frac{1}{2}} \bs^{(k)})\|^2 &= \sum_{i=1}^n \mathbb{E}\|\nabla f_i([W^{\frac{1}{2}} \bs^{(k)}]_i)\|^2 \nonumber \\
&= \sum_{i=1}^n \mathbb{E}\|\nabla f_i([W^{\frac{1}{2}} \bs^{(k)}]_i) - \nabla f_i([W^{\frac{1}{2}} \bar{\bs}^{(k)}]_i) + \nabla f_i([W^{\frac{1}{2}} \bar{\bs}^{(k)}]_i) - \nabla f(\bar{s}^{(k)}) + \nabla f(\bar{s}^{(k)})\|^2 \nonumber \\
&\le 3 L^2 \|\bs^{(k)} - \bar{\bs}^{(k)}\|^2 + 3n \hat{b}^2 + 3n \|\nabla f(\bar{s}^{(k)})\|^2.
\end{align}
where the last inequality holds because of Assumptions A.1 and A.4, and $[W^{\frac{1}{2}} \bar{\bs}^{(k)}]_i = [\bar{\bs}^{(k)}]_i = \bar{s}^{(k)}$. Substituting the above inequality to \eqref{237sdh}, we have
\begin{align}
\mathbb{E}\|\bs^{(k+1)} - \bar{\bs}^{(k+1)}\|^2 \le&\ \rho \mathbb{E}\|\bs^{(k)} - \bar{\bs}^{(k)}\|^2 + \frac{2\gamma^2 \beta^2}{(1-\beta)^2(1-\rho)}\mathbb{E}\|\vm_s^{(k)} - \bar{\vm}_s^{(k)}\|^2 + \frac{6\rho \gamma^2L^2}{1-\rho} \mathbb{E}\|\bs^{(k)} - \bar{\bs}^{(k)}\|^2 \nonumber \\
&\hspace{1cm} + \frac{6\rho n\gamma^2}{1-\rho} \mathbb{E}\|\nabla f(\bar{s}^k)\|^2 + \frac{6\rho n\gamma^2\hat{b}^2}{1-\rho} + n \rho \gamma^2\sigma^2
\end{align}
If learning rate $\gamma$ is sufficiently small such that $\rho + \frac{6 \rho \gamma^2L^2}{1-\rho}  \le \frac{1+\rho}{2}$, which can be satisfied by letting $\gamma \le \frac{1-\rho}{4\sqrt{\rho}L}$, then the result in \eqref{s-consensus} holds. 
\end{proof}

\begin{lemma} [\sc Descent Lemma for $\mathbb{E}\|\vm_s^{(k)} - \bar{\vm}_s^{(k)}\|^2$] \label{lemma-momentum-consensus}
Under Assumptions A.1-A.3, it holds that 
\begin{align}
\mathbb{E}\|\vm_s^{(k+1)} - \bar{\vm}_s^{(k+1)}\|^2 &\le \beta \mathbb{E}\|\vm_s^{(k)} - \bar{\vm}_s^{(k)}\|^2 + \frac{2(1-\beta)(1+3\rho \gamma^2 L^2)}{\gamma^2} \mathbb{E}\|\bs^{(k)} - \bar{\bs}^{(k)}\|^2 + 6(1-\beta)\rho n\mathbb{E}\|\nabla f(\bar{s}^k)\|^2 \nonumber \\
&\hspace{1cm} + 6\rho n(1-\beta)\hat{b}^2 + (1-\beta)^2 \rho n\sigma^2
\end{align}
\end{lemma}
\begin{proof}
Recall from \eqref{eqn:decentlam-7} and  \eqref{eqn:decentlam-7-0} that 
\begin{align}
\vm_s^{(k+1)} - \bar{\vm}_s^{(k+1)} = \beta(\vm_s^{(k)} - \bar{\vm}_s^{(k)}) + (1-\beta)(W^\frac{1}{2} - R) \nabla F(W^{\frac{1}{2}}\bs^{(k)};\bxi^{(k)}) + \frac{1-\beta}{\gamma}(I-W)(\bs^{(k)} - \bar{\bs}^{(k)}).
\end{align}
By taking mean-square-expectation, it holds that 
\begin{align}
&\ \mathbb{E}\|\vm_s^{(k+1)} - \bar{\vm}_s^{(k+1)}\|^2 \nonumber \\
&\le  \mathbb{E}\|\beta(\vm_s^{(k)} - \bar{\vm}_s^{(k)}) + (1-\beta)(W^\frac{1}{2} - R) \nabla f(W^{\frac{1}{2}}\bs^{(k)}) + \frac{1-\beta}{\gamma}(I-W)(\bs^{(k)} - \bar{\bs}^{(k)})\|^2 + (1-\beta)^2 n \rho \sigma^2 \nonumber \\
&\le \frac{\beta^2}{\eta}\mathbb{E}\|\vm_s^{(k)} - \bar{\vm}_s^{(k)}\|^2  +  \frac{2(1-\beta)^2}{\gamma^2(1-\eta)}\mathbb{E}\|\bs^{(k)} - \bar{\bs}^{(k)}\|^2 + \frac{2(1-\beta)^2 \rho}{1-\eta}\mathbb{E}\|\nabla f(W^{\frac{1}{2}}\bs^{(k)})\|^2 + (1-\beta)^2 n \rho \sigma^2  \nonumber \\
&\overset{(a)}{=} \beta \mathbb{E}\|\vm_s^{(k)} - \bar{\vm}_s^{(k)}\|^2 + 2(1-\beta) \rho \mathbb{E}\|\nabla f(W^{\frac{1}{2}}\bs^{(k)})\|^2  + \frac{2(1-\beta)}{\gamma^2}\mathbb{E}\|\bs^{(k)} - \bar{\bs}^{(k)}\|^2 + (1-\beta)^2 n \rho \sigma^2 \nonumber \\
&\overset{\eqref{grad-f-Ws}}{\le}\beta \mathbb{E}\|\vm_s^{(k)} - \bar{\vm}_s^{(k)}\|^2 + \frac{2(1-\beta)(1+3 \rho \gamma^2 L^2)}{\gamma^2} \mathbb{E}\|\bs^{(k)} - \bar{\bs}^{(k)}\|^2 + 6(1-\beta)\rho n\mathbb{E}\|\nabla f(\bar{s}^k)\|^2 \nonumber \\
&\hspace{1cm} + 6n(1-\beta)\rho \hat{b}^2 + (1-\beta)^2 n \rho \sigma^2
\end{align}
where (a) holds by setting $\eta = \beta$.
\end{proof}
With the help of Lemmas \ref{lemma-s-descent} and \ref{lemma-momentum-consensus}, we are ready to establish the bound for consensus error term $\mathbb{E}\|\bs^{(k+1)} - \bar{\bs}^{(k+1)}\|^2$. 

\begin{lemma}[\sc Consensus Lemma for Non-convex Scenario]\label{lem:sc-consensus}
Under Assumptions A.1-A.4, if $\gamma \le \frac{1-\rho}{2\sqrt{\rho}L}$ and $\frac{32\beta^2}{(1-\beta)(1-\rho)^2} +\beta  \le \frac{3+\rho}{4}$, it holds that 
\begin{align}\label{x7623gsdb9}
\mathbb{E}\|\bs^{(k+1)} - \bar{\bs}^{(k+1)}\|^2  \le \frac{12n\rho \gamma^2}{1-\rho} \sum_{\ell=0}^k \big(\frac{3+\rho}{4}\big)^{k-\ell} \mathbb{E}\|\nabla f(\bar{s}^{(\ell)})\|^2 +\frac{48 n \rho \gamma^2\hat{b}^2}{(1-\rho)^2} + \frac{8  n \rho \gamma^2\sigma^2}{1-\rho}
\end{align}
If we further take running average over both sides, it holds that  
\begin{align}\label{nc-consensus-2}
\frac{1}{T+1}\sum_{k=0}^{T} \mathbb{E}\|\bs^{(k)} - \bar{\bs}^{(k)}\|^2 \le \frac{48n\rho \gamma^2}{(1-\rho)^2(T+1)} \sum_{k=0}^{T} \mathbb{E}\|\nabla f(\bar{s}^{(k)})\|^2+ \frac{48 n \rho \gamma^2\hat{b}^2}{(1-\rho)^2} + \frac{8 n\rho \gamma^2\sigma^2}{1-\rho}.
\end{align}
\end{lemma}

\begin{proof}
For notation simplicity, we let 
\begin{align}
&A^{(k)} = \mathbb{E}\|\bs^{(k)} - \bar{\bs}^{(k)}\|^2, \quad B^{(k)} = \mathbb{E}\|\vm_s^{(k)} - \bar{\vm}_s^{(k)}\|^2 \nonumber \\
&C_1^{(k)} = \frac{6 n \rho \gamma^2}{1-\rho} \mathbb{E}\|\nabla f(\bar{s}^{(k)})\|^2 + \frac{6 n \rho \gamma^2\hat{b}^2}{1-\rho} + n \rho \gamma^2\sigma^2 \nonumber \\
&C_2^{(k)} = 6(1-\beta)\rho n\mathbb{E}\|\nabla f(\bar{s}^{(k)})\|^2 + 6\rho n(1-\beta)\hat{b}^2 + (1-\beta)^2 \rho n\sigma^2 \nonumber 
\end{align}
From Lemmas \ref{lemma-s-descent} and \ref{lemma-momentum-consensus}, it holds that
\begin{align}
A^{(k+1)} &\le \frac{1+\rho}{2} A^{(k)} + \frac{2\gamma^2 \beta^2}{(1-\beta)^2(1-\rho)}B^{(k)} + C_1^{(k)} \\
B^{(k+1)} &\le \beta B^{(k)} + \frac{2(1-\beta)(1+3 \rho \gamma^2 L^2)}{\gamma^2} A^{(k)} + C_2^{(k)}
\end{align}
From the above two inequalities, we have (fro a positive constant $a$):
\begin{align}\label{zx2376sd}
A^{(k+1)} + a \gamma^2 B^{(k+1)} \le&\ \Big(\frac{1+\rho}{2} + 2a(1-\beta)(1+3 \rho \gamma^2L^2)\Big) A^{(k)}  + \Big( \frac{2 \gamma^2 \beta^2}{(1-\beta)^2(1-\rho)} + a \beta \gamma^2  \Big) B^{(k)}  + C_1^{(k)} + a\gamma^2 C_2^{(k)} \nonumber \\
\overset{(a)}{\le}&\ \Big(\frac{1+\rho}{2} + 4a(1-\beta)\Big) A^{(k)} + \Big( \frac{2\beta^2}{a (1-\beta)^2(1-\rho)} + \beta  \Big)a\gamma^2 B^{(k)}  + C_1^{(k)} + a\gamma^2 C_2^{(k)} \nonumber \\
\overset{(b)}{=}&\ \frac{3+\rho}{4} A^{(k)} + \Big( \frac{32\beta^2}{(1-\beta)(1-\rho)^2} +\beta \Big) a\gamma^2 B^{(k)} + C_1^{(k)} + a\gamma^2 C_2^{(k)}
\end{align}
where (a) holds when $\gamma$ is sufficiently small such that $1+3 \rho \gamma^2 L^2 \le 2$ (which can be satisfied by setting $\gamma \le \frac{1}{2\sqrt{\rho}L}$), and (b) holds when $a = \frac{1-\rho}{16(1-\beta)}$. If $\beta$ satisfies
\begin{align}\label{23bsdbz121}
\frac{32\beta^2}{(1-\beta)(1-\rho)^2} +\beta  \le \frac{3+\rho}{4}
\end{align}
inequality \eqref{zx2376sd} becomes 
\begin{align}
A^{(k+1)} + a \gamma^2 B^{(k+1)} &\le \frac{3+\rho}{4} \Big( A^{(k)} + a \gamma^2 B^{(k)}\Big) + C_1^{(k)} + a\gamma^2 C_2^{(k)} \nonumber \\
&\le \frac{3+\rho}{4} \Big( A^{(k)} + a \gamma^2 B^{(k)}\Big) + \Big(\frac{1}{1-\rho} + \frac{1-\rho}{16} \Big)6n \rho \gamma^2\mathbb{E}\|\nabla f(\bar{s}^{(k)})\|^2 \nonumber \\
&\quad + \Big( \frac{1}{1-\rho} + \frac{1-\rho}{16} \Big)6 n \rho \gamma^2\hat{b}^2 + n \rho \gamma^2 \sigma^2\big( 1 + \frac{(1-\rho)(1-\beta)}{16} \big) \nonumber \\
&\le \frac{3+\rho}{4} \Big( A^{(k)} + a \gamma^2 B^{(k)}\Big)  + \frac{12 n \rho \gamma^2}{1-\rho}\mathbb{E}\|\nabla f(\bar{s}^{(k)})\|^2 + \frac{12 n \rho \gamma^2\hat{b}^2}{1-\rho} + {2 n \rho \gamma^2\sigma^2}
\end{align}
If we keep iterating the above inequality, it holds that 
\begin{align}\label{xc2w76}
&\ A^{(k+1)} + a \gamma^2 B^{(k+1)} \nonumber \\
\le&\ \big(\frac{3+\rho}{4}\big)^{k+1} \big(A^{(0)} + a \gamma^2 B^{(0)}\big) + \frac{12n\rho\gamma^2}{1-\rho} \sum_{\ell=0}^k \big(\frac{3+\rho}{4}\big)^{k-\ell} \mathbb{E}\|\nabla f(\bar{s}^{(\ell)})\|^2 +\frac{48 n \rho \gamma^2\hat{b}^2}{(1-\rho)^2} + \frac{8 n \rho \gamma^2\sigma^2}{1-\rho} \nonumber \\
=&\ \frac{12n\rho \gamma^2}{1-\rho} \sum_{\ell=0}^k \big(\frac{3+\rho}{4}\big)^{k-\ell} \mathbb{E}\|\nabla f(\bar{s}^{(\ell)})\|^2 +\frac{48 n \rho \gamma^2\hat{b}^2}{(1-\rho)^2} + \frac{8 n \rho \gamma^2\sigma^2}{1-\rho}
\end{align}
where the last equality holds because $A^{(0)}=0$ (by setting $\bs^{(0)} = 0$) and $B^{(0)} = 0$ (by setting $\bar{\vm}_s^{(0)} = 0$).  Since $B^{(k)} \ge 0$, we achieve the result in \eqref{x7623gsdb9}. If taking the running average of both sides of \eqref{x7623gsdb9}, we have
\begin{align}
\frac{1}{T+1}\sum_{k=0}^{T} A^{(k)} &\le \frac{12n\rho\gamma^2}{(1-\rho)(T+1)} \sum_{k=1}^{T} \sum_{\ell=0}^{k-1} \big(\frac{3+\rho}{4}\big)^{k-1-\ell} \mathbb{E}\|\nabla f(\bar{s}^{(\ell)})\|^2 + \frac{48 n \rho \gamma^2\hat{b}^2}{(1-\rho)^2} + \frac{8 n \rho \gamma^2\sigma^2}{1-\rho} \nonumber \\
&= \frac{12n\rho\gamma^2}{(1-\rho)(T+1)} \sum_{\ell=0}^{T-1} \sum_{k=\ell+1}^{T}  \big(\frac{3+\rho}{4}\big)^{k-1-\ell} \mathbb{E}\|\nabla f(\bar{s}^{(\ell)})\|^2+ \frac{48 n \rho \gamma^2\hat{b}^2}{(1-\rho)^2} + \frac{8 n \rho \gamma^2\sigma^2}{1-\rho} \nonumber \\
&\le \frac{48 n \rho \gamma^2}{(1-\rho)^2(T+1)} \sum_{\ell=0}^{T-1} \mathbb{E}\|\nabla f(\bar{s}^{(\ell)})\|^2+ \frac{48 n \rho \gamma^2\hat{b}^2}{(1-\rho)^2} + \frac{8  n\rho \gamma^2\sigma^2}{1-\rho} \nonumber \\
&\le \frac{48n\rho \gamma^2}{(1-\rho)^2(T+1)} \sum_{k=0}^{T} \mathbb{E}\|\nabla f(\bar{s}^{(k)})\|^2+ \frac{48 n \rho \gamma^2\hat{b}^2}{(1-\rho)^2} + \frac{8  n\gamma^2\sigma^2}{1-\rho}.
\end{align}
which completes the proof.
\end{proof}

\subsubsection{Strongly-convex scenario}

Under Assumption A.5, it holds that each function $f_i(x)$ is strongly convex. Let $x^\star$ be the unique global solution to problem \eqref{dist-opt}. Since ${\bs}^\star = W^{\frac{1}{2}}\vx^\star = \vx^\star$, we know $s^\star = x^\star$ and hence $f(s^\star) = f(x^\star) = f^\star$. The proof of the following lemma is inspired by \cite{koloskova2020unified}.

\begin{lemma}[\sc Consensus Lemma for Strongly-convex Scenario]
	Under Assumptions A.1--A.3 and A.5, if $\gamma \le \frac{1-\rho}{2\sqrt{\rho}L}$ and $\frac{32\beta^2}{(1-\beta)(1-\rho)^2} +\beta  \le \frac{3+\rho}{4}$, it holds that 
	\begin{align}\label{x7623gsdb9-sc-9}
	\mathbb{E}\|\bs^{(k+1)} - \bar{\bs}^{(k+1)}\|^2   \le \frac{24 n \rho L \gamma^2}{1-\rho} \sum_{\ell=0}^k \big(\frac{3+\rho}{4}\big)^{k-\ell} \big( \mathbb{E} f(\bar{s}^{(\ell)}) -  f^\star \big)  + \frac{48 n \rho \gamma^2 {b}^2}{(1-\rho)^2} + \frac{8 n \rho \gamma^2\sigma^2}{1-\rho}.
	\end{align}
	where $b^2 = \frac{1}{n}\sum_{i=1}^n\|\nabla f_i(s^\star)\|^2 = \frac{1}{n}\sum_{i=1}^n\|\nabla f_i(x^\star)\|^2$.	Furthermore, if we introduce weights $\{h_k\}_{k=0}^\infty$ such that 
	\begin{align}\label{w-cond-9}
	h_k \le h_\ell  \big(1 + \frac{1-\rho}{8}\big)^{k-\ell} \mbox{for any $k\ge 0$ and $0\le \ell \le k$,}
	\end{align}
	it holds that 
	\begin{align}\label{wz8-9}
	\frac{1}{H_T}\sum_{k=0}^T h_k \mathbb{E}\|\bs^{(k)} - \bar{\bs}^{(k)}\|^2 \le \frac{216 n L \gamma^2}{(1-\rho)^2H_T}\sum_{k=0}^{T} h_k\, \big( \mathbb{E}f(\bar{s}^{(k)}) - f^\star \big) + \frac{48 n \rho \gamma^2 {b}^2}{(1-\rho)^2} + \frac{8 n \rho \gamma^2\sigma^2}{1-\rho}.
	\end{align}
	where $H_T = \sum_{k=0}^T h_k$. 
\end{lemma}
\begin{proof}
	Since $f(x)$ is convex, we have 
	\begin{align}\label{grad-f-Ws-sc}
	\mathbb{E}\|\nabla  f(W^{\frac{1}{2}} \bs^{(k)})\|^2 &= \mathbb{E}\|\nabla  f(W^{\frac{1}{2}} \bs^{(k)}) - \nabla  f(W^{\frac{1}{2}} \bar{\bs}^{(k)}) + \nabla  f(W^{\frac{1}{2}} \bar{\bs}^{(k)}) - \nabla f(\bs^\star) + \nabla f(\bs^\star) \|^2 \nonumber \\
	&\overset{(a)}{\le} 3 L^2 \|\bs^{(k)} - \bar{\bs}^{(k)}\|^2 + 3 \mathbb{E}\|\nabla  f(\bar{\bs}^{(k)}) - \nabla f(\bs^\star)\|^2+ 3nb^2 \nonumber \\
	&\overset{(b)}{\le}  3 L^2 \|\bs^{(k)} - \bar{\bs}^{(k)}\|^2 + 6nL \big( \mathbb{E} f(\bar{s}^{(k)}) -  f^\star \big) + 3nb^2
	\end{align}
	where (a) holds because each $f_i(\cdot)$ is $L$-smooth so that
	\begin{align}
	\|\nabla  f(W^{\frac{1}{2}} \bs^{(k)}) - \nabla  f({\bs}^\star)\|^2 = \|\nabla  f(W^{\frac{1}{2}} \bs^{(k)}) - \nabla  f(W^{\frac{1}{2}} {\bs}^{\star})\|^2 \le L^2 \|W^{\frac{1}{2}}(\bs^{(k)} - {\bs}^\star)\| \le L^2 \|\bs^{(k)} - {\bs}^{\star}\|^2 
	\end{align}
	and $b^2$ is defined as $b^2 = \frac{1}{n}\sum_{i=1}^n\|\nabla f_i(s^\star)\|^2 = \frac{1}{n}\sum_{i=1}^n\|\nabla f_i(x^\star)\|^2$. Inequality (b) holds because 
	\begin{align}\label{237sdbg10}
	\mathbb{E}\|\nabla  f(\bar{\bs}^{(k)}) - \nabla f(\bs^\star)\|^2 &= \sum_{i=1}^n \mathbb{E} \|\nabla f_i(\bar{s}^{(k)}) - \nabla f_i({s}^\star) \|^2 \nonumber \\
	&\overset{(c)}{\le} \sum_{i=1}^n \mathbb{E}\Big(2L \big( f_i(\bar{s}^{(k)}) -  f_i({s}^\star) \big) - \langle \nabla f_i(s^\star),  \bar{s}^{(k)} - {s}^\star \rangle \Big) \nonumber \\
	&= \sum_{i=1}^n \mathbb{E}\Big(2L \big( f_i(\bar{s}^{(k)}) -  f_i({x}^\star) \big) - \langle \nabla f_i(x^\star),  \bar{s}^{(k)} - {x}^\star \rangle \Big) \nonumber \\
	& = 2nL  \big( \mathbb{E} f(\bar{s}^{(k)}) -  f^\star \big) 
	\end{align}
	and (c) holds because each $f_i$ is convex and $L$-smooth. With inequality \eqref{grad-f-Ws-sc}, we follow arguments \eqref{x72ds6}--\eqref{xc2w76} to achieve  
	\begin{align}\label{x7623gsdb9-sc-2}
	\mathbb{E}\|\bs^{(k+1)} - \bar{\bs}^{(k+1)}\|^2  \le \frac{24 n \rho L \gamma^2}{1-\rho} \sum_{\ell=0}^k \big(\frac{3+\rho}{4}\big)^{k-\ell} \big( \mathbb{E} f(\bar{s}^{(\ell)}) -  f^\star \big)  + \frac{48 n \rho \gamma^2 {b}^2}{(1-\rho)^2} + \frac{8 n \rho \gamma^2\sigma^2}{1-\rho}.
	\end{align}
	Now we introduce weights $h_k$ and take the weighted average of both sides in the above inequality to achieve
	\begin{align}\label{23vbs41}
	\frac{1}{H_T}\sum_{k=0}^T h_k \mathbb{E}\|\bs^{(k)} - \bar{\bs}^{(k)}\|^2 &\overset{(a)}{\le} \frac{24 n \rho L \gamma^2}{(1-\rho)H_T}\sum_{k=1}^T\sum_{\ell=0}^{k-1} h_k \big(\frac{3+\rho}{4}\big)^{k-1-\ell} \big( \mathbb{E} f(\bar{s}^{(\ell)}) -  f^\star \big)  + C \nonumber \\
	&= \frac{24 n \rho L \gamma^2}{(1-\rho)H_T}\sum_{\ell=0}^{T-1}\sum_{k=\ell+1}^{T} h_k \big(\frac{3+\rho}{4}\big)^{k-1-\ell} \big( \mathbb{E} f(\bar{s}^{(\ell)}) -  f^\star \big) + C \nonumber \\
	&\overset{(b)}{\le} \frac{24 n \rho L \gamma^2}{(1-\rho)H_T}\sum_{\ell=0}^{T-1}\sum_{k=\ell+1}^{T} h_\ell \big(1 + \frac{1-\rho}{8}\big)^{k-\ell}  \big(\frac{3+\rho}{4}\big)^{k-1-\ell} \big( \mathbb{E} f(\bar{s}^{(\ell)}) -  f^\star \big) + C \nonumber \\
	&= \frac{24 n \rho L \gamma^2}{(1-\rho)H_T}\big(1 + \frac{1-\rho}{8}\big)\sum_{\ell=0}^{T-1}\sum_{k=\ell+1}^{T} h_\ell \big(1 - \frac{1-\rho}{8}\big)^{k-1-\ell}   \big( \mathbb{E} f(\bar{s}^{(\ell)}) -  f^\star \big) + C \nonumber \\
	&\le \frac{216\gamma^2}{(1-\rho)^2H_T}\sum_{\ell=0}^{T-1} h_\ell\, \big( \mathbb{E} f(\bar{s}^{(\ell)}) -  f^\star \big) + C \nonumber \\
	&\le \frac{216\gamma^2}{(1-\rho)^2H_T}\sum_{\ell=0}^{T} h_\ell\, \big( \mathbb{E} f(\bar{s}^{(\ell)}) -  f^\star \big) + C
	\end{align}
	where (a) holds because $C = \frac{48 n \rho \gamma^2 {b}^2}{(1-\rho)^2} + \frac{8 n \rho \gamma^2\sigma^2}{1-\rho}$ and $H_T = \sum_{k=0}^T h_k$, and (b) holds if weight $h_k$ satisfies \eqref{w-cond-9}. 
\end{proof}

\section{Convergence Analysis for Non-convex Scenario}
\label{app-section-non-convex}

Inspired by \cite{liu2020improved}, we consider the following Lyapunov function 
\begin{align}\label{Lya-func}
\cL^k=f(\bar{t}^{(k)})-f^\star +\sum\limits_{i=0}^{k-1}c_i\|\bar{s}^{(k-i)}-\bar{s}^{(k-i-1)}\|^2.
\end{align}

\subsection{Descent Lemma for the Lyapunov Function}

\begin{lemma}[\sc Descent Lemma for Lyapunov Function in Non-convex Scenario]
	Under Assumptions A.1--A.4, with appropriately chosen $c_i\geq 0$, if learning rate $\gamma \leq  \min\{\frac{(1-\beta)^2}{5\sqrt{\beta+\beta^2}L}, \frac{(1-\beta)^2}{(5-\beta+2\beta^2)L}, \frac{(1-\beta)^2}{12L\beta^2}\} = O(\frac{(1-\beta)^2}{L})$, it holds that 
	\begin{align}\label{lya-inequality}
	\EE[\cL^{k+1}-\cL^k]\leq -\frac{\gamma}{2(1-\beta)} \mathbb{E}\| \nabla f( \bar{s}^{(k)}) \|^2 +\sum\limits_{i=0}^k b_{k,i}\EE\|\bs^{(i)}-\bar{\bs}^{(i)}\|^2+\frac{3\gamma^2 \sigma^2 L}{2n(1-\beta)^2}
	\end{align}
	with $b_{k,k} = \frac{\gamma L^2}{n(1-\beta)}$ and  $b_{k,i} = \frac{6L^3\gamma^2 \beta^2}{n(1-\beta)^2}\beta^{k-1-i}$ for $i\in[0,k-1]$.  
\end{lemma}

\begin{proof}
Following \eqref{x762bzz} and the definition of $\cL$ in \eqref{Lya-func}, we reach
\begin{align}\label{eqn:76gfgdgds}
&\ \EE[\cL^{k+1}-\cL^k] \nonumber \\
\leq &  - \frac{\gamma}{2(1-\beta)}\mathbb{E}\| \nabla f( \bar{s}^{(k)}) \|^2  + \frac{\gamma L^2}{2n(1-\beta)}\mathbb{E}\|\bar{\bs}^{(k)} - \bs^{(k)}\|^2 +  \frac{\gamma^2 L\sigma^2}{(1-\beta)n} \nonumber \\ 
&\quad - \Big(\frac{\gamma}{2(1-\beta)} - \frac{\gamma^2 L}{(1-\beta)^3} - \frac{L\gamma^2}{2(1-\beta)^2} - \frac{\gamma^2 \beta^2 L}{(1-\beta)^3}\Big)\mathbb{E}\|\frac{1}{n}\mathds{1}^T \nabla f(W^{\frac{1}{2}}\bs^{(k)})\|^2 \nonumber \\
&\quad + \frac{\gamma^2 \beta^2 L (1-\beta^k)^2}{(1-\beta)^3} \mathbb{E}\|\frac{1-\beta}{1-\beta^k}\sum_{i=0}^{k-1}\beta^{k-1-i} [\frac{1}{n}\mathds{1}^T\nabla f(W^{\frac{1}{2}}\bs^{(i)}) - \frac{1}{n}\mathds{1}^T\nabla f(W^{\frac{1}{2}}\bs^{(k)})]\|^2 \nonumber \\
&\quad+\sum_{i=0}^{k-1}\left(c_{i+1}-c_{i}\right) \mathbb{E}\left[\left\|\bar{s}^{(k-i)}-\bar{s}^{(k-i-1)}\right\|^{2}\right]+c_0\EE\|\bar{s}^{(k+1)}-\bar{s}^{(k)}\| \nonumber \\
\overset{\eqref{bound-bar-s-diff}}{\le}  & - \frac{\gamma}{2(1-\beta)}\mathbb{E}\| \nabla f( \bar{s}^{(k)}) \|^2  + \frac{\gamma L^2}{2n(1-\beta)}\mathbb{E}\|\bar{\bs}^{(k)} - \bs^{(k)}\|^2  + \frac{\gamma^2 L\sigma^2}{(1-\beta)n}  + \frac{2c_0 \gamma^2 \sigma^2 }{n(1-\beta)} \nonumber \\
&\quad - \Big(\frac{\gamma}{2(1-\beta)} - \frac{\gamma^2 L}{(1-\beta)^3} - \frac{L\gamma^2}{2(1-\beta)^2} - \frac{\gamma^2 \beta^2 L}{(1-\beta)^3}-\frac{4c_0\gamma^2}{(1-\beta)^2}\Big)\mathbb{E}\|\frac{1}{n}\mathds{1}^T \nabla f(W^{\frac{1}{2}}\bs^{(k)})\|^2  \nonumber \\
&\quad + \Big(\frac{\gamma^2 \beta^2 L (1-\beta^k)^2}{(1-\beta)^3}+\frac{4c_0\gamma^2\beta^2(1-\beta^k)^2}{(1-\beta)^2}\Big) \mathbb{E}\|\frac{1-\beta}{1-\beta^k}\sum_{i=0}^{k-1}\beta^{k-1-i} [\frac{1}{n}\mathds{1}^T\nabla f(W^{\frac{1}{2}}\bs^{(i)}) - \frac{1}{n}\mathds{1}^T\nabla f(W^{\frac{1}{2}}\bs^{(k)})]\|^2  \nonumber \\
&\quad+\sum_{i=0}^{k-1}\left(c_{i+1}-c_{i}\right) \mathbb{E}\left[\left\|\bar{s}^{(k-i)}-\bar{s}^{(k-i-1)}\right\|^{2}\right]
\end{align}
Since $\frac{1-\beta}{1-\beta^k}\sum_{i=0}^{k-1}\beta^{k-1-i} = 1$, it holds from the Jensen's inequality that 
\begin{align}\label{2bnsba1}
&\ \mathbb{E}\|\frac{1-\beta}{1-\beta^k}\sum_{i=0}^{k-1}\beta^{k-1-i} [\frac{1}{n}\mathds{1}^T\nabla f(W^{\frac{1}{2}}\bs^{(i)}) - \frac{1}{n}\mathds{1}^T\nabla f(W^{\frac{1}{2}}\bs^{(k)})]\|^2 \nonumber \\
\le&\ \frac{1-\beta}{1-\beta^k}\sum_{i=0}^{k-1}\beta^{k-1-i} \mathbb{E}\|\frac{1}{n}\mathds{1}^T\nabla f(W^{\frac{1}{2}}\bs^{(i)}) - \frac{1}{n}\mathds{1}^T\nabla f(W^{\frac{1}{2}}\bs^{(k)}) \|^2 \nonumber \\
\overset{(a)}{\le}&\ \frac{(1-\beta)L^2}{n(1-\beta^k)}\sum_{i=0}^{k-1}\beta^{k-1-i} \mathbb{E}\|\bs^{(i)} - \bs^{(k)}\|^2 \nonumber \\
=&\ \frac{(1-\beta)L^2}{n(1-\beta^k)}\sum_{i=0}^{k-1}\beta^{k-1-i} \mathbb{E}\|(\bs^{(i)} - \bar{\bs}^{(i)}) - (\bs^{(k)} - \bar{\bs}^{(k)}) + (\bar{\bs}^{(i)} - \bar{\bs}^{(k)})\|^2 \nonumber \\
\le&\ \frac{3(1-\beta)L^2}{n(1-\beta^k)}\sum_{i=0}^{k-1}\beta^{k-1-i}\mathbb{E}\|\bs^{(i)} - \bar{\bs}^{(i)}\|^2 + \frac{3(1-\beta)L^2}{n(1-\beta^k)}\sum_{i=0}^{k-1}\beta^{k-1-i}\mathbb{E}\|\bs^{(k)} - \bar{\bs}^{(k)}\|^2 \nonumber \\
&\quad + \frac{3(1-\beta)L^2}{1-\beta^k}\sum_{i=0}^{k-1}\beta^{k-1-i}\mathbb{E}\|\bar{s}^{(i)} - \bar{s}^{(k)}\|^2 \nonumber \\
=&\frac{3(1-\beta)L^2}{n(1-\beta^k)}\sum_{i=0}^{k-1}\beta^{k-1-i}\mathbb{E}\|\bs^{(i)} - \bar{\bs}^{(i)}\|^2 + \frac{3L^2}{n}\mathbb{E}\|\bs^{(k)} - \bar{\bs}^{(k)}\|^2 + \frac{3(1-\beta)L^2}{1-\beta^k}\sum_{i=0}^{k-1}\beta^{k-1-i}\mathbb{E}\|\bar{s}^{(i)} - \bar{s}^{(k)}\|^2
\end{align}
where (a) holds because $\mathbb{E}\|\frac{1}{n}\mathds{1}^T\nabla f(W^{\frac{1}{2}}\bs^{(i)}) - \frac{1}{n}\mathds{1}^T\nabla f(W^{\frac{1}{2}}\bs^{(k)}) \|^2\le \frac{L^2}{n}\mathbb{E}\|\bs^{(i)} - \bs^{(k)}\|^2$. To bound the last term in the above inequality, we have 
\begin{align}\label{2bnsba1-2}
\frac{1-\beta}{1-\beta^k}\sum_{i=0}^{k-1}\beta^{k-1-i}\mathbb{E}\|\bar{s}^{(i)} - \bar{s}^{(k)}\|^2 &\le \frac{1-\beta}{1-\beta^k}\sum_{i=0}^{k-1}\beta^{k-1-i} (k-i) \sum_{j=i}^{k-1}\mathbb{E}\|\bar{s}^{(j+1)} - \bar{s}^{(j)}\|^2 \nonumber \\
&= \frac{1-\beta}{1-\beta^k} \sum_{j=0}^{k-1} \big(\sum_{i=0}^j \beta^{k-1-i}(k-i) \big) \mathbb{E}\|\bar{s}^{(j+1)} - \bar{s}^{(j)}\|^2 \nonumber \\
&\le  \frac{1}{1-\beta^k} \sum_{j=0}^{k-1}  \beta^{k-1-j} \big( \frac{\beta}{1-\beta} + k - j \big) \mathbb{E}\|\bar{s}^{(j+1)} - \bar{s}^{(j)}\|^2 \nonumber \\
&\overset{(a)}{=} \sum_{j=0}^{k-1} a_{k,j} \mathbb{E}\|\bar{s}^{(j+1)} - \bar{s}^{(j)}\|^2 = \sum_{i=0}^{k-1} a_{k,k-1-i} \mathbb{E}\|\bar{s}^{(k-i)} - \bar{s}^{(k-1-i)}\|^2
\end{align}
where (a) holds because we define 
\begin{align}\label{akj}
a_{k,j} = \frac{\beta^{k-1-j}}{1-\beta^k}\big( \frac{\beta}{1-\beta} + k - j \big).
\end{align}
Substituting \eqref{2bnsba1} and \eqref{2bnsba1-2} into \eqref{eqn:76gfgdgds}, and supposing $\gamma$ is sufficiently small such that 
\begin{align}\label{z8}
\frac{\gamma}{2(1-\beta)} - \frac{\gamma^2 L}{(1-\beta)^3} - \frac{L\gamma^2}{2(1-\beta)^2} - \frac{\gamma^2 \beta^2 L}{(1-\beta)^3}-\frac{4c_0\gamma^2}{(1-\beta)^2} \le 0,
\end{align}
we achieve 
\begin{align}
\EE[\cL^{k+1}-\cL^k]\leq &  - \frac{\gamma}{2(1-\beta)}\mathbb{E}\| \nabla f( \bar{s}^{(k)}) \|^2   +\Big( \frac{\gamma L^2}{2n(1-\beta)}+\frac{3 L^2}{n}\frac{(1+4\frac{c_0}{L}(1-\beta))\gamma^2 \beta^2 L (1-\beta^k)^2}{(1-\beta)^3}\Big)\mathbb{E}\|\bar{\bs}^{(k)} - \bs^{(k)}\|^2  \nonumber \\
&\quad + \frac{3 L^2}{n}\frac{(1+4\frac{c_0}{L}(1-\beta))\gamma^2 \beta^2 L}{(1-\beta)^2}\sum\limits_{i=0}^{k-1}\beta^{k-1-i}\EE\|\bs^{(i)}-\bar{\bs}^{(i)}\|^2  \nonumber \\
&\quad + \frac{\gamma^2 L\sigma^2}{(1-\beta)n} + \frac{2c_0 \gamma^2 \sigma^2 }{n(1-\beta)}  \nonumber \\
&\quad +3L^2 \Big(\frac{\gamma^2 \beta^2 L (1-\beta^k)^2}{(1-\beta)^3}+\frac{4c_0\gamma^2\beta^2(1-\beta^k)^2}{(1-\beta)^2}\Big) \sum\limits_{i=0}^{k-1}a_{k,k-1-i}\EE\|\bar{s}^{(k-i)}-\bar{s}^{(k-i-1)}\|^2  \nonumber \\
&\quad+\sum_{i=0}^{k-1}\left(c_{i+1}-c_{i}\right) \mathbb{E}\left[\left\|\bar{s}^{(k-i)}-\bar{s}^{(k-i-1)}\right\|^{2}\right]. \label{q1}
\end{align}
Now we decide the values of $\{c_i\}$ to let the last two terms in the above inequality be negative. To this end, it suffices to let 
\begin{align}\label{o1}
c_{i+1}&\leq c_i- 3L^2 \Big(\frac{\gamma^2 \beta^2 L (1-\beta^k)^2}{(1-\beta)^3}+\frac{4c_0\gamma^2\beta^2(1-\beta^k)^2}{(1-\beta)^2}\Big) a_{k,k-i-1} \nonumber \\
&\overset{\eqref{akj}}{=}c_i -  \Big(\frac{\gamma^2 \beta^2 L (1-\beta^k)^2}{(1-\beta)^3}+\frac{4c_0\gamma^2\beta^2(1-\beta^k)^2}{(1-\beta)^2}\Big) \frac{3L^2\beta^{i}}{1-\beta^{k}}(i+1+\frac{\beta}{1-\beta}) \nonumber \\
&=c_i-3L^2\Big(\frac{\gamma^2L(1-\beta^{k})}{(1-\beta)^3}+\frac{4c_0\gamma^2(1-\beta^{k})}{(1-\beta)^2}\Big)\beta^{i+2}(i+1\frac{\beta}{1-\beta})
\end{align}
Since $1-\beta^{k+1}<1$, it suffices to let
\begin{equation}\label{ci-ci-1}
\begin{split}
c_{i+1}= c_i- 3L^2\Big(\frac{\gamma^2L}{(1-\beta)^3}+\frac{4c_0\gamma^2}{(1-\beta)^2}\Big)\beta^{i+2}(i+1+\frac{\beta}{1-\beta})
\end{split}
\end{equation}
so that inequality \eqref{o1} holds. With recursion \eqref{ci-ci-1}, we have 
\begin{equation}
\begin{split}
c_0 &= \sum\limits_{i=0}^{\infty}3L^2\Big(\frac{\gamma^2L}{(1-\beta)^3}+\frac{4c_0\gamma^2}{(1-\beta)^2}\Big)\beta^{i+2}(i+1+\frac{\beta}{1-\beta})\\
&=3L^2\Big(\frac{\gamma^2L}{(1-\beta)^3}+\frac{4c_0\gamma^2}{(1-\beta)^2}\Big)\frac{\beta^2+\beta^3}{(1-\beta)^2}
\end{split}
\end{equation}
which indeed yields 
\begin{equation}
c_0= \frac{3\frac{(\beta^2+\beta^3)\gamma^2L^3}{(1-\beta)^5}}{1-\frac{12(\beta^2+\beta^3)\gamma^2L^2}{(1-\beta)^4}},
\end{equation}
and $c_0$ is positive as long as $\gamma<\frac{(1-\beta)^2}{2\sqrt{3}\sqrt{\beta^3+\beta^2}L}$. Furthermore, when $\gamma\leq \frac{(1-\beta)^2}{5\sqrt{\beta+\beta^2}L}$, we have
\begin{equation}\label{eqn:c_0_upp}
c_0\leq \frac{L}{4(1-\beta)}
\end{equation}
In other words, when $\gamma\leq \frac{(1-\beta)^2}{5\sqrt{\beta+\beta^2}L}$, there exists a positive sequence $\{c_i\}_{i=0}^\infty$ such that the last two terms in \eqref{q1} are non-positive. Substituting \eqref{eqn:c_0_upp} into \eqref{q1} and choosing $\gamma \le \frac{(1-\beta)^2}{12L\beta^2}$, we have
\begin{align}
\frac{\gamma L^2}{2n(1-\beta)}+\frac{3 L^2}{n}\frac{(1+4\frac{c_0}{L}(1-\beta))\gamma^2 \beta^2 L (1-\beta^k)^2}{(1-\beta)^3} &\le \frac{\gamma L^2}{n(1-\beta)} \nonumber \\
 \frac{3 L^2}{n}\frac{(1+4\frac{c_0}{L}(1-\beta))\gamma^2 \beta^2 L}{(1-\beta)^2} &\le \frac{6L^3\gamma^2 \beta^2}{n(1-\beta)^2} \nonumber \\
 \frac{\gamma^2 L\sigma^2}{(1-\beta)n}  + \frac{2c_0 \gamma^2 \sigma^2 }{n(1-\beta)} &\le \frac{3\gamma^2 \sigma^2 L}{2n(1-\beta)^2} \nonumber
\end{align}
Substituting the above relations into \eqref{q1}, we achieve the result in \eqref{lya-inequality}. Finally, to guarantee \eqref{z8}, it suffices to let $\gamma \leq \frac{(1-\beta)^2}{(5-\beta+2\beta^2)L}$. 
\end{proof}

\vspace{2mm}
\subsection{Proof of Theorem \ref{thm-decentLaM--nc}.} With \eqref{lya-inequality}, we can achieve 
\begin{align}\label{eqn:92bgbs}
\frac{1}{T+1}\sum\limits_{k=0}^{T} \mathbb{E}\| \nabla f( \bar{s}^{(k)}) \|^2 &\leq \frac{2(1-\beta)}{\gamma(T+1)}\cL^0+\frac{3\gamma\sigma^2 L}{n(1-\beta)}+\frac{2(1-\beta)}{\gamma(T+1)}\sum\limits_{k=0}^T\sum\limits_{i=0}^k b_{k,i} \EE\|\bs^{(i)}-\bar{\bs}^{(i)}\|^2 \nonumber \\ 
 &= \frac{2(1-\beta)}{\gamma(T+1)}\cL^0+\frac{3\gamma\sigma^2 L}{n(1-\beta)}+\frac{2(1-\beta)}{\gamma(T+1)}\sum\limits_{i=0}^T\sum\limits_{k=i}^T b_{k,i} \EE\|\bs^{(i)}-\bar{\bs}^{(i)}\|^2
\end{align}
Note that 
\begin{align}
\sum\limits_{k=i}^T b_{k,i}=\sum\limits_{k=i+1}^{T}\frac{6L^3\gamma^2 \beta^2}{n(1-\beta)^2}\beta^{k-1-i} + \frac{\gamma L^2}{n(1-\beta)} \le \frac{6L^3\gamma^2 \beta^2}{n(1-\beta)^3} + \frac{\gamma L^2}{n(1-\beta)}  \le \frac{2\gamma L^2}{n(1-\beta)} 
\end{align}
where the last inequality holds when $\gamma\leq \frac{(1-\beta)^2}{3L\beta^2}$. With the above inequality, we have 
\begin{align}
&\ \frac{2(1-\beta)}{\gamma(T+1)}\sum\limits_{i=0}^T\sum\limits_{k=i}^T b_{k,i} \EE\|\bs^{(i)}-\bar{\bs}^{(i)}\|^2 \nonumber \\
&\le\  \frac{4L^2}{n(T+1)} \sum_{k=0}^T \EE\|\bs^{(k)}-\bar{\bs}^{(k)}\|^2 \nonumber \\
&\overset{\eqref{nc-consensus-2}}{\le}\  \frac{192L^2 \rho \gamma^2}{(1-\rho)^2(T+1)} \sum_{k=0}^{T} \mathbb{E}\|\nabla f(\bar{s}^{(k)})\|^2+ \frac{192L^2 \rho \gamma^2\hat{b}^2}{(1-\rho)^2} + \frac{32 \rho L^2 \gamma^2\sigma^2}{1-\rho} \nonumber \\
&\le \ \frac{1}{2(T+1)} \sum_{k=0}^{T} \mathbb{E}\|\nabla f(\bar{s}^{(k)})\|^2+ \frac{192L^2 \rho \gamma^2\hat{b}^2}{(1-\rho)^2} + \frac{32 \rho L^2 \gamma^2\sigma^2}{1-\rho}
\end{align}
where the last inequality holds when $\gamma \le \frac{1-\rho}{20\sqrt{\rho}L}$. Substituting the above inequality into \eqref{eqn:92bgbs}, we achieve 
\begin{align}
\frac{1}{T+1}\sum\limits_{k=0}^{T} \mathbb{E}\| \nabla f( \bar{s}^{(k)}) \|^2 &\leq \frac{4(1-\beta)}{\gamma(T+1)}\cL^0 + \frac{6\gamma\sigma^2 L}{n(1-\beta)} + \frac{64 \rho L^2 \gamma^2\sigma^2}{1-\rho} + \frac{384 L^2 \rho \gamma^2\hat{b}^2}{(1-\rho)^2}  \nonumber \\
&=\cO(\frac{1-\beta}{\gamma(T+1)})+\cO(\frac{\gamma}{1-\beta}\frac{\sigma^2}{n})+ \cO(\frac{ \gamma^2\sigma^2}{1-\rho}) +  \cO(\frac{ \gamma^2 \hat{b}^2}{(1-\rho)^2}). 
\end{align}
Note that $\bar{x}^{(k)} = \bar{s}^{(k)}$ (see Sec.~\ref{app-support-lemmas}), we achieve the final result.

\vspace{2mm}
\subsection{Proof of Corollary \ref{coro-decentLaM-nc}.} 
Choose $\gamma = \frac{\sqrt{n}(1-\beta)}{\sqrt{
T+1}\sigma}$, for $T+1\geq \frac{n}{\sigma^2\min\{\frac{1-\beta}{12L},\frac{1-\rho}{20\sqrt{\rho}L}\}^2}$, the step size $\gamma$ satisfies the condition in Theorem \ref{thm-decentLaM--nc}, we thus have 
\begin{align}
\frac{1}{T+1}\sum\limits_{k=0}^{T} \mathbb{E}\| \nabla f( \bar{s}^{(k)}) \|^2 &\leq \frac{4(1-\beta)}{\gamma(T+1)}\cL^0 + \frac{6\gamma\sigma^2 L}{n(1-\beta)} + \frac{64 (1-\beta)^2 \rho L^2 \gamma^2\sigma^2}{1-\rho} + \frac{384 L^2 \rho \gamma^2\hat{b}^2}{(1-\rho)^2}  \nonumber \\
& = \frac{4}{\sqrt{n}\sqrt{T+1}}\cL^0 + \frac{6\sigma}{\sqrt{n}\sqrt{T+1}} + \frac{64 (1-\beta)^4n \rho L^2 }{(1-\rho)(T+1)} + \frac{384(1-\beta)^2n\rho L^2\hat{b}^2}{(1-\rho)^2(T+1)\sigma^2} \nonumber\\
&=\cO(\frac{1}{\sqrt{n}\sqrt{T+1}})+\cO(\frac{(1-\beta)^4n}{(1-\rho)(T+1)})+  \cO(\frac{(1-\beta)^2n}{(1-\rho)^2(T+1)})\nonumber\\
&=\cO(\frac{1}{\sqrt{n}\sqrt{T+1}}+\frac{(1-\beta)^2n}{(1-\rho)^2(T+1)})=\cO(\frac{1}{\sqrt{n(T+1)}})
\end{align}

\section{Convergence Analysis for Strongly-convex Scenario}
\label{app-thm-decentLaM--sc}

\subsection{Descent Lemma for Lyapunov Function}
\begin{lemma}[\sc Descent Lemma for Lyapunov Function for Strongly Convex Scenario]\label{lem:sc-final}
    Under Assumptions A.1--A.3 and A.5, with appropriately chosen $c_i\geq 0$, if learning rate $\gamma \leq  \min\{\frac{(1-\beta)^2}{12L},\frac{(1-\beta)^2}{(7-\beta+5\beta^2)L}\} = O(\frac{(1-\beta)^2}{L})$, it holds that 
    \begin{align}\label{lya-inequality-sc}
     &\mathbb{E}[\cL^{k+1} - \cL^k] \nonumber\\
    \leq &- \frac{\gamma L \mu}{2(1-\beta)(L+\mu)}\EE[\cL^k]-  \frac{\gamma L \mu}{2(1-\beta)(L+\mu)} \big(\mathbb{E}f( \bar{s}^{(k)}) - f^\star\big) +\frac{3\gamma^2 L\sigma^2}{2(1-\beta)n} \nonumber\\
    &\quad + \frac{5L\mu\gamma^2 \beta^2 \sigma^2}{(L+\mu)2n(1-\beta)^2} +\sum\limits_{i=0}^{k}b_{k,i}\EE\|\bs^{(i)}-\bar{\bs}^{(i)}\|^2
    \end{align}
    with $b_{k,k} \leq \frac{\gamma L^2}{2n(1-\beta)}+ \frac{9\gamma^2L^3}{n(1-\beta)^3}+\frac{9L^3\mu \gamma^2 }{(L+\mu)n(1-\beta)^3}$ and  $b_{k,i} \leq  \Big( \frac{9\gamma^2 L^3 }{n(1-\beta)^2} + \frac{9L^3\mu \gamma^2 }{(L+\mu)n(1-\beta)^2} \Big)\beta^{k-i}$ for $i\in[0,k-1]$.
\end{lemma}
\begin{proof}
We firstly derive a lower bound for $\mathbb{E}\| \nabla f( \bar{s}^{(k)}) \|^2$. Recall from Lemma \eqref{x762bzz}, we have
\begin{align}\label{x762bzz}
\mathbb{E}[f(\bar{t}^{(k+1)})] &\le \mathbb{E}[f(\bar{t}^{(k)})]  - \frac{\gamma}{2(1-\beta)}\mathbb{E}\| \nabla f( \bar{s}^{(k)}) \|^2  + \frac{\gamma L^2}{2n(1-\beta)}\mathbb{E}\|\bar{\bs}^{(k)} - \bs^{(k)}\|^2 +  \frac{\gamma^2 L\sigma^2}{(1-\beta)n} \nonumber \\
&\quad - \Big(\frac{\gamma}{2(1-\beta)} - \frac{\gamma^2 L}{(1-\beta)^3} - \frac{L\gamma^2}{2(1-\beta)^2} - \frac{\gamma^2 \beta^2 L}{(1-\beta)^3}\Big)\mathbb{E}\|\frac{1}{n}\mathds{1}^T \nabla f(W^{\frac{1}{2}}\bs^{(k)})\|^2 \nonumber \\
&\quad + \frac{\gamma^2 \beta^2 L}{(1-\beta)^3} \mathbb{E}\|(1-\beta)\sum_{i=0}^{k-1}\beta^{k-1-i} [\frac{1}{n}\mathds{1}^T\nabla f(W^{\frac{1}{2}}\bs^{(i)}) - \frac{1}{n}\mathds{1}^T\nabla f(W^{\frac{1}{2}}\bs^{(k)})]\|^2.
\end{align} \label{lm-descent-general}
Since $f(x)$ is $\mu$-strongly convex, we have
\begin{align}\label{23bxsc7a}
\mathbb{E}\| \nabla f( \bar{s}^{(k)}) \|^2 \ge 2\mu \big( \mathbb{E}f( \bar{s}^{(k)}) - f^\star \big).
\end{align}
On the other hand, since $f(x)$ is $L$-smooth, it holds that 
\begin{equation*}
\begin{split}
\mathbb{E}\left[f(\bar{t}^{(k)})\right] \leq & \mathbb{E}\left[f(\bar{s}^{(k)})\right]+\mathbb{E}\left[\left\langle \nabla f(\bar{s}^{(k)}), \bar{t}^{(k)}-\bar{s}^{(k)}\right\rangle\right]+\frac{L}{2} \mathbb{E}\|\bar{t}^{(k)}-\bar{s}^{(k)}\|^{2} \\
\overset{\eqref{t-recursion-1}}{=}&\mathbb{E}\left[f(\bar{s}^{(k)})\right] + \mathbb{E}\left[\left\langle \nabla f(\bar{s}^{(k)})-\frac{1-\beta}{1-\beta^{k+1}}\sum_{i=0}^{k} \beta^{k-i} \nabla f(\bar{s}^{(i)})+\frac{1-\beta}{1-\beta^{k+1}}\sum_{i=0}^{k} \beta^{k-i} \nabla f(\bar{s}^{(i)}),-\frac{\gamma \beta}{(1-\beta)^2} \bar{m}_s^{(k)}\right\rangle\right] \\
&+\frac{L}{2} \mathbb{E}\left[\left\|\frac{\gamma \beta}{(1-\beta)^2} \bar{m}_s^{(k)}\right\|^{2}\right] \\
\overset{(a)}{\leq}&\mathbb{E}\left[f(\bar{s}^{(k)})\right] +\frac{\gamma \eta_1}{2(1-\beta)}  \mathbb{E}\left[\left\|\nabla f(\bar{s}^{(k)})-\frac{1-\beta}{1-\beta^{k+1}} \sum_{i=0}^{k} \beta^{k-i} \nabla f(\bar{s}^{(i)})\right\|^{2}\right]+\frac{\gamma}{2(1-\beta)\eta_1} \mathbb{E}\left[\left\|\frac{\beta}{1-\beta} \bar{m}_s^{(k)}\right\|^{2}\right] \\
&+\mathbb{E}\left[\left\langle\frac{1-\beta}{1-\beta^{k+1}} \sum_{i=0}^{k} \beta^{k-i} \nabla f(\bar{s}^{(i)}),-\frac{\gamma \beta}{(1-\beta)^2} \bar{m}_s^{(k)}\right\rangle\right]+\frac{L}{2} \mathbb{E}\left[\left\|\frac{\gamma \beta}{(1-\beta)^2} \bar{m}_s^{(k)}\right\|^{2}\right] \\
\overset{(b)}{\leq} & \mathbb{E}\left[f(\bar{s}^{(k)})\right] +\frac{\gamma \eta_1}{2(1-\beta)}  \mathbb{E}\left[\left\|\nabla f(\bar{s}^{(k)})-\frac{1-\beta}{1-\beta^{k+1}} \sum_{i=0}^{k} \beta^{k-i} \nabla f(\bar{s}^{(i)})\right\|^{2}\right] \\
&+\left(\frac{\gamma}{2(1-\beta)\eta_1} +\frac{L \gamma^{2}}{2(1-\beta)^2}\right) \left(\frac{\beta}{1-\beta}\right)^{2} \mathbb{E}\left[\left\|\bar{m}_s^{(k)}\right\|^{2}\right] \\
&+\frac{\gamma\beta}{(1-\beta)^2}\left(\frac{\eta_{2}}{2} \mathbb{E}\left[\left\|\frac{1-\beta}{1-\beta^{k+1}} \sum_{i=0}^{k} \beta^{k-i} \nabla f(\bar{s}^{(i)})\right\|^{2}\right]+\frac{1}{2 \eta_{2}} \mathbb{E}\left[\left\|\bar{m}_s^{(k)}\right\|^{2}\right]\right) \\
\overset{(c)}{\leq}& \mathbb{E}\left[f(\bar{s}^{(k)})\right]  + \frac{\gamma(1-\beta^{k+1})}{2(1-\beta)^2} \mathbb{E}\left[\left\|\nabla f(\bar{s}^{(k)})-\frac{1-\beta}{1-\beta^{k+1}} \sum_{i=0}^{k} \beta^{k-i} \nabla f(\bar{s}^{(i)})\right\|^{2}\right] \\
&+\left(\frac{\gamma\beta^2}{(1-\beta)^2(1-\beta^{k+1})}+\frac{L \gamma^{2}}{2}\frac{\beta^2}{(1-\beta)^4}\right) \mathbb{E}\left[\left\|\bar{m}_s^{(k)}\right\|^{2}\right]  +\frac{\gamma(1-\beta^{k+1})}{2(1-\beta)^2} \mathbb{E}\left[\left\|\frac{1-\beta}{1-\beta^{k+1}} \sum_{i=0}^{k} \beta^{k-i} \nabla f(\bar{s}^{(i)})\right\|^{2}\right],
\end{split}
\end{equation*}
where (a) and (b) hold because of the Cauchy's inequality, and $(c)$ holds by letting $\eta_1=\frac{1-\beta^{k+1}}{1-\beta}$, $\eta_2=\frac{1}{\beta(1-\beta^{k+1})}$ respectively and using the fact $1-\beta^{k+1}\leq \frac{1}{1-\beta^{k+1}}$. Substituting the above inequality into \eqref{23bxsc7a}, we reach
\begin{align}\label{xczn2ewg}
&\ \mathbb{E}\| \nabla f( \bar{s}^{(k)}) \|^2 \nonumber \\
 \ge&\ \mu \big( \mathbb{E}f( \bar{s}^{(k)}) - f^\star \big) + \mu \Biggl( \mathbb{E}f( \bar{t}^{(k)}) - f^\star  - \frac{\gamma(1-\beta^{k+1})}{2(1-\beta)^2} \mathbb{E}\left[\left\|\nabla f(\bar{s}^{(k)})-\frac{1-\beta}{1-\beta^{k+1}} \sum_{i=0}^{k} \beta^{k-i} \nabla f(\bar{s}^{(i)})\right\|^{2}\right] \nonumber \\
&\ -  \left(\frac{\gamma\beta^2}{(1-\beta)^2(1-\beta^{k+1})}+\frac{L \gamma^{2}}{2}\frac{\beta^2}{(1-\beta)^4}\right) \mathbb{E}\left[\left\|\bar{m}_s^{(k)}\right\|^{2}\right]  - \frac{\gamma(1-\beta^{k+1})}{2(1-\beta)^2} \mathbb{E}\left[\left\|\frac{1-\beta}{1-\beta^{k+1}} \sum_{i=0}^{k} \beta^{k-i} \nabla f(\bar{s}^{(i)})\right\|^{2}\right]\Biggl)
\end{align}

By Lemma \ref{lm-bound-ms} and set $\mathbf{y}=W\bs^{(k)} $, we have
\begin{align}\label{p2378-sc}
\mathbb{E}\|\bar{m}_s^{(k)}\|^2 &\leq \frac{2(1-\beta)(1-\beta^{2(k+1)})\sigma^2}{n}+4(1-\beta^k)^2 \mathbb{E}\|\frac{1}{n}\mathds{1}^T\nabla f(W\bs^{(k)})\|^2 \nonumber\\
&\quad + 4\mathbb{E}\|(1-\beta)\sum_{i=0}^{k-1}\beta^{k-1-i} [\frac{1}{n}\mathds{1}^T\nabla f(W\bs^{(i)}) - \frac{1}{n}\mathds{1}^T\nabla f(W\bs^{(k)})]\|^2
\end{align}
We also have
\begin{align}\label{p2378-sc-2}
\mathbb{E}\left[\left\|\frac{1-\beta}{1-\beta^{k+1}} \sum_{i=0}^{k} \beta^{k-i} \nabla f(\bar{s}^{(i)})\right\|^{2}\right]
\le& 2 \mathbb{E}\left[\left\|\frac{1-\beta}{1-\beta^{k+1}} \sum_{i=0}^{k} \beta^{k-i} \nabla f(\bar{s}^{(i)}) - \nabla f(\bar{s}^{(k)})\right\|^{2}\right] + 2\mathbb{E}\|\nabla f(\bar{s}^{(k)})\|^2\nonumber\\
\end{align}
Substituting \eqref{p2378-sc} and \eqref{p2378-sc-2} to \eqref{xczn2ewg}, we have 
\begin{align}\label{eqn:gvbfbiodf}
&(1 + \frac{\gamma \mu}{(1-\beta)^2}) \mathbb{E}\| \nabla f( \bar{s}^{(k)}) \|^2 \nonumber\\
\ge&\ \mu \big( \mathbb{E}f( \bar{s}^{(k)}) - f^\star \big) + \mu \Biggl( \mathbb{E}f( \bar{t}^{(k)}) - f^\star  - \frac{3\gamma(1-\beta^{k+1})}{2(1-\beta)^2} \mathbb{E}\left[\left\|\nabla f(\bar{s}^{(k)})-\frac{1-\beta}{1-\beta^{k+1}} \sum_{i=0}^{k} \beta^{k-i} \nabla f(\bar{s}^{(i)})\right\|^{2}\right] \nonumber \\
&\ -  4(1-\beta^k)^2\left(\frac{\gamma\beta^2}{(1-\beta)^2(1-\beta^{k+1})}+\frac{L \gamma^{2}}{2}\frac{\beta^2}{(1-\beta)^4}\right)\mathbb{E}\|\frac{1}{n}\mathds{1}^T\nabla f(W\bs^{(k)})\|^2\nonumber\\
&- 4\left(\frac{\gamma\beta^2}{(1-\beta)^2(1-\beta^{k+1})}+\frac{L \gamma^{2}}{2}\frac{\beta^2}{(1-\beta)^4}\right)\mathbb{E}\|(1-\beta)\sum_{i=0}^{k-1}\beta^{k-1-i} [\frac{1}{n}\mathds{1}^T\nabla f(W\bs^{(i)}) - \frac{1}{n}\mathds{1}^T\nabla f(W\bs^{(k)})]\|^2 \nonumber \\
&\ -\frac{2\sigma^2}{n}\left(\frac{\gamma\beta^2(1+\beta^{k+1})}{1-\beta}+\frac{L \gamma^{2}}{2}\frac{\beta^2(1-\beta^{2(k+1)})}{(1-\beta)^3}\right) \Biggl)\nonumber\\
\overset{(a)}{\geq}&\ \mu \big( \mathbb{E}f( \bar{s}^{(k)}) - f^\star \big) + \mu \Biggl( \mathbb{E}f( \bar{t}^{(k)}) - f^\star  - \frac{3\gamma(1-\beta^{k+1})}{2(1-\beta)^2} \mathbb{E}\left[\left\|\nabla f(\bar{s}^{(k)})-\frac{1-\beta}{1-\beta^{k+1}} \sum_{i=0}^{k} \beta^{k-i} \nabla f(\bar{s}^{(i)})\right\|^{2}\right] \nonumber \\
& - \frac{6\gamma}{(1-\beta)^2}\mathbb{E}\|\frac{1}{n}\mathds{1}^T\nabla f(W\bs^{(k)})\|^2-\frac{5\gamma\beta^2\sigma^2}{n(1-\beta)}\nonumber\\
&- \frac{\gamma\beta^2}{(1-\beta)^2}\big(\frac{4}{1-\beta^{k+1}}+2\big)\mathbb{E}\|(1-\beta)\sum_{i=0}^{k-1}\beta^{k-1-i} [\frac{1}{n}\mathds{1}^T\nabla f(W\bs^{(i)}) - \frac{1}{n}\mathds{1}^T\nabla f(W\bs^{(k)})]\|^2\Biggl)
\end{align}
where, in (a), we use the fact $\frac{(1-\beta^k)^2}{1-\beta^{k+1}}<1$ and $\frac{L \gamma^{2}}{2}\frac{\beta^2}{(1-\beta)^4}\leq \frac{\gamma \beta^2}{2(1-\beta)^2}$ because of $\gamma\leq \frac{(1-\beta)^2}{L}$. Again when because $\gamma\leq\frac{(1-\beta)^2}{L}$, we have $1+\frac{\gamma\mu}{(1-\beta)^2}\leq 1+\frac{\mu}{L}$ and hence by \eqref{eqn:gvbfbiodf}, it holds that
\begin{align}\label{23zzzzt23}
&\mathbb{E}\| \nabla f( \bar{s}^{(k)}) \|^2 \nonumber \\
\ge&\ \frac{L\mu}{L+\mu}  \big( \mathbb{E}f( \bar{s}^{(k)}) - f^\star \big) + \frac{L\mu}{L+\mu}  \Biggl( \mathbb{E}f( \bar{t}^{(k)}) - f^\star  - \frac{3\gamma(1-\beta^{k+1})}{2(1-\beta)^2} \mathbb{E}\left[\left\|\nabla f(\bar{s}^{(k)})-\frac{1-\beta}{1-\beta^{k+1}} \sum_{i=0}^{k} \beta^{k-i} \nabla f(\bar{s}^{(i)})\right\|^{2}\right] \nonumber \\
& - \frac{6\gamma }{(1-\beta)^2}\mathbb{E}\|\frac{1}{n}\mathds{1}^T\nabla f(W\bs^{(k)})\|^2-\frac{5\gamma\beta^2\sigma^2}{n(1-\beta)} \nonumber\\
& - \frac{\gamma\beta^2}{(1-\beta)^2}\big(\frac{4}{1-\beta^{k+1}}+2\big)\mathbb{E}\|(1-\beta)\sum_{i=0}^{k-1}\beta^{k-1-i} [\frac{1}{n}\mathds{1}^T\nabla f(W\bs^{(i)}) - \frac{1}{n}\mathds{1}^T\nabla f(W\bs^{(k)})]\|^2  \Biggl). 
\end{align}

Now, we consider the Lyapunov function $\cL^k=f(\bar{t}^{(k)})-f^\star +\sum\limits_{i=0}^{k-1}c_i\|\bar{s}^{(k-i)}-\bar{s}^{(k-i-1)}\|^2$, where $c_i\geq 0$ will be decided later. By \eqref{eqn:76gfgdgds}, we have
\begin{equation}\label{eqn:77gfgdgds_copy}
\begin{split}
     \EE[\cL^{k+1}-\cL^k]\leq &  - \frac{\gamma}{2(1-\beta)}\mathbb{E}\| \nabla f( \bar{s}^{(k)}) \|^2  + \frac{\gamma L^2}{2n(1-\beta)}\mathbb{E}\|\bar{\bs}^{(k)} - \bs^{(k)}\|^2 +  \frac{\gamma^2 L\sigma^2}{(1-\beta)n} +\frac{2c_0\gamma^2\sigma^2}{(1-\beta)n}\\
    &\quad - \Big(\frac{\gamma}{2(1-\beta)} - \frac{\gamma^2 L}{(1-\beta)^3} - \frac{L\gamma^2}{2(1-\beta)^2} - \frac{\gamma^2 \beta^2 L}{(1-\beta)^3}-\frac{4c_0\gamma^2}{(1-\beta)^2}\Big)\mathbb{E}\|\frac{1}{n}\mathds{1}^T \nabla f(W^{\frac{1}{2}}\bs^{(k)})\|^2 \\
    &\quad + \Big(\frac{\gamma^2 \beta^2 L }{(1-\beta)^3}+\frac{4c_0\gamma^2\beta^2}{(1-\beta)^2}\Big) \mathbb{E}\|(1-\beta)\sum_{i=0}^{k-1}\beta^{k-1-i} [\frac{1}{n}\mathds{1}^T\nabla f(W\bs^{(i)}) - \frac{1}{n}\mathds{1}^T\nabla f(W\bs^{(k)})]\|^2\\
    &\quad+\sum_{i=0}^{k-1}\left(c_{i+1}-c_{i}\right) \mathbb{E}\left[\left\|\bar{s}^{(k-i)}-\bar{s}^{(k-i-1)}\right\|^{2}\right]
\end{split}
\end{equation}
Substituting \eqref{23zzzzt23} into \eqref{eqn:77gfgdgds_copy}, we achieve 
\begin{align}\label{x762bzz-sc-2}
&\ \mathbb{E}[\cL^{k+1} - \cL^k] \nonumber \\
&\le  - \frac{\gamma L \mu}{2(1-\beta)(L+\mu)} \big(\mathbb{E}f( \bar{t}^{(k)}) - f^\star\big) -  \frac{\gamma L \mu}{2(1-\beta)(L+\mu)} \big(\mathbb{E}f( \bar{s}^{(k)}) - f^\star\big) + \frac{\gamma L^2}{2n(1-\beta)}\mathbb{E}\|\bar{\bs}^{(k)} - \bs^{(k)}\|^2 \nonumber \\
&\quad +  \frac{\gamma^2 L\sigma^2}{(1-\beta)n} +  \frac{5L\mu\gamma^2 \beta^2 \sigma^2}{(L+\mu)2n(1-\beta)^2}+ \frac{2c_0 \gamma^2 \sigma^2}{n(1-\beta)}\nonumber \\
&\quad - \Big(\frac{\gamma}{2(1-\beta)} - \frac{\gamma^2 L}{(1-\beta)^3} - \frac{L\gamma^2}{2(1-\beta)^2} - \frac{\gamma^2 \beta^2 L}{(1-\beta)^3} - \frac{3L\mu \gamma^2}{(L+\mu)(1-\beta)^3} - \frac{4c_0\gamma^2}{(1-\beta)^2}\Big)\mathbb{E}\|\frac{1}{n}\mathds{1}^T \nabla f(W^{\frac{1}{2}}\bs^{(k)})\|^2 \nonumber \\
&\quad +\Big( \frac{\gamma^2 \beta^2L }{(1-\beta)^3} + \frac{L\mu \gamma^2 \beta^2}{(L+\mu)(1-\beta)^3}\big(\frac{2}{1-\beta^{k+1}}+1\big) + \frac{4c_0\gamma^2\beta^2}{(1-\beta)^2} \Big) \mathbb{E}\|(1-\beta)\sum_{i=0}^{k-1}\beta^{k-1-i} [\frac{1}{n}\mathds{1}^T\nabla f(W\bs^{(i)}) - \frac{1}{n}\mathds{1}^T\nabla f(W\bs^{(k)})]\|^2 \nonumber \\
&\quad + \frac{3L\mu\gamma^2(1-\beta^{k+1})}{4(1-\beta)^3(L+\mu)} \mathbb{E}\big\|\nabla f(\bar{s}^{(k)})-\frac{1-\beta}{1-\beta^{k+1}} \sum_{i=0}^{k} \beta^{k-i} \nabla f(\bar{s}^{(i)})\big\|^{2}\nonumber \\
&\quad + \sum_{i=0}^{k-1}(c_{i+1}- c_i)\mathbb{E}\|\bar{s}^{(k-i)} - \bar{s}^{(k-1-i)}\|^2
\end{align}

Following \eqref{2bnsba1} and \eqref{2bnsba1-2}, we have
\begin{align}\label{eqn:79hghgwq_copy}
&\ \mathbb{E}\|(1-\beta)\sum_{i=0}^{k-1}\beta^{k-1-i} [\frac{1}{n}\mathds{1}^T\nabla f(W^{\frac{1}{2}}\bs^{(i)}) - \frac{1}{n}\mathds{1}^T\nabla f(W^{\frac{1}{2}}\bs^{(k)})]\|^2 \nonumber \\
\leq&\frac{3(1-\beta)(1-\beta^k)L^2}{n}\sum_{i=0}^{k-1}\beta^{k-1-i}\mathbb{E}\|\bs^{(i)} - \bar{\bs}^{(i)}\|^2 + \frac{3(1-\beta^k)^2L^2}{n}\mathbb{E}\|\bs^{(k)} - \bar{\bs}^{(k)}\|^2 \nonumber\\
&\quad + 3(1-\beta^k)^2L^2\sum_{i=0}^{k-1} a_{k,k-1-i} \mathbb{E}\|\bar{s}^{(k-i)} - \bar{s}^{(k-1-i)}\|^2
\end{align}
and 
\begin{align}\label{eqn:79hghgwq_copy2}
&(1-\beta^{k+1})\mathbb{E}\big\|\nabla f(\bar{s}^{(k)})-\frac{1-\beta}{1-\beta^{k+1}} \sum_{i=0}^{k} \beta^{k-i} \nabla f(\bar{s}^{(i)})\big\|^{2}\nonumber\\
\leq &\frac{(1-\beta)^2\beta}{1-\beta^{k+1}}\sum\limits_{i=0}^k\beta^{k-i}\sum\limits_{i=0}^{k-1}\beta^{k-1-i}L^2\mathbb{E}\big\|\bar{s}^{(k)}-\bar{s}^{(i)})\big\|^{2}\nonumber \\
\leq &(1-\beta^k)\beta L^2 \sum\limits_{i=0}^{k-1}a_{k,k-1-i}\EE\|\bar{s}^{(k-i)}-\bar{s}^{(k-i-1)}\|^2
\end{align}
with $a_{k,k-1-i}=\frac{\beta^i}{1-\beta^k}(\frac{\beta}{1-\beta}+i+1)$.

Suppose $\gamma$ is sufficiently small such that
\begin{equation}\label{eqn:suppose_sc}
    \frac{\gamma}{2(1-\beta)} - \frac{\gamma^2 L}{(1-\beta)^3} - \frac{L\gamma^2}{2(1-\beta)^2} - \frac{\gamma^2 \beta^2 L}{(1-\beta)^3} - \frac{3L\mu \gamma^2 \beta^2}{(L+\mu)(1-\beta)^3} - \frac{4c_0\gamma^2}{(1-\beta)^2}\geq 0
\end{equation}
which we will check later on. Then substituting \eqref{eqn:79hghgwq_copy}, \eqref{eqn:79hghgwq_copy2} and \eqref{eqn:suppose_sc} into \eqref{x762bzz-sc-2} and noting $\frac{1-\beta^k}{1-\beta^{k+1}}<1$, $\beta^2\leq \beta<1$, we have
\begin{align}\label{x762bzz-sc-22}
&\mathbb{E}[\cL^{k+1} - \cL^k] \nonumber \\
&\le  - \frac{\gamma L \mu}{2(1-\beta)(L+\mu)} \big(\mathbb{E}f( \bar{t}^{(k)}) - f^\star\big) -  \frac{\gamma L \mu}{2(1-\beta)(L+\mu)} \big(\mathbb{E}f( \bar{s}^{(k)}) - f^\star\big) \nonumber\\
&\quad +  \frac{\gamma^2 L\sigma^2}{2(1-\beta)n} + \frac{5L\mu\gamma^2 \beta^2 \sigma^2}{(L+\mu)2n(1-\beta)^2}+ \frac{2c_0 \gamma^2 \sigma^2}{n(1-\beta)}\nonumber \\
&\quad + \Big(\frac{\gamma L^2}{2n(1-\beta)}+\frac{3L^2}{n}\Big( \frac{\gamma^2 L }{(1-\beta)^3} + \frac{3L\mu \gamma^2 }{(L+\mu)(1-\beta)^3} + \frac{4c_0\gamma^2}{(1-\beta)^2} \Big)\Big)\mathbb{E}\|\bar{\bs}^{(k)} - \bs^{(k)}\|^2 \nonumber \\
&\quad +\frac{3(1-\beta)L^2}{n}\Big( \frac{\gamma^2 L }{(1-\beta)^3} + \frac{3L\mu \gamma^2 }{(L+\mu)(1-\beta)^3} + \frac{4c_0\gamma^2}{(1-\beta)^2} \Big)\sum\limits_{i=0}^{k-1}\beta^{k-i}\EE\|\bs^{(i)}-\bar{\bs}^{(i)}\|^2\nonumber\\
&\quad +\Biggl(3\Big( \frac{\gamma^2 L }{(1-\beta)^3} + \frac{3L\mu \gamma^2 }{(L+\mu)(1-\beta)^3} + \frac{4c_0\gamma^2}{(1-\beta)^2} \Big) + \frac{3L\mu\gamma^2}{4(1-\beta)^3(L+\mu)}\Biggl)L^2\beta(1-\beta^{k}) \sum\limits_{i=0}^{k-1}a_{k,k-i-1}\EE\|\bar{s}^{(k-i)}-\bar{s}^{(k-i-1)}\|^2\nonumber \\
&\quad + \sum_{i=0}^{k-1}(c_{i+1}- c_i)\mathbb{E}\|\bar{s}^{(k-i)} - \bar{s}^{(k-1-i)}\|^2\nonumber\\
&=  - \frac{\gamma L \mu}{2(1-\beta)(L+\mu)}\big(\mathbb{E}f( \bar{t}^{(k)}) - f^\star\big) -  \frac{\gamma L \mu}{2(1-\beta)(L+\mu)} \big(\mathbb{E}f( \bar{s}^{(k)}) - f^\star\big) \nonumber\\
&\quad + \frac{\gamma^2 L\sigma^2}{2(1-\beta)n} + \frac{5L\mu\gamma^2 \beta^2 \sigma^2}{(L+\mu)2n(1-\beta)^2}+ \frac{2c_0 \gamma^2 \sigma^2}{n(1-\beta)}\nonumber\\
&\quad +\sum\limits_{i=0}^{k}b_{k,i}\EE\|\bs^{(i)}-\bar{\bs}^{(i)}\|^2+\sum\limits_{i=0}^{k-1}(c_{i+1}-c_i+P\beta L^2(1-\beta^{k})a_{k,k-i-1})\EE\|\bar{s}^{(k-i)}-\bar{s}^{(k-i-1)}\|^2
\end{align}
where
\begin{align}
b_{k,k}&=\frac{\gamma L^2}{2n(1-\beta)}+\frac{3L^2}{n}\Big( \frac{\gamma^2 L }{(1-\beta)^3} + \frac{3L\mu \gamma^2 }{(L+\mu)(1-\beta)^3} + \frac{4c_0\gamma^2}{(1-\beta)^2} \Big)\label{eqn:bkk}\\
b_{k,i}&=\Big( \frac{\gamma^2 L }{(1-\beta)^2} + \frac{3L\mu \gamma^2 }{(L+\mu)(1-\beta)^2} + \frac{4c_0\gamma^2}{1-\beta} \Big)\frac{3L^2}{n}\beta^{k-i}\quad \mbox{ if }0\leq i\leq k-1\label{eqn:bki}\\
P&= \frac{3\gamma^2 L }{(1-\beta)^3} + \frac{39L\mu \gamma^2 }{4(L+\mu)(1-\beta)^3} + \frac{12c_0\gamma^2}{(1-\beta)^2}.
\end{align}

Now we show there exist $c_i\geq 0$ such that $c_0\leq \frac{L}{2(1-\beta)}$ and for $0\leq i\leq k-1$, $\forall\, k\geq 1$
\begin{align}\label{eqn:ci-1}
&c_{i+1}-c_{i}+PL^2\beta(1-\beta^{k})a_{k,k-i-1}\leq  -\frac{\gamma L \mu}{2(1-\beta)(L+\mu)}c_i\nonumber\\
\end{align}
hence we can achieve that 
\begin{align}
&\mathbb{E}[\cL^{k+1} - \cL^k] \nonumber\\
\leq &- \frac{\gamma L \mu}{2(1-\beta)(L+\mu)}\EE[\cL^k]-  \frac{\gamma L \mu}{2(1-\beta)(L+\mu)} \big(\mathbb{E}f( \bar{s}^{(k)}) - f^\star\big) +\frac{3\gamma^2 L\sigma^2}{2(1-\beta)n} \nonumber\\
&\quad + \frac{5L\mu\gamma^2 \beta^2 \sigma^2}{(L+\mu)2n(1-\beta)^2} +\sum\limits_{i=0}^{k}b_{k,i}\EE\|\bs^{(i)}-\bar{\bs}^{(i)}\|^2
\end{align}
Note \eqref{eqn:ci-1} is equivalent to 
\begin{align}\label{eqn:ci-2}
&c_{i+1}+PL^2\beta^{i+1}\big(i+1+\frac{\beta}{1-\beta}\big)\leq (1-\frac{\gamma L \mu}{2(1-\beta)(L+\mu)})c_i\nonumber\\
\Longleftrightarrow\,&\frac{c_{i+1}}{ (1-\frac{\gamma L \mu}{2(1-\beta)(L+\mu)})^{i+1}}\leq \frac{c_i}{ (1-\frac{\gamma L \mu}{2(1-\beta)(L+\mu)})^i}-PL^2\frac{\beta^{i+1}(i+1+\frac{\beta}{1-\beta})}{ (1-\frac{\gamma L \mu}{2(1-\beta)(L+\mu)})^{i+1}}.
\end{align}
Therefore, by choosing $c_i(i\geq 1)$ such that the equality in \eqref{eqn:ci-2} holds and noting
\[
\gamma\leq \frac{(1-\beta)^2}{L}\leq \frac{2(1-\beta)^2(L+\mu)}{L\mu(1+\sqrt{\beta})}\,\Longrightarrow\,\frac{\beta}{ 1-\frac{\gamma L \mu}{2(1-\beta)(L+\mu)}}\leq \sqrt{\beta}<1
\] 
in order to assure positiveness, it suffices to let 
\begin{align}\label{eqn:ci-3}
c_0&=PL^2 \frac{6}{(1-\beta)^2}\\
&\geq PL^2\frac{1}{(1-\beta)^2} (\sqrt{\beta}(1+\sqrt{\beta})^2+\beta^\frac{3}{2}(1+\sqrt{\beta}))\nonumber\\
&=PL^2\Big(\frac{\sqrt{\beta}}{(1-\sqrt{\beta})^2}+\frac{\beta}{1-\beta}\frac{\sqrt{\beta}}{1-\sqrt{\beta}}\Big)\nonumber \\
&=PL^2\sum\limits_{i=0}^\infty (\sqrt{\beta})^{i+1}(i+1+\frac{\beta}{1-\beta})\nonumber\\
&\geq PL^2\sum\limits_{i=0}^\infty \frac{\beta^{i+1}}{ (1-\frac{\gamma L \mu}{2(1-\beta)(L+\mu)})^{i+1}}\big(i+1+\frac{\beta}{1-\beta}\big)\nonumber
\end{align}
Note $P= \frac{3\gamma^2 L }{(1-\beta)^3} + \frac{39L\mu \gamma^2 }{4(L+\mu)(1-\beta)^3} + \frac{12c_0\gamma^2}{(1-\beta)^2}$. Hence \eqref{eqn:ci-3} yields 
\begin{equation}
c_0 =\frac{\frac{\gamma^2L^3}{(1-\beta)^5}\Big(18+\frac{117\mu}{2(L+\mu)}\Big)}{1-\frac{72\gamma^2L^2}{(1-\beta)^4}}
\end{equation}
which is positive as long as $\gamma<\frac{(1-\beta)^2}{6\sqrt{2}L}$. Furthermore, if $\gamma\leq \frac{(1-\beta)^2}{12L}$, we have 
\begin{equation}\label{eqn:c0_upp_sc}
c_0\leq \frac{\frac{\gamma^2L^3}{(1-\beta)^5}\Big(18+\frac{117\mu}{2(L+\mu)}\Big)}{\frac{72\gamma^2L^2}{(1-\beta)^4}}=\frac{(\frac{1}{4}+\frac{117\mu}{144(L+\mu)})L}{(1-\beta)}\leq \frac{L}{2(1-\beta)}
\end{equation}
Then, following \eqref{eqn:bkk} and \eqref{eqn:bki}, we reach
\begin{align}\label{eqn:bkk-2}
b_{k,k}&\leq \frac{\gamma L^2}{2n(1-\beta)}+ \frac{9\gamma^2L^3}{n(1-\beta)^3}+\frac{9L^3\mu \gamma^2 }{(L+\mu)n(1-\beta)^3}\\
b_{k,i}&\leq \Big( \frac{9\gamma^2 L^3 }{n(1-\beta)^2} + \frac{9L^3\mu \gamma^2 }{(L+\mu)n(1-\beta)^2} \Big)\beta^{k-i}\quad \mbox{ if }0\leq i\leq k-1.\label{eqn:bki-2}
\end{align}
Finally, to guarantee \eqref{eqn:suppose_sc}, it suffices to let $\gamma \leq \frac{(1-\beta)^2}{(7-\beta+5\beta^2)L}$.
\end{proof}


\subsection{Proof of Theorem \ref{thm-decentLaM--sc}}
By Lemma \ref{lem:sc-final}, we have
\begin{align}\label{lya-inequality-sc-copy2}
     &\mathbb{E}[\cL^{k+1} - \cL^k] \nonumber\\
    \leq &- \frac{\gamma L \mu}{2(1-\beta)(L+\mu)}\EE[\cL^k]-  \frac{\gamma L \mu}{2(1-\beta)(L+\mu)} \big(\mathbb{E}f( \bar{s}^{(k)}) - f^\star\big) +\frac{3\gamma^2 L\sigma^2}{2(1-\beta)n} \nonumber\\
    &\quad + \frac{5L\mu\gamma^2 \beta^2 \sigma^2}{(L+\mu)2n(1-\beta)^2} +\sum\limits_{i=0}^{k}b_{k,i}\EE\|\bs^{(i)}-\bar{\bs}^{(i)}\|^2
    \end{align}
    with $b_{k,k} \leq \frac{\gamma L^2}{2n(1-\beta)}+ \frac{9\gamma^2L^3}{n(1-\beta)^3}+\frac{9L^3\mu \gamma^2 }{(L+\mu)n(1-\beta)^3}$ and  $b_{k,i} \leq  \Big( \frac{9\gamma^2 L^3 }{n(1-\beta)^2} + \frac{9L^3\mu \gamma^2 }{(L+\mu)n(1-\beta)^2} \Big)\beta^{k-i}$ for $i\in[0,k-1]$.
Note \eqref{lya-inequality-sc-copy2} is equivalent to 
\begin{align}
h_k\big(\mathbb{E}f( \bar{s}^{(k)}) - f^\star\big)&\leq\Big( \frac{2(1-\beta)(L+\mu)}{\gamma L\mu}-1\Big)h_k\EE[\cL^k] -\frac{2(1-\beta)(L+\mu)}{\gamma L\mu}h_k\mathbb{E}[\cL^{k+1} ]\nonumber\\
&\quad +\big(\frac{3\gamma (L+\mu)\sigma^2}{n\mu }+\frac{5\gamma \beta^2\sigma^2}{n(1-\beta)}\big)h_k\nonumber\\
&\quad +\frac{2(1-\beta)(L+\mu)}{\gamma L\mu}\sum\limits_{i=0}^{k}b_{k,i}\EE\|\bs^{(i)}-\bar{\bs}^{(i)}\|^2.
\end{align}
We thus have 
\begin{align}
\frac{1}{H_T}\sum\limits_{k=0}^Th_k\big(\mathbb{E}f( \bar{s}^{(k)}) - f^\star\big)&\leq \frac{1}{H_T} \sum\limits_{k=0}^T\Big[\Big( \frac{2(1-\beta)(L+\mu)}{\gamma L\mu}-1\Big)h_k\EE[\cL^k] -\frac{2(1-\beta)(L+\mu)}{\gamma L\mu}h_k\mathbb{E}[\cL^{k+1} ]\Big]\\
&\quad + \frac{1}{H_T}\frac{2(1-\beta)(L+\mu)}{\gamma L\mu}\sum\limits_{k=0}^T \sum\limits_{i=0}^k b_{k,i}h_k\mathbb{E}\|\bar{\bs}^{(k)} - \bs^{(k)}\|^2+ \frac{3\gamma (L+\mu)\sigma^2}{n\mu }+\frac{5\gamma \beta^2\sigma^2}{n(1-\beta)}.
\end{align}
Choose $\frac{2(1-\beta)(L+\mu)}{\gamma L\mu}h_k= \Big( \frac{2(1-\beta)(L+\mu)}{\gamma L\mu}-1\Big)h_{k+1}$, i.e.
\begin{equation}\label{eqn:HT-1}
h_k = \big(1-\frac{\gamma L\mu}{2(1-\beta)(L+\mu)})^{-k},\quad \frac{h_0}{H_T}=\frac{1}{H_T}\leq \frac{1}{h_T}=\big(1-\frac{\gamma L\mu}{2(1-\beta)(L+\mu)})^{T}
\end{equation}
then by telescoping, we reach
\begin{align}\label{gojgprtn}
\frac{1}{H_T}\sum\limits_{k=0}^Th_k\big(\mathbb{E}f( \bar{s}^{(k)}) - f^\star\big)&\leq \frac{h_0}{H_T} \Big( \frac{2(1-\beta)(L+\mu)}{\gamma L\mu}-1\Big)\EE[\cL^0]\nonumber \\
&\quad + \frac{1}{H_T}\frac{2(1-\beta)(L+\mu)}{\gamma L\mu}\sum\limits_{i=0}^T\sum\limits_{k=i}^{T}h_kb_{k,i}\EE\|\bs^{(i)}-\bar{\bs}^{(i)}\|^2+ \frac{3\gamma (L+\mu)\sigma^2}{n\mu }+\frac{5\gamma \beta^2\sigma^2}{n(1-\beta)}.   
\end{align}
Since 
\begin{align}\label{eqn:bkk-3}
\frac{2(1-\beta)(L+\mu)}{\gamma L\mu}b_{k,k} &\overset{\eqref{eqn:bkk-2}}{\leq }\frac{2(1-\beta)(L+\mu)}{\gamma L\mu}\Big(\frac{\gamma L^2}{2n(1-\beta)}+ \frac{9\gamma^2L^3}{n(1-\beta)^3}+\frac{9L^3\mu \gamma^2 }{(L+\mu)n(1-\beta)^3}\Big)\nonumber\\
&= \frac{L}{n}\frac{L+\mu}{\mu}\Big(1+\frac{18 \gamma L}{(1-\beta)^2}+\frac{18\gamma L\mu}{(L+\mu)(1-\beta)^2}\Big)\nonumber\\
&\overset{(a)}{\leq} \frac{2L}{n}\frac{L+\mu}{\mu}
\end{align}
where we use $\gamma \leq \frac{(1-\beta)^2}{27L}$ and $\frac{\mu}{L+\mu}\leq \frac{1}{2}$ in (a). Similarly, we reach
\begin{equation}\label{eqn:bki-3}
\frac{2(1-\beta)(L+\mu)}{\gamma L\mu}b_{k,i}\leq \frac{L}{n}\frac{L+\mu}{\mu}(1-\beta)\beta^{k-i}.
\end{equation}
Since $\gamma \leq\frac{(1-\beta)^2}{L} \leq \frac{(1-\beta)^2(L+\mu)}{L\mu}$ yields $\beta(1-\frac{\gamma L \mu}{2(1-\beta)(L+\mu)})^{-1}\leq \frac{2\beta}{1+\beta} $
\begin{align}\label{hghjhgjwe}
&\frac{2(1-\beta)(L+\mu)}{\gamma L\mu}\sum\limits_{k=i}^Th_kb_{k,i}\nonumber\\
\overset{\eqref{eqn:bkk-3},\eqref{eqn:bki-3}}{\leq} &\frac{2L}{n}\frac{L+\mu}{\mu} h_i+ \frac{L}{n}\frac{L+\mu}{\mu}(1-\beta)\sum\limits_{k=i+1}^T\beta^{k-i}h_k\nonumber\\
=&\frac{2L}{n}\frac{L+\mu}{\mu} h_i+ \frac{L}{n}\frac{L+\mu}{\mu}(1-\beta)\sum\limits_{k=i+1}^T\beta^{k-i}(1-\frac{\gamma L \mu}{2(1-\beta)(L+\mu)})^{-(k-i)}h_i\nonumber\\
\leq &\frac{2L}{n}\frac{L+\mu}{\mu} h_i+ \frac{L}{n}\frac{L+\mu}{\mu}(1-\beta)\sum\limits_{k=i+1}^\infty\big(\frac{2\beta}{1+\beta}\big)^{-(k-i)}h_i\nonumber\\
\overset{(c)}{\leq} &\frac{4L}{n}\frac{L+\mu}{\mu} h_i
\end{align}
where we use the fact $\sum\limits_{k=1}^{\infty}(\frac{2\beta}{1+\beta})^{k}=\frac{2\beta}{1-\beta}\leq \frac{2}{1-\beta}$ in (c).

Therefore, combining \eqref{hghjhgjwe} with \eqref{gojgprtn}, we reach that 
\begin{align}
&\quad \frac{1}{H_T}\sum\limits_{k=0}^Th_k\big(\mathbb{E}f( \bar{s}^{(k)}) - f^\star\big)\nonumber\\
&\leq \frac{h_0}{H_T} \Big( \frac{2(1-\beta)(L+\mu)}{\gamma L\mu}-1\Big)\EE[\cL^0]+\frac{1}{H_T}\sum\limits_{i=0}^T \frac{4L}{n}\frac{L+\mu}{\mu}h_i\EE\|\bs^{(i)}-\bar{\bs}^{(i)}\|^2+\frac{3\gamma (L+\mu)\sigma^2}{n\mu }+\frac{5\gamma \beta^2\sigma^2}{n(1-\beta)}\nonumber\\
\overset{(a)}{\leq}& \frac{h_0}{H_T} \Big( \frac{2(1-\beta)(L+\mu)}{\gamma L\mu}-1\Big)\EE[\cL^0]+\frac{3\gamma (L+\mu)\sigma^2}{n\mu }+\frac{5\gamma \beta^2\sigma^2}{n(1-\beta)}\nonumber \\
&\quad +\frac{4L}{n}\frac{L+\mu}{\mu} \Biggl(\frac{216 n L \gamma^2}{(1-\rho)^2H_T}\sum_{k=0}^{T} h_k\, \big( \mathbb{E}f(\bar{s}^{(k)}) - f^\star \big) + \frac{48 n\gamma^2{b}^2}{(1-\rho)^2} + \frac{8n\gamma^2\sigma^2}{1-\rho}\Biggl)
\end{align}
where we have use Lemma \ref{lem:sc-consensus} the fact that $(1-\frac{\gamma L\mu}{2(1-\beta)(L+\mu})^{-1}\leq1+\frac{1-\rho}{8}$ because of $\gamma \leq \frac{(1-\rho)(1-\beta)(L+\mu)}{5L\mu}$.
When $\gamma \leq \frac{1-\rho}{\sqrt{1728}L\sqrt{\kappa+1}}$ $(\kappa=\frac{L}{\mu})$, we have $1-\frac{4L}{n}\frac{L+\mu}{\mu}\frac{216nL^2\gamma^2}{(1-\rho)^2}\leq \frac{1}{2}$ and hence
\begin{align}
&\frac{1}{H_T}\sum\limits_{k=0}^Th_k\big(\mathbb{E}f( \bar{s}^{(k)}) - f^\star\big)\nonumber\\
\leq& \frac{h_0}{H_T} \frac{4(1-\beta)(L+\mu)}{\gamma L\mu}\cL^0+\frac{6\gamma (L+\mu)\sigma^2}{n\mu }+\frac{10\gamma \beta^2\sigma^2}{n(1-\beta)}+4L\frac{L+\mu}{\mu}\big(\frac{48\gamma^2\hat{b}^2}{(1-\rho)^2}+\frac{8\gamma^2\sigma^2}{1-\rho}\big)\nonumber\\
\overset{\eqref{eqn:HT-1}}{\leq} &\frac{4(1-\beta)(L+\mu)}{\gamma L\mu}\cL^0(1-\frac{\gamma L\mu}{2(1-\beta)(L+\mu)})^T+\frac{6\gamma (L+\mu)\sigma^2}{n\mu }+\frac{10\gamma \beta^2\sigma^2}{n(1-\beta)}+4L\frac{L+\mu}{\mu}\big(\frac{48\gamma^2\hat{b}^2}{(1-\rho)^2}+\frac{8\gamma^2\sigma^2}{1-\rho}\big)\nonumber\\
=&\cO\Big(\frac{(1-\beta)\kappa}{\gamma}(1-\frac{\gamma}{(1-\beta)\kappa})^{T}+\frac{\gamma \sigma^2}{n}\max\{\kappa,\frac{1}{1-\beta}\} +\frac{\gamma^2\kappa\hat{b}^2}{(1-\rho)^2}+\frac{\gamma^2\kappa \sigma^2}{1-\rho}\Big).
\end{align}
Note that $\bar{x}^{(k)} = \bar{s}^{(k)}$ (see Sec.~\ref{app-support-lemmas}), we achieve the final result.

\subsection{Proof of Corollary \ref{coro-sc}}
Choose $\gamma = \frac{2(1-\beta)(L+\mu)\ln(nT^2)}{L\mu T}$, for $T$ sufficiently large, the step size $\gamma$ satisfies the condition in Theorem \ref{thm-decentLaM--sc}, hence we have 
\begin{align}
&\frac{1}{H_T}\sum\limits_{k=0}^Th_k\big(\mathbb{E}f( \bar{s}^{(k)}) - f^\star\big)\nonumber\\
\leq &\frac{4(1-\beta)(L+\mu)}{\gamma L\mu}\cL^0(1-\frac{\gamma L\mu}{2(1-\beta)(L+\mu)})^T+\frac{6\gamma (L+\mu)\sigma^2}{n\mu }+\frac{10\gamma \beta^2\sigma^2}{n(1-\beta)}+4L\frac{L+\mu}{\mu}\big(\frac{48\gamma^2\hat{b}^2}{(1-\rho)^2}+\frac{8\gamma^2\sigma^2}{1-\rho}\big)\nonumber\\
\leq &\frac{4(1-\beta)(L+\mu)}{\gamma L\mu}\cL^0\exp(-\frac{\gamma L\mu}{2(1-\beta)(L+\mu)}T)+\frac{6\gamma (L+\mu)\sigma^2}{n\mu }+\frac{10\gamma \beta^2\sigma^2}{n(1-\beta)}+4L\frac{L+\mu}{\mu}\big(\frac{48\gamma^2\hat{b}^2}{(1-\rho)^2}+\frac{8\gamma^2\sigma^2}{1-\rho}\big)\nonumber\\
= &\frac{2}{\ln(nT^2)}\cL^0\frac{1}{nT}+\frac{12(1-\beta) (L+\mu)^2\sigma^2\ln(nT^2)}{nL\mu^2T }+\frac{20 (L+\mu)\sigma^2\ln(T^2)}{nL\mu T}\nonumber\\
&+4L\frac{L+\mu}{\mu}\big(\frac{48\hat{b}^2}{(1-\rho)^2}+\frac{8\sigma^2}{1-\rho}\big)\frac{4(1-\beta)^2(L+\mu)^2\ln(nT^2)^2}{L^2\mu^2T^2}\nonumber\\
=&\tilde{\cO}(\frac{1}{nT})
\end{align}
where $\tilde{\cO}$ hides constants and polylogarithmic factors.

\section{More Experimental Details}

\subsection{Experimental setting for Table \ref{talbe:motivation}}
\label{app-table-introduction}
Cifar-10 dataset contains 50,000 training samples and 10,000 validating samples. We follow the SOTA training scheme and train totally 200 epochs. The learning rate is linearly scaled and gradually warmed up form a relatively small value (e.g. 0.1) in the first 5 epochs. We decay the learning rate by a factor of 10 at 100, 150 epochs. To eliminate the effect of topology with different size in decentralized algorithms, we used 8 workers (i.e. 8 GPUs) in all decentralized training and changed the batch size of every single GPU respectively. For ImageNet experiments, the training setting is introduced in Sec.~\ref{sec:experiment} in details.

\subsection{Experimental setting for linear regression (i.e., Figs.~\ref{fig:ls-dsgd-dmsgd} and \ref{fig:ls-dsgd-dmsgd-decentLaM})}
\label{app-exp-linear-regression}
In this experiment, we consider a linear regression problem:
\begin{align}
\label{23nbsdbasd}
\min_{x\in \mathbb{R}^d}\quad \frac{1}{n}\sum_{i=1}^n f_i(x) \quad \mbox{where} \quad f_i(x) = \frac{1}{2}\|A_i x - b_i\|^2.
\end{align}
In the above problem, we set $n=8$ and all computing nodes are organized into the mesh topolology, see Fig.~\ref{fig:topo}. The weight matrix is generated from the Metropolis-Hastings rule \cite[Table~14.1]{sayed2014adaptation} so that it satisfies Assumption A.3. Quantities $A_i \in \mathbb{R}^{50\times 30}$ and $b_i \in \mathbb{R}^{50}$ are local data held in node $i$. Each $A_i$ is generated from the standard Gaussian distribution $\mathcal{N}(0,1)$, and $b_i = A_x x^o + s$ in which $x^o \in \RR^{30}$ is a predefined solution, and $s$ is a white noise with magnitude $0.01$. For DSGD, DmSGD, and DecentLaM, we set learning rate $\gamma=0.001$ and $\beta=0.8$. To evaluate the inconsistency bias, we let each node $i$ access the accurate gradient $\nabla f_i(x) = A_i^T(A_i x- b_i)$ rather than the stochastic gradient descent. The $y$-axis indicates the relative error $\frac{1}{n}\sum_{i=1}^n \|x_i^{(k)} - x^\star\|^2/\|x^\star\|^2$ in which $x^\star$ is the optimal solution to problem \eqref{23nbsdbasd}.

\subsection{Network topologes used in Table \ref{table-imagenet-resnet50-topo}}\label{app-topo}
We empirically investigate a series of undirected deterministic and  time-varying topologies. We generate the weight matrix $W$ according to the Metropolis-Hastings rule \cite[Table~14.1]{sayed2014adaptation} so that it is satisfies Assumption A.3. A positive-definite $W$ is not required in any of our experiments. We organize all computing nodes into the following topologies with BlueFog \cite{bluefog}. 

\begin{itemize}
    \item \textbf{Ring}. All nodes forms a logical ring topology. Every node communicate with its direct neighbors (i.e. 2 peers).
    \item \textbf{Mesh}. All nodes forms a logical mesh topology. Every node communicate with its direct neighbors. It is a multi-peer topology.
    \item \textbf{Symmetric Exponential Graph \cite{bluefog,assran2019stochastic}}. The node with odd rank $i$ communicates with even ranks $i+2^0-1, i+2^1-1, ..., i+2^{\left \lfloor \log_2 (n-1) \right \rfloor}$ by sending a message and waiting for a response. 
    \item \textbf{Bipartite Random Match}. All nodes are evenly divided into two non-overlapping groups randomly per iteration. Communication is only allowed within each pair of the nodes. We keep the same random seed in all nodes to avoid deadlocks.
\end{itemize}

We also visually illustrate the aforementioned topologies with 8 nodes in Fig.~\ref{fig:topo}.

\begin{figure}[t!]
    \centering
    \includegraphics[width=0.9\textwidth]{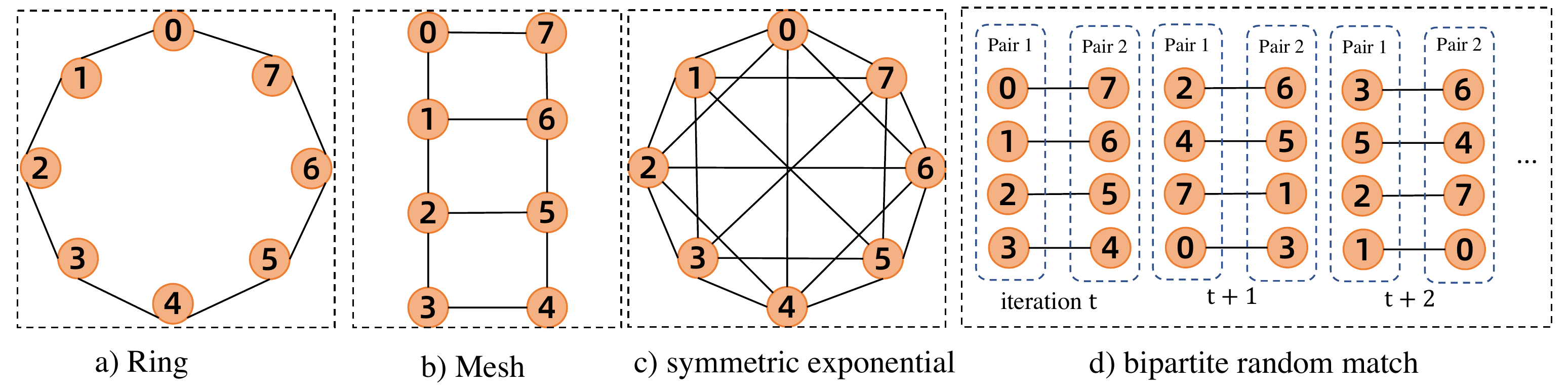}
    \caption{\small Different topologies with 8 nodes. } 
    \label{fig:topo} \vspace{-2mm}
    \vspace{-2mm}
\end{figure}

\end{document}